%% file: main.tex
\documentclass[Afour,sageh,times]{sagej}

\usepackage{moreverb,url}

\usepackage{amsmath,amssymb,amsfonts,cases}
\usepackage[vlined,linesnumbered,boxed]{algorithm2e}
\usepackage{booktabs}
\usepackage{tabularx}
\usepackage{array}
\usepackage{subcaption}
\usepackage{overpic}
\usepackage{color, colortbl}
\usepackage{multirow}

\usepackage{multicol}

\usepackage[colorlinks,bookmarksopen,bookmarksnumbered,citecolor=red,urlcolor=red]{hyperref}

\newcommand\BibTeX{{\rmfamily B\kern-.05em \textsc{i\kern-.025em b}\kern-.08em
T\kern-.1667em\lower.7ex\hbox{E}\kern-.125emX}}

\def\volumeyear{2016}

\theoremstyle{definition}

\newtheorem{assumption}{\textbf{Assumption}}
\newtheorem{proposition}{\textbf{Proposition}}
\newtheorem{theorem}{\textbf{Theorem}}
\newtheorem{lemma}{\textbf{Lemma}}
\newtheorem{remark}{\textbf{Remark}}
\newtheorem{definition}{\textbf{Definition}}
\newtheorem{example}{\textbf{Example}}

\newcommand*{\defeq}{\stackrel{\mathrm{def}}{=}}

\newcommand{\matf}[1]{\mathbf{#1}}
\newcommand{\vecf}[1]{\mathbf{#1}}
\newcommand*{\trans}{\mathsf{T}}

\begin{document}

\runninghead{Bai~\textit{et al.}: KernelGPA}

\title{KernelGPA: A Globally Optimal Solution to Deformable SLAM in Closed-form}

\author{Fang Bai\affilnum{1}, Kanzhi Wu\affilnum{2} and Adrien Bartoli\affilnum{3}}

\affiliation{\affilnum{1}Under transition. Unknown affiliation to be assigned.
\affilnum{2}vivo Mobile Communication Co., Ltd., Shenzhen, China.
\affilnum{3}ENCOV, TGI, Université Clermont Auvergne, Clermont-Ferrand, France.
}

\corrauth{Fang Bai,
is now with School of Electrical and Electronic Engineering,
Nanyang Technological University, Singapore. The initial manuscript was written during the author's visit to Tongji University, Shanghai, China.}

\email{Fang.Bai@yahoo.com}

\begin{abstract}
We study the generalized Procrustes analysis (GPA), as a minimal formulation to the simultaneous localization and mapping (SLAM) problem. We propose KernelGPA, a novel global registration technique to solve SLAM in the deformable environment. We propose the concept of deformable transformation which encodes the entangled pose and deformation. We define deformable transformations using a kernel method, and show that both the  deformable transformations and the environment map can be solved globally in closed-form, up to global scale ambiguities. We solve the scale ambiguities by an optimization formulation that maximizes rigidity. We demonstrate KernelGPA using the Gaussian kernel, and validate the superiority of KernelGPA with various datasets. Code and data are available at \url{https://bitbucket.org/FangBai/deformableprocrustes}.

\vspace{10pt}
This paper has been accepted for publication in the International Journal of Robotics Research, 2023.

DOI:\href{https://doi.org/10.1177/02783649231195380}{10.1177/02783649231195380}.

\end{abstract}


\maketitle

\section{Introduction}

The simultaneous localization and mapping (SLAM), as an enabling technology for sensor localization and scene reconstruction, has witnessed a huge success in the past decade~\citet*{cadena2016past}.
However, the successful application of SLAM critically relies on the assumption of a rigid (or static) scene~\citet*{dissanayake2001solution}.

\vspace{5pt}
\noindent
\textbf{Deformable SLAM.}
Recently, researchers have started to consider SLAM in the nonrigid and dynamic cases.
While both terms seem close, they are referred to as quite different problems.
The nonrigid case typically occurs in medical or surgical applications, whereas the dynamic case occurs in outdoor applications with moving pedestrians or traffic. 
In contrast to SLAM in the dynamic case, where the movement in the scene is almost random and thus is difficult to model,
SLAM in the nonrigid environment is largely well-posed, because \textit{the deformation is typically low dimensional, or follows certain structures or constraints.}
It is therefore possible to model and estimate the deformation in the nonrigid scene, revealing the possibility of a \textit{deformable SLAM} approach.
We use the term \textit{deformable SLAM} to refer to SLAM in the nonrigid (or deformable) case.
The research of deformable SLAM is gaining popularity and has found its applications in surgical applications~\citet*{huang2021iros}.

\vspace{5pt}
\noindent
\textbf{SfT and deformable tracking.}
The first generation of deformable SLAM systems are basically based on tracking technologies.
In vision and graphics, matching a deformed shape to a \textit{given template},
termed shape-from-template (SfT) in \citet*{bartoli2015shape, malti2017elastic}, is a well researched problem.
These days, SfT can be solved under a large range of deformation models, see a brief review in Section~\ref{section. related work. deformation models}.
The SfT methods are the pillar of deformable tracking systems published in robotics, for instance
DynamicFusion~\citet*{newcombe2015dynamicfusion},
Surfelwarp~\citet*{gao2019surfelwarp},
KillingFusion~\citet*{slavcheva2017killingfusion},
SobolevFusion~\citet*{slavcheva2018sobolevfusion},
and MIS-SLAM~\citet*{song2018mis}, to name but a few.
In SLAM, the template (\textit{i.e.,}~the environment map) is never known ahead.
Thus these systems rely on an open-loop mechanism that incrementally construct the template.
As a consequence, the estimation error of these tracking systems accumulates along the trajectory, due to the lack of global feedback.
Hence, these solutions are inevitably suboptimal.

\vspace{5pt}
\noindent
\textbf{Loop-closure and global registration.}
In SLAM, the global feedback is constructed under the term of \textit{loop-closures}, which has been well understood in the case of a rigid scene.
In specific, when traveling in the scene, the sensor observes identifiable geometric points at different poses to form global feedback.
In SLAM, such a global feedback is referred to as a \textit{loop-closure}, and an identifiable geometric point in the scene as a \textit{correspondence}.
In essence, the re-observation of correspondences at different poses provides additional information, and thus reduces the uncertainty of estimation.
It must be noted that the observations are defined in local coordinate frames relative to the sensor's poses.
Thus a \textit{global registration} technique is required to fuse the observations of correspondences together.
This technique is the generalized Procrustes analysis (GPA), see Section~\ref{sec. Generalized Procrustes Analysis} for details, or structure-from-motion (SfM) if the sensors are projective cameras.
We emphasize that GPA and SfM are minimal formulations of SLAM, as they decide the poses and the scene reconstruction completely.
In the rigid case, both GPA and SfM are well solved --- that is why SLAM in the rigid case is considered a solved problem.

\vspace{5pt}
\noindent
\textbf{Global registration with deformations.}
If the scene is nonrigid, we envision that a global registration technique that handles deformations is the key to solve deformable SLAM.
Unfortunately, at this stage, the research of such techniques is rather sparse.
Some representative works include: a) the low-rank shape basis decomposition~\citet*{bregler2000recovering, xiao2006closed, dai2014simple}, b) the isometric nonrigid structure-from-motion~\citet*{parashar2017isometric}, implemented in the DefSLAM system~\citet*{lamarca2020defslam}, and more recently c) DefGPA~\citet*{bai2022ijcv}, a GPA method with the linear basis warps (LBWs), see Section~\ref{section. review global registration techniques} for a brief review and comparison.
All these methods are developed under certain assumptions about the deformations the scene undergoes.
For example, methods a) assume structural deformations (\textit{e.g.,}~gestures or facial expressions) to ensure the existence of a low-rank shape basis;
methods b) assume isometric deformations which are suitable for foldable surfaces (\textit{e.g.,}~papers or cloths).
We feature method c) which assumes \textit{smooth and low-dimensional deformations},
which is more suitable for visceral deformations occurring in surgical applications.
In this work, we contribute further to the GPA family with a novel kernel based deformation model.

\vspace{5pt}
\noindent
\textbf{Problem statement.}
We study GPA with smooth and low-dimensional deformations, termed \textit{deformable GPA}, a global registration technique for deformable SLAM.
Deformable GPA can be considered as a minimal formulation of deformable SLAM.
To make the context clear, deformable GPA is formulated under the following constraints:
\begin{enumerate}
\item[1)] \textit{No temporal information.}
We assume observations are made without sequential information, thus technologies based on tracking do not apply here.
\item[2)] \textit{No template.}
We assume a template of the scene is not available, and disallow inexact methods that incrementally construct and refine a template.
\item[3)] \textit{No aids on pose estimation.}
We assume additional information on the sensor's pose is not available.
\end{enumerate}
\textit{We assume that the only available information is observations of correspondences at different poses.}
The correspondences are used to capture two pieces of information: a) the sensor's motion, and b) the deformation of the scene.
As we shall see shortly in Section~\ref{sec. Deformable SLAM}, the sensor's pose and the deformation of the scene are entangled in deformable SLAM, making the registration extra difficult.

\vspace{5pt}
\noindent
\textbf{Contributions.}
This article is an extension to the KernelGPA method initially appeared in the proceedings of \textit{Robotics: Science and Systems (RSS)}~\citet*{Bai-RSS-22}.
Concretely, this work contains the following contributions:

\begin{enumerate}
\item
We unify the entangled poses and deformations together, and formally introduce the concept of \textit{deformable transformation}.
This way, we avoid the ambiguities in poses and deformations, because the deformable transformation is well defined and can be estimated (up to scale ambiguities).

\item
We introduce a novel deformable transformation, termed \textit{kernel based transformation (KBT)}.
As the name suggests, the KBT is motivated from the kernel method.
Compared with the LBWs in~\citet*{bai2022ijcv}, the KBT is more flexible and easier to design.

\item
We propose KernelGPA, using KBT as the deformable transformation in GPA.
We enforce implicit transformation constraints by constraining: a) the geometric center of the correspondence point-cloud to be at the origin of the coordinate frame, and b) the point-cloud covariance to be diagonalized as an unknown $\boldsymbol{\Lambda}$.

\item
We show that KernelGPA can be solved globally in closed-form up to $\sqrt{\boldsymbol{\Lambda}}$ whose diagonal elements represent the global scale ambiguities.
Our solution is based on a special eigenvalue problem first proposed in~\citet*{bai2022ijcv}.
However, the exposition of relevant proofs is more concise in this paper.

\item
We give a novel method to estimate the unknown $\sqrt{\boldsymbol{\Lambda}}$.
Compared with~\citet*{bai2022ijcv}, the novel method does not require the existence of globally visible correspondences,
thus is more suitable for partial observations occurring in SLAM.
We give an affine relaxation to obtain a closed-form $\sqrt{\boldsymbol{\Lambda}}$.

\item
We demonstrate the registration performance of KernelGPA using various datasets.
We use three 3D datasets with correspondences.
The first one comprises a set of 3D liver meshes with simulated smooth deformations.
The second one comprises a set of 3D face meshes with various facial expressions.
The third one comprises six deformed point-clouds extracted from computerized tomography (CT) data.
We will release the relevant data to foster future research.
\end{enumerate}

This article makes serveral improvements over the initial version appeared in RSS~\citet*{Bai-RSS-22}.
We have rewritten most of the text for better clarity,
for instance, the exposition of the constraints in Section~\ref{sec. Transformation Constraint} and the special eigenvalue problem in Section~\ref{section: the kernel method to solve for transformation y and map M up to unknown Lambda}.
Importantly, we have refined the method to estimate $\sqrt{\boldsymbol{\Lambda}}$ in Section~\ref{section: solve rigid transformation R, t and prior Lambda together from given map principal axes X}, and have additionally added the discussion of degeneracies in Section~\ref{section. Degeneracy}.
Lastly, we have used more advanced experiments in this version to demonstrate the usefulness of our method.

The remainder of this paper is organized as follows.
We briefly review related work on deformation models and global registration techniques in Section~\ref{section. related work}.
We introduce the concepts of deformable transformation and deformable GPA in Section~\ref{section. Formulation of Deformable SLAM and its Connection to GPA}.
We present the KBT in Section~\ref{section. transformation model using kernel methods},
and registration constraints in Section~\ref{sec. Transformation Constraint}.
We draw the connection to a special eigenvalue problem in Section~\ref{section: the kernel method to solve for transformation y and map M up to unknown Lambda},
and propose the method to estimate $\sqrt{\boldsymbol{\Lambda}}$ in Section~\ref{section: solve rigid transformation R, t and prior Lambda together from given map principal axes X}.
We discuss degeneracies in Section~\ref{section. Degeneracy}, and implementation details in Section~\ref{section. implementation details}.
We present our experimental results in Section \ref{section. experimental results},
and conclude the paper in Section \ref{section. conclusion}.

\section{Related Work}
\label{section. related work}

\subsection{Deformation Models}
\label{section. related work. deformation models}

We shall use landmarks, \textit{i.e.,}~points, as the environment representation and define deformations accordingly.
This representation has a long history in shape analysis \citet*{kendall1984shape, kilian2007geometric}.
There has been a rich class of smooth deformation models (also termed smooth warps) developed based on landmark representations,
\textit{e.g.,}~the Free-Form Deformations (FFD) \citet*{rueckert1999nonrigid, szeliski1997spline},
the Radial Basis Functions (RBF) \citet*{bookstein1989principal,fornefett2001radial}
and the Thin-Plate Spline (TPS) \citet*{duchon1976interpolation, bookstein1989principal}.
Beyond smooth models, there exist a class of models defined piece-wisely by implementing local transformations associated to a set of control points and modeling the deformations on other parts by interpolation.
Representatives of such models include
the ARAP deformation model~\citet*{sorkine2007rigid},
the embedded deformation graph~\citet*{allen2003space, sumner2007embedded}, and Lie-bodies~\citet*{freifeld2012lie}.

Beyond landmark based models,
other models based on curves \citet*{joshi2007novel, younes2008metric} or surfaces have been proposed.
Some well-known models include
level sets \citet*{osher2003level},
medial surfaces
\citet*{bouix2005hippocampal},
Q-maps \citet*{kurtek2010novel, kurtek2011elastic}, and Square Root Normal Fields (SRNF) \citet*{jermyn2012elastic, laga2017numerical}.
Some models implement an articulated skeleton structure. Representative works include the medial axis representations (M-rep) \citet*{fletcher2004principal},
and
SCAPE~\citet*{anguelov2005scape}.
We refer interested readers to the review papers \citet*{younes2012spaces, laga2018survey} for more details.

\subsection{Global Registration Techniques}
\label{section. review global registration techniques}

\vspace{5pt}
\noindent
\textbf{Generalized Procrustes analysis.}
The GPA framework was used as a fundamental technique in shape analysis to obtain an initial alignment. Both the rigid and affine transformations were recovered in the classical literature
\citet*{kendall1984shape,goodall1991procrustes, rohlf1990extensions}.
Recently, a novel GPA technique with deformation models was proposed in \citet*{bai2022ijcv}.
The deformation model in \citet*{bai2022ijcv} is termed LBWs, which includes the affine transformation and a rich class of nonlinear deformation models~\citet*{rueckert1999nonrigid, szeliski1997spline, bookstein1989principal,fornefett2001radial, bartoli2010generalized} using radial-basis functions, \textit{e.g,} the well-known TPS~\citet*{bookstein1989principal}.

The work \citet*{bai2022ijcv} is the closest to ours. However, we use a kernel method to model deformations, which is a novel deformation model compared to the LBWs used in \citet*{bai2022ijcv}.
In addition, we propose a novel method to estimate the global scale ambiguities, which does not require some correspondences to be globally visible, thus is more suitable for SLAM applications.

\vspace{5pt}
\noindent
\textbf{Nonrigid structure-from-motion.}
SfM is a well-known global registration method that handles camera projections~\citet*{Hartley2004}.
We do not consider projective cameras in this work, thus will only mention several nonrigid SfM (NRSfM) methods for references.
One line of NRSfM methods use
low-rank shape bases \citet*{bregler2000recovering, xiao2006closed, dai2014simple}.
These methods model deformations as a linear combination of the basis shapes,
which are jointly factorized by the singular value decomposition (SVD).
Another line of NRSfM methods use differential geometry, where the deformations are constrained to be isometric or conformal,
\textit{e.g.,}
the isometric NRSfM~\citet*{parashar2017isometric} which has been successfully implemented in DefSLAM~\citet*{lamarca2020defslam}.
We refer interested readers to a recent work using Cartan’s connections~\citet*{parashar2019local} and references therein.

\vspace{5pt}
\noindent
\textbf{Characterization by deformations.}
The work~\citet*{bai2022ijcv} assumes smooth and low-dimensional deformations, as implied by the usage of the LBW.
The works~\citet*{bregler2000recovering, xiao2006closed, dai2014simple} require the existence of the low-rank shape basis.
This is possible if the scene undergoes structural deformations \textit{e.g.,}~gestures or facial expressions.
The works~\citet*{parashar2017isometric, lamarca2020defslam} require the deformation to follow isometry, preserving infinitesimal rigidity on the surface of the scene.
This is usually true for foldable surfaces like papers or cloths.

In general, the visceral deformation is neither structural nor isometric, but is smooth (to avoid visceral damages) and low-dimensional (as driven by a limited number of force sources, \textit{e.g.,}~from muscles).
In this work, we propose the KBT, a smooth and low-dimensional model suitable for visceral deformations, to meet the demand of surgical applications.

\section{Formulation of Deformable SLAM and its Connection to GPA}
\label{section. Formulation of Deformable SLAM and its Connection to GPA}

\begin{figure*}[t]
\centering
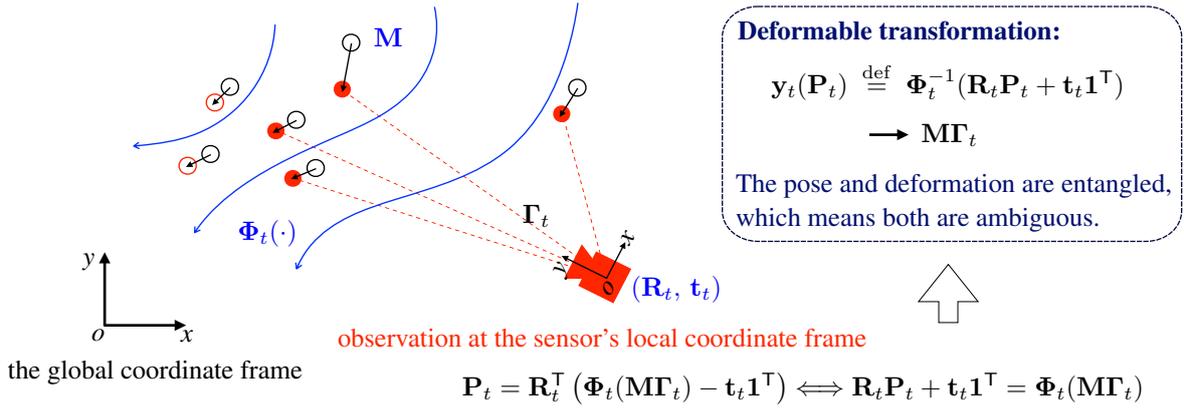
\caption{Deformable SLAM as the generalized Procrustes analysis (GPA) problem with deformable transformations.
{\color{black}
Our method is based on correspondences whose movements reflect deformations.
The movements of the correspondences, as plotted by the arrows from the black to the red circles, are driven by a low-dimensional deformation field $\matf{\Phi}_t(\cdot)$.
The unknowns are colored in blue, including a) the rigid pose $(\matf{R}_t ,\, \vecf{t}_t)$, b) the low-dimensional deformation $\matf{\Phi}_t(\cdot)$, and c) a canonical environment map $\matf{M}$.
From the observation model, we notice that a) and b) are entangled, which means we need to know one in order to infer the other. In this work, we instead propose to solve $\vecf{y}_t(\cdot)$, a unified deformable transformation which encodes both poses and deformations.
We derive that both $\vecf{y}_t(\cdot)$ and the environment map $\matf{M}$ can be estimated globally in closed-from up to $d$ scale ambiguities.
The global coordinate frame is implicitly specified by the transformation constraints to be illustrated in Figure~\ref{fig. transformation constraints illustrated by rigid motions}.}
}
\label{fig. The generalized Procrustes analysis problem}
\end{figure*}

\subsection{Deformable SLAM}
\label{sec. Deformable SLAM}

Our method is based on correspondences, and is independent of the detailed dense environment model to be used.

\vspace{5pt}
\noindent
\textbf{Environment modeling.}
We are concerned with a collection of $m$ landmarks $\matf{M} \in \mathbb{R}^{d \times m}$ residing in the $d$-dimensional environment, where $d =2$ or $d=3$.
The onboard sensor observes these landmarks in $\matf{M}$ at discrete time points $t=1,2\dots,n$.
We denote the sensor's pose at time $t$ by $(\matf{R}_t \in \mathrm{SO}(d),\, \vecf{t}_t \in \mathbb{R}^{d})$.
The sensor at $t$ observes $m_t$ partial landmarks in $\matf{M}$, denoted by $\matf{M} \matf{\Gamma}_t \in \mathbb{R}^{d \times m_t}$, with the help of a \textit{visibility matrix} $\matf{\Gamma}_t$ to be defined below.
It can be easily verified that
$
\matf{\Gamma}_t^{\trans}
\vecf{1}_{m} = \vecf{1}_{m_t}
$.

\begin{definition}[(Visibility matrix)]
We denote the identity matrix in $\mathbb{R}^{m \times m}$ as a set of standard basis vectors in $\mathbb{R}^m$:
$$
\matf{I}_m = [
\vecf{e}_1,\, \vecf{e}_2,\, \dots, \vecf{e}_m
]
\in \mathbb{R}^{m \times m}
.
$$
Obviously, $\matf{M} \matf{I}_m = \matf{M}$.
The columns of a visibility matrix $\matf{\Gamma}_t$ are constructed from the standard basis vectors in $\mathbb{R}^m$:
\begin{equation*}
\matf{\Gamma}_t =
[
\vecf{e}_{j_1},\, \vecf{e}_{j_2},\, \dots, \vecf{e}_{j_{m_t}}
]
\in \mathbb{R}^{m \times m_t}
,
\end{equation*}
where the subscripts $j_1,\, j_2, \dots, j_{m_t} \in [1 : m]$ denote the $m_t$ points visible in $\matf{P}_t$.

\end{definition}

\begin{remark}
In \citet*{bai2022ijcv}, the authors use the augmented visibility matrix $\matf{\bar{\Gamma}}_t$ defined as:
\begin{equation*}
\matf{\bar{\Gamma}}_t
=
\matf{\Gamma}_t
\matf{\Gamma}_t^{\trans}
=
\sum_{j=1}^{m_t}  \matf{e}_{t_j} \matf{e}_{t_j}^{\trans}
\in \mathbb{R}^{m \times m}
.
\end{equation*}
Such a $\matf{\bar{\Gamma}}_t$ is a diagonal matrix whose $(k, k)$-th element is $1$ if the $k$-th point in $\matf{M}$ occurs in $\matf{M}_t$, and $0$ otherwise.
$\matf{\Gamma}_t$ is obtained by deleting the columns of zeros in $\matf{\bar{\Gamma}}_t$.
\end{remark}

\begin{example}
Given $5$ points, if the first and the third points are visible, the visibility matrices are defined as:
\begin{equation*}
\matf{\Gamma} = 
\begin{bmatrix}
	1 & 0 \\
	0 & 0 \\
	0 & 1 \\
	0 & 0 \\
	0 & 0 \\
\end{bmatrix}
,
\quad
\matf{\bar{\Gamma}} = 
\matf{\Gamma} \matf{\Gamma}^{\trans}
=
\begin{bmatrix}
	1 & 0 & 0 & 0 & 0  \\
	0 & 0 & 0 & 0 & 0  \\
	0 & 0 & 1 & 0 & 0  \\
	0 & 0 & 0 & 0 & 0  \\
	0 & 0 & 0 & 0 & 0  \\	
\end{bmatrix}
.
\end{equation*}
\end{example}

\vspace{5pt}
\noindent
\textbf{Point-cloud observation of deformable environment.}
In deformable SLAM, the environment deforms over time. We denote the deformation as a time varying function $\matf{\Phi}_t(\cdot)$. In particular the deformed environment at time $t$ is:
\begin{equation*} 
\matf{\Phi}_t (\matf{M}_t)
=
\matf{\Phi}_t (\matf{M} \matf{\Gamma}_t)
.
\end{equation*}
We denote the sensor's measurement at $t$ by a point-cloud $\matf{P}_t \in \mathbb{R}^{d \times m_t}$ defined in the sensor's local coordinate frame.
In the noise-free case,
the measurement $\matf{P}_t$ at $t$ is the observation of the deformed environment $\matf{\Phi}_t (\matf{M} \matf{\Gamma}_t)$:
\begin{align}
&
\matf{P}_t = \matf{R}_t^{\trans} 
\left( \matf{\Phi}_t (\matf{M} \matf{\Gamma}_t) -  \vecf{t}_t \vecf{1}^{\trans} \right)
\notag
\\[5pt] & \Leftrightarrow
\matf{R}_t \matf{P}_t + \vecf{t}_t \vecf{1}^{\trans}
=
\matf{\Phi}_t (\matf{M} \matf{\Gamma}_t)
.
\label{eq. entangled (R, t) and Phi}
\end{align}

\vspace{5pt}
\noindent
\textbf{Composed transformation.}
From the above, we see that \textit{the deformation and the pose are entangled.}
In order to estimate one, we need to know the other (see Remark \ref{remark. the connection between pose and deformation}).
To resolve this ambiguity, we fairly assume the deformation function $\matf{\Phi}_t(\cdot)$ is invertible, and thus define $\vecf{y}_t (\cdot)$ as a composition of both the pose $(\matf{R}_t,\, \vecf{t}_t)$ and the deformation $\matf{\Phi}_t^{-1}(\cdot)$:
\begin{equation}
\label{eq. deformation transformation y_t}
\vecf{y}_t (\matf{P}_t)
\defeq
\matf{\Phi}_t^{-1} ( \matf{R}_t \matf{P}_t + \vecf{t}_t \vecf{1}^{\trans})
=
\matf{M} \matf{\Gamma}_t
.
\end{equation}
In what follows, we term $\vecf{y}_t(\cdot)$ \textit{deformable transformation}.

\vspace{5pt}
\noindent
\textbf{Deformable SLAM.}
We define deformable SLAM as the problem that estimates 1) the deformable transformations $\vecf{y}_t (\cdot)$ and 2) the environment map $\matf{M}$, using a collection of sensor measurements $(\matf{P}_t,\, \matf{\Gamma}_t)$ at time points $t=1,2\dots,n$.
Formally,
we formulate deformable SLAM as:
\begin{equation}
\label{eq. cost function of deformable SLAM}
\min\ 
\sum_{t=1}^{n} \varphi_t
\quad\mathrm{with\ }
\varphi_t = 
\left\Vert  
\vecf{y}_t (\matf{P}_t) - \matf{M} \matf{\Gamma}_t
\right\Vert_{\mathcal{F}}^2
.
\end{equation}

\begin{remark}
\label{remark. the connection between pose and deformation}
Given the pose $(\matf{R}_t,  \vecf{t}_t)$ and $\matf{M}$,
the deformation field $\matf{\Phi}_t(\cdot)$ is characterized by the vector flow:
$$
\matf{M} \matf{\Gamma}_t
\longrightarrow
\matf{R}_t \matf{P}_t + \vecf{t}_t \vecf{1}^{\trans}
.
$$
Conversely, give the deformation field $\matf{\Phi}_t(\cdot)$, the pose $(\matf{R}_t,  \vecf{t}_t)$ is characterized by the rigid Procrustes analysis.
Thus given $\matf{M}$, the disentanglement is possible once either the deformation $\matf{\Phi}_t(\cdot)$ or the pose $(\matf{R}_t,  \vecf{t}_t)$ is known.
In this work, we focus on how to solve $\matf{M}$ and $\vecf{y}_t(\cdot)$.
\end{remark}

\subsection{Generalized Procrustes Analysis}
\label{sec. Generalized Procrustes Analysis}

The deformable SLAM formulation (\ref{eq. cost function of deformable SLAM}) is essentially a GPA problem with deformable transformations, see Figure~\ref{fig. The generalized Procrustes analysis problem}.
In the classical literature, GPAs with both the rigid transformation and the affine transformation are well studied.

\vspace{5pt}
\noindent
\textbf{GPA with the rigid transformation.}
In this case, from formulation (\ref{eq. cost function of deformable SLAM}),
we define  $\vecf{y}_t(\cdot)$ as:
\begin{equation*}
\vecf{y}_t (\matf{P}_t)
\defeq
\matf{R}_t \matf{P}_t + \vecf{t}_t \vecf{1}^{\trans}
, \quad
(\matf{R}_t \in \mathrm{SO}(d),\, \vecf{t}_t \in \mathbb{R}^{d})
.
\end{equation*}
There exists a closed-form solution for the case of $n=2$ point-clouds. In general, for $n \ge 3$, the solution is computed iteratively by nonlinear least squares (NLS) optimization techniques, \textit{e.g.,}~Gauss-Newton or Levenberg-Marquardt.

\vspace{5pt}
\noindent
\textbf{GPA with the affine transformation.}
In this case, from formulation (\ref{eq. cost function of deformable SLAM}),
we define  $\vecf{y}_t(\cdot)$ as:
\begin{equation*}
\vecf{y}_t (\matf{P}_t)
\defeq
\matf{A}_t \matf{P}_t + \vecf{a}_t \vecf{1}^{\trans}
, \quad
(\matf{A}_t \in \mathbb{R}^{d \times d},\, \vecf{a}_t \in \mathbb{R}^{d})
.
\end{equation*}
The resulting GPA problem is degenerate.
The optimal solution is $\matf{A}_t = \matf{O}$, $\vecf{a}_t = \vecf{0}$, $\matf{M} = \matf{O}$, which however is useless.
In order to construct a meaningful solution, we need to build a set of constraints, for example in the rigid case the transformation preserves the distance.

\vspace{5pt}
We shall term GPA with the rigid transformation as Rigid-GPA, and GPA with the affine transformation as Affine-GPA.

\section{Deformable Transformation}
\label{section. transformation model using kernel methods}

\subsection{Linear Basis Warp}

The linear basis warp (LBW) in~\citet*{bai2022ijcv}, is a generalization of a class of deformable transformations,
\textit{e.g.,}~the free-form deformations (FFD)~\citet*{rueckert1999nonrigid, szeliski1997spline}, and the thin-plate spline (TPS)~\citet*{duchon1976interpolation, bookstein1989principal}.

\begin{definition}[(LBW in~\citet*{bai2022ijcv})]
Given a query point $\vecf{p} \in \mathbb{R}^{d}$, the LBW is defined as:
\begin{equation}
\label{eq: transformation of each point cloud in W, A, t model}
\vecf{y}_t (\vecf{p})
\defeq
\matf{W}_t^{\trans} \boldsymbol{\beta}_t (\vecf{p})
,
\quad
(\matf{W}_t \in \mathbb{R}^{l \times d})
,
\end{equation}
where $\boldsymbol{\beta}_t(\cdot):\, \mathbb{R}^{d} \rightarrow \mathbb{R}^{l}$ is an embedding to the $l$-dimensional feature space.
$\boldsymbol{\beta}_t(\cdot)$ is typically designed from radial basis functions (RBFs) \citet*{fornefett2001radial}.
\end{definition}

\vspace{5pt}
\noindent
\textbf{Regularization.}
Typically, the LBW is used together with a \textit{regularization term}:
\begin{equation}
\label{eq. transformation regularizer - LBW}
\mathcal{R}_t
= 
\mu_t 
\mathrm{tr}
\left(
\matf{W}_t^{\trans}
\matf{\Xi}_t  \matf{W}_t
\right)
,
\quad
(\mu_t > 0)
,
\end{equation}
where $\matf{\Xi}_t$ is a known matrix.
Intuitively, the regularization $\mathcal{R}_t$ acts as a penalty to control the allowed deformation.

\begin{example}
The affine transformation is a special case of the LBW where we use:
\begin{equation*}
\matf{W}_t  = 
\begin{bmatrix}
\matf{A}_t  &  \vecf{a}_t
\end{bmatrix}^{\trans}
,
\quad
\boldsymbol{\beta}_t (\vecf{p}) =
\begin{bmatrix}
\vecf{p}  \\[5pt] 1
\end{bmatrix}
.
\end{equation*}
There is no regularization in this case, $\mathcal{R}_t = 0$.
\end{example}

\begin{example}
\label{example. the basis function of TPS warp}
In case of the TPS warp, $\boldsymbol{\beta}_t(\cdot)$ is designed as:
\begin{equation*}
\boldsymbol{\beta}_t (\vecf{p}) =
\boldsymbol{ \mathcal{E} }^{\trans}
\begin{bmatrix}
\rho(\Vert \vecf{c}_1 - \vecf{p} \Vert) \\[5pt]
\rho(\Vert \vecf{c}_2 - \vecf{p} \Vert) \\[5pt]
\vdots \\[5pt]
\rho(\Vert \vecf{c}_l - \vecf{p} \Vert) \\[5pt]
\vecf{p}  \\[5pt]
 1
\end{bmatrix}
,
\end{equation*}
where $\vecf{c}_1, \cdots, \vecf{c}_l \in \mathbb{R}^{d}$ are $l$ control points, and $\rho(\cdot)$ is a scalar function called the TPS kernel function.
$\boldsymbol{ \mathcal{E} } \in \mathbb{R}^{(l+d+1) \times l }$ is a matrix constant decided from the control points and the TPS kernel function.
The TPS warp thus defined implicitly includes a free affine transformation~\citet*{bai2022ijcv}.

Matrix $\matf{\Xi}_t$ used for regularization is chosen as the bending energy matrix~\citet*{bookstein1989principal}.
With this choice,
the regularization is imposed on the nonlinear deformation only, thus leaving the implicit affine transformation free.
\end{example}

\subsection{Kernel Based Transformation}

\begin{definition}[(Kernel function)]
A \textit{kernel function} $k(\cdot , \cdot):\, \mathcal{X} \times \mathcal{X} \rightarrow \mathbb{R}$ evaluates the inner product in some feature space $\mathcal{H}$ defined by $\boldsymbol{\phi} (\cdot) : \, \mathcal{X} \rightarrow \mathcal{H}$ as:
\begin{equation*}
k(\vecf{x}_i , \vecf{x}_j)  = \left< \boldsymbol{\phi}(\vecf{x}_i),\, \boldsymbol{\phi}(\vecf{x}_j) \right>_{\mathcal{H}}
,
\quad
\vecf{x}_i , \vecf{x}_j \in \mathcal{X}
.
\end{equation*}
\end{definition}
The spirit of a kernel method is to transform all the computation related to $\boldsymbol{\phi} (\cdot)$ to the inner product $\left< \cdot , \cdot \right>_{\mathcal{H}}$, thus an explicit $\boldsymbol{\phi} (\cdot)$ will never be required.
This way, one can design a kernel method based on $k(\cdot , \cdot)$ directly.

\begin{definition}[(Kernel matrix)]
Given any $\vecf{x}_1, \dots , \vecf{x}_m \in \mathcal{X}$, and a kernel function $k(\cdot , \cdot)$, the kernel matrix $\matf{K} \in \mathbb{R}^{m \times m}$ (also called the Gram matrix) is constructed as:
\begin{equation*}
\matf{K}
=
\begin{bmatrix}
k(\vecf{x}_1, \vecf{x}_1) & \cdots & k(\vecf{x}_1, \vecf{x}_m) \\[5pt]
\vdots & \ddots & \vdots \\[5pt]
k(\vecf{x}_m, \vecf{x}_1) & \cdots & k(\vecf{x}_m, \vecf{x}_m) 
\end{bmatrix}
.
\end{equation*}
\end{definition}

\begin{lemma}[(\citet*{Shawe-Taylor2004Kernel})]
\label{lemma. the classic result between kernel function and feature mapping}
If $k(\vecf{x}_i , \vecf{x}_j)  = \left< \boldsymbol{\phi}(\vecf{x}_i),\, \boldsymbol{\phi}(\vecf{x}_j) \right>_{\mathcal{H}}$ for some feature mapping $\boldsymbol{\phi}(\cdot) : \, \mathcal{X} \rightarrow \mathcal{H}$, then for any $\vecf{x}_1, \dots , \vecf{x}_m \in \mathcal{X}$ the kernel matrix $\matf{K}$ is \textit{symmetric positive (semi-)definite}.
The converse is also true.
If the kernel matrix $\matf{K}$ constructed from a kernel function $k(\cdot , \cdot)$ is symmetric positive (semi-)definite for any $\vecf{x}_1, \dots , \vecf{x}_m \in \mathcal{X}$,
then there exists a feature mapping $\boldsymbol{\phi}(\cdot) : \, \mathcal{X} \rightarrow \mathcal{H}$ such that $k(\vecf{x}_i , \vecf{x}_j)  = \left< \boldsymbol{\phi}(\vecf{x}_i),\, \boldsymbol{\phi}(\vecf{x}_j) \right>_{\mathcal{H}}$.
\end{lemma}

\begin{definition}[(Positive (semi-)definite kernel)]
A kernel function $k(\cdot , \cdot)$ that ensures the symmetric positive (semi-)definiteness of $\matf{K}$ for any $\vecf{x}_1, \dots , \vecf{x}_m \in \mathcal{X}$ is called a \textit{positive (semi-)definite kernel}.
\end{definition}

By Lemma~\ref{lemma. the classic result between kernel function and feature mapping}, a positive (semi-)definite kernel function $k(\cdot , \cdot)$ implicitly determines a feature mapping $\boldsymbol{\phi}(\cdot)$ by the inner product $\left< \cdot , \cdot \right>_{\mathcal{H}}$ in some Hilbert space $\mathcal{H}$.
Such an $\mathcal{H}$ induced from the kernel function $k(\cdot , \cdot)$ is termed the \textit{reproducing kernel Hilbert space (RKHS)}.
For more details, we refer interested readers to Chapter 3 of the book~\citet*{Shawe-Taylor2004Kernel}.

\vspace{5pt}
\noindent
\textbf{Function representer.}
Given $m$ data points $\vecf{x}_1, \dots , \vecf{x}_m \in \mathcal{X}$, and a positive (semi-)definite kernel $k(\cdot , \cdot)$, we parameterize a function $f (\cdot) :\, \mathcal{X} \rightarrow \mathbb{R}$ as an expansion of kernel functions $k(\vecf{x}_j,  \cdot )$ over all data points:
\begin{equation}
\label{eq. function expansion}
f (\vecf{x}) = \sum_{j=1}^{m}
\alpha_j
k(\vecf{x}_j,  \vecf{x} )
,
\quad
\forall \vecf{x}\, \in \mathrm{domain}\ f
.
\end{equation}
Such an expansion is motivated from the \textit{reproducing property} of the RKHS,
and its expressiveness is backed by the representer theorem in \citet*{scholkopf2001generalized}.

\begin{assumption}
We assume positive definite kernel $k(\cdot , \cdot)$.
Thus the constructed kernel matrix is positive definite.
\end{assumption}

\begin{definition}[(Kernel based transformation)]
Given the point-cloud $\matf{P}_t = \left[
\vecf{p}_1,\,
\vecf{p}_2,\,
\dots,
\vecf{p}_{m_t}
\right]
\in
\mathbb{R}^{d \times m_t}
$,
and a query point $\vecf{p} \in \mathbb{R}^{d}$,
we propose a KBT, as:
\begin{multline}
\label{eq: transformation of a point by affine-kernel}
\vecf{y}_t (\vecf{p})
\defeq
\matf{A}_t \vecf{p} + \vecf{a}_t
+
\boldsymbol{\Omega}_t^{\trans}
\vecf{k}_t(\vecf{p})
\\[5pt]  \quad\quad\quad 
(\matf{A}_t \in \mathbb{R}^{d \times d},\, \vecf{a}_t \in \mathbb{R}^{d},\, \boldsymbol{\Omega}_t \in \mathbb{R}^{m_t \times d})
,
\end{multline}
where:
$$
\vecf{k}_t(\vecf{p}) = 
\begin{bmatrix}
k(\vecf{p}_1,\, \vecf{p}) \\[5pt]
\vdots \\[5pt]
k(\vecf{p}_{m_t},\, \vecf{p}) 
\end{bmatrix}
,
$$
with $k(\cdot, \cdot)$ a user specified positive (semi-)definite kernel.
\end{definition}

The deformable transformation $\vecf{y}_t (\cdot)$ thus constructed contains two components: the affine part $(\matf{A}_t ,\, \vecf{a}_t)$ and the deformation part $\boldsymbol{\Omega}_t^{\trans} \vecf{k}_t(\cdot)$.
The deformation part is an obvious extension from the expansion (\ref{eq. function expansion}) to each of $x-$, $y-$, and $z-$coordinates.
The motivation behind the affine part is that we require $\vecf{y}_t (\cdot)$ to model global orientations and translations.
Such information can indeed be lost in the kernel construction, for example if we choose
$k(\vecf{x}_i , \vecf{x}_j) = \kappa ( \Vert \vecf{x}_i - \vecf{x}_j \Vert_2 )$.
In addition, orientations and translations are global, meaning consistently applied to each point, which is not emphasized in the deformation part.

\vspace{5pt}
\noindent
\textbf{Regularization.}
For the KBT (\ref{eq: transformation of a point by affine-kernel}),
we propose to use the following regularization:
\begin{equation}
\label{eq. transformation regularizer - kernel}
\mathcal{R}_t
= 
\mu_t
\mathrm{tr}
\left(
\boldsymbol{\Omega}_t^{\trans}
\matf{K}_t  \boldsymbol{\Omega}_t
\right)
,
\quad
(\mu_t > 0)
,
\end{equation}
where we define the \textit{kernel matrix} $\matf{K}_t \in \mathbb{R}^{m_t \times m_t}$:
\begin{equation}
\label{eq. definition of K_t from P_t}
\matf{K}_t  = 
\begin{bmatrix}
k(\vecf{p}_1,\, \vecf{p}_1) & \cdots & k(\vecf{p}_1,\, \vecf{p}_{m_t}) \\[5pt]
\vdots  & \ddots & \vdots \\[5pt]
k(\vecf{p}_{m_t},\, \vecf{p}_1) & \cdots & k(\vecf{p}_{m_t},\, \vecf{p}_{m_t})
\end{bmatrix}
.
\end{equation}
The motivation of this regularization will be given shortly, near equation (\ref{eq. transformation regularizer - relation}).

\subsection{Operating on the Point-cloud}

Given the point-cloud $\matf{P}_t = \left[
\vecf{p}_1,\,
\vecf{p}_2,\,
\dots,
\vecf{p}_{m_t}
\right]
\in
\mathbb{R}^{d \times m_t} $,
we apply the deformable transformation $\vecf{y}_t(\cdot)$ to each point of $\matf{P}_t$ in sequence:
$$
\vecf{y}_t (\matf{P}_t) 
 \defeq
\left[
\begin{matrix}
\vecf{y}_t(\vecf{p}_1), & \vecf{y}_t(\vecf{p}_2), & \cdots & \vecf{y}_t(\vecf{p}_{m_t})
\end{matrix}
\right]
.
$$

\vspace{5pt}
\noindent
\textit{For the LBW}, the result is:
\begin{align}  
\vecf{y}_t (\matf{P}_t) 
& \defeq 
\matf{W}_t^{\trans} \,
\underbrace{
\begin{bmatrix}
\boldsymbol{\beta}_t(\vecf{p}_1) &
\boldsymbol{\beta}_t(\vecf{p}_2) &
\dots &
\boldsymbol{\beta}_t(\vecf{p}_{m_t})
\end{bmatrix}
}_{ \boldsymbol{\mathcal{B}}_t(\matf{P}_t) }
\notag
\\[2pt]  & \defeq
\matf{W}_t^{\trans}
\boldsymbol{\mathcal{B}}_t(\matf{P}_t)
,
\quad
(\matf{W}_t \in \mathbb{R}^{l \times d})
.
\end{align}

\vspace{5pt}
\noindent
\textit{For the KBT}, the result is:
\begin{multline}
\label{eq. transformation model - kernel - affine}
\vecf{y}_t (\matf{P}_t) 
\defeq
\matf{A}_t \matf{P}_t + \vecf{a}_t \vecf{1}^{\trans} + \boldsymbol{\Omega}_t^{\trans}
\matf{K}_t
, \\[5pt]
(\matf{A}_t \in \mathbb{R}^{d \times d},\, \vecf{a}_t \in \mathbb{R}^{d},\, \boldsymbol{\Omega}_t \in \mathbb{R}^{m_t \times d})
,
\end{multline}
where $\matf{K}_t$ is defined in equation (\ref{eq. definition of K_t from P_t}).

\subsection{Derivation of the KBT from the LBW}

\vspace{5pt}
\noindent
\textbf{Derivation of the deformation part.}
We consider the task of transforming the point-cloud $\matf{P}_t$ to a given target point-cloud $\matf{Z}_t$, using the LBW and an identity regularization term.
This task can be formulated as minimizing a regression cost:
\begin{equation}
\label{eq: cost of linear regression model with regularization}
\eta_t (\matf{W}_t) = \left\Vert \matf{W}_t^{\trans} \boldsymbol{\mathcal{B}}_t (\matf{P}_t) 
- \matf{Z}_t \right\Vert_{\mathcal{F}}^2
+ \mu_t
\left\Vert \matf{W}_t \right\Vert_{\mathcal{F}}^2
.
\end{equation}
Cost (\ref{eq: cost of linear regression model with regularization}) is convex.
Its global minimum is attained when the gradient vanishes:
$$
\frac{\partial  \eta_t }{\partial \matf{W}_t } = \matf{O}
.
$$
After computing the matrix differential, and with some trivial matrix calculations, we rewrite the above equation as:
\begin{align}
\matf{W}_t & = 
\boldsymbol{\mathcal{B}}_t(\matf{P}_t) \,
\underbrace{\left(
	- \frac{1}{\mu_t} \matf{W}_t^{\trans} \boldsymbol{\mathcal{B}}_t(\matf{P}_t) 
	+ \frac{1}{\mu_t}  \matf{Z}_t
	\right)^{\trans}}_{\boldsymbol{\Omega}_t}
\notag
\\[0pt]
& \defeq
\boldsymbol{\mathcal{B}}_t(\matf{P}_t) \, \boldsymbol{\Omega}_t
.
\label{eq. the first-order optimality condition of linear regression cost in W}
\end{align}
In this form, $\boldsymbol{\Omega}_t$ is called the \textit{dual variable},
as it converts the LBW to the KBT as:
\begin{align}
\label{eq: the dual formulation of linear regression model}
\matf{W}_t^{\trans} \boldsymbol{\mathcal{B}}_t(\matf{P}_t)
& =
\boldsymbol{\Omega}_t^{\trans} \,
\underbrace{\boldsymbol{\mathcal{B}}_t(\matf{P}_t)^{\trans}
\boldsymbol{\mathcal{B}}_t(\matf{P}_t)}_{\matf{K}_t}
\defeq
\boldsymbol{\Omega}_t^{\trans} \,
\matf{K}_t
,
\\[0pt]
&
\mathrm{where}\quad
\matf{K}_t =
\boldsymbol{\mathcal{B}}_t(\matf{P}_t)^{\trans}
\boldsymbol{\mathcal{B}}_t(\matf{P}_t)
\label{eq: definition of kernel matrix}
.
\end{align}
Note that the dimension $l$ of the feature space of $\boldsymbol{\beta}_t(\cdot)$ may go to infinity; however we can still express $\boldsymbol{\Omega}_t^{\trans} \,
\boldsymbol{\mathcal{B}}_t(\matf{P}_t)$ as $\boldsymbol{\Omega}_t^{\trans}
\matf{K}_t$ within $m_t$ points in the kernel based model.

\vspace{5pt}
\noindent
\textbf{Derivation of the regularization.}
From equation (\ref{eq. the first-order optimality condition of linear regression cost in W}),
the regularization $\mu_t
\left\Vert \matf{W}_t \right\Vert_{\mathcal{F}}^2$ can be reformulated with respect to the dual variable $\boldsymbol{\Omega}_t$ and the kernel matrix $\matf{K}_t$ as:
\begin{equation}
\label{eq. transformation regularizer - relation}
\mu_t
\left\Vert \matf{W}_t \right\Vert_{\mathcal{F}}^2
= \mu_t
\left\Vert 
\boldsymbol{\mathcal{B}}_t(\matf{P}_t) \boldsymbol{\Omega}_t
\right\Vert_{\mathcal{F}}^2
= 
\mu_t
\mathrm{tr}
\left(
\boldsymbol{\Omega}_t^{\trans}
\matf{K}_t  \boldsymbol{\Omega}_t
\right)
,
\end{equation}
which is how we obtain the regularization in equation (\ref{eq. transformation regularizer - kernel}).

\begin{remark}[(The independent affine transformation)]
In the KBT (\ref{eq: transformation of a point by affine-kernel}), we include an independent affine transformation, which is different from the LBWs.
This is because for the LBWs, the affine transformation is typically implemented by the design of the basis function $\boldsymbol{\beta}_t(\cdot)$.
However, for the KBT, the kernel function $k(\cdot , \cdot)$ uniformly decides the elements in $\matf{K}_t$, excluding the possibility to use a handcrafted affine transformation.
The usage of the independent affine transformation can be equivalently thought of as singling out the affine part in the LBW (\ref{eq: transformation of each point cloud in W, A, t model}) as:
\begin{multline*}
\vecf{y}_t (\vecf{p})
\defeq
\matf{A}_t \vecf{p} + \vecf{a}_t 
+
\matf{W}_t^{\trans} \boldsymbol{\beta}_t (\vecf{p})
,
\\
\quad
(\matf{A}_t \in \mathbb{R}^{d \times d},\, \vecf{a}_t \in \mathbb{R}^{d},\, \matf{W}_t \in \mathbb{R}^{l \times d})
.
\end{multline*}
In this form, $\boldsymbol{\beta}_t(\cdot)$ only models deformations.
By regularization (\ref{eq. transformation regularizer - relation}), we see the affine part is free, which is in the same spirit of common LBWs.
\end{remark}

\begin{figure*}[t]
\centering
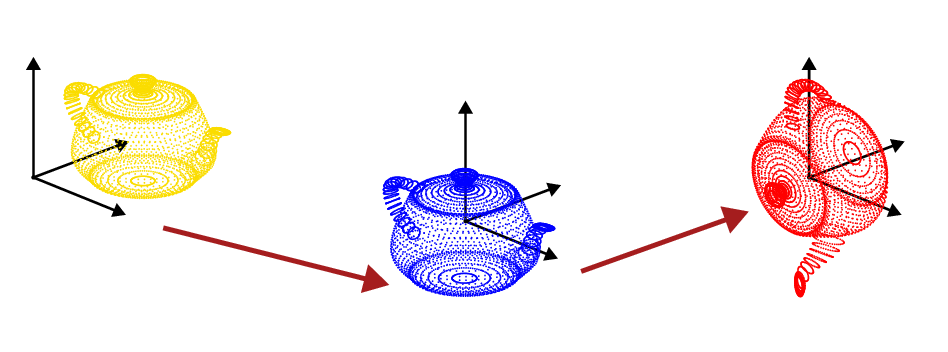
\caption{The proposed constraints
$\matf{M} \vecf{1} = \vecf{0}$, $\matf{M} \matf{M}^{\trans} = \boldsymbol{\Lambda}$ with an unknown diagonal matrix $\boldsymbol{\Lambda}$,
allow all possible geometries (\textit{i.e.,}~shapes) for $\matf{M}$.
This is explained as follows.
We assume $\matf{M}$ to be an arbitrary point-cloud, and denote $\mathbb{C}\mathrm{ov}(\matf{M}) = \matf{Q} \boldsymbol{\Lambda} \matf{Q}^{\trans}$ the eigenvalue decomposition of its point-cloud covariance.
Then we can rigidly transform $\matf{M}$ to $\matf{M}_r$ which has exactly the same geometry as $\matf{M}$ while $\matf{M}_{r} \vecf{1} = \vecf{0}$,
$\matf{M}_{r} \matf{M}_{r}^{\trans} = \boldsymbol{\Lambda}$.
Note that as $\matf{M}$ is unknown, we do not know the eigenvalues $\boldsymbol{\Lambda}$.
Fortunately $\boldsymbol{\Lambda}$ is never required explicitly to derive the globally optimal solution to the GPA formulation, and thus can be estimated afterwards.
{\color{black}The constraints $\matf{M} \vecf{1} = \vecf{0}$, $\matf{M} \matf{M}^{\trans} = \boldsymbol{\Lambda}$ implicitly specify the global coordinate frame in Figure~\ref{fig. The generalized Procrustes analysis problem}, by requiring $\matf{M}$ positioned this way.
}
}
\label{fig. transformation constraints illustrated by rigid motions}
\end{figure*}

\section{Transformation Constraint}
\label{sec. Transformation Constraint}

\begin{definition}[(Zero-centered point-cloud)]
\label{definition. zero-centered point-cloud}
A point-cloud $\matf{M}$ is zero-centered if and only if $\matf{M} \vecf{1} = \vecf{0}$.
In particular, $\matf{\bar{M}}$ is a zero-centered point-cloud of $\matf{M}$ where:
$$
\matf{\bar{M}}
=
\matf{M} - \frac{1}{m}\matf{M}\vecf{1}\vecf{1}^{\trans}
.
$$
\end{definition}

\begin{definition}[(Point-cloud covariance)]
\label{definition. point-cloud covariance}
We define the point-cloud covariance
$\mathbb{C}\mathrm{ov}(\matf{M}) = \matf{\bar{M}} \matf{\bar{M}}^{\trans}$
with
$
\matf{\bar{M}}
=
\matf{M} - \frac{1}{m}\matf{M}\vecf{1}\vecf{1}^{\trans}
$
being the zero-centered point-cloud of $\matf{M}$.
\end{definition}

We can simplify $\mathbb{C}\mathrm{ov}( \matf{M} )$ in Definition~\ref{definition. point-cloud covariance} to $\matf{M} \matf{M}^{\trans}$ by requiring $\matf{M}$ to be zero-centered as in Definition~\ref{definition. zero-centered point-cloud}.

\begin{lemma}[(Lemma 2 in \citet*{bai2022ijcv})]
\label{lemma. cov(RM + t) = R cov(M) R^T}
For any $\matf{M}$, any rotation $\matf{R}$ and any translation $\vecf{t}$, we have:
\begin{equation}
\mathbb{C}\mathrm{ov}(\matf{R} \matf{M} + \vecf{t} \vecf{1}^{\trans})
=
\matf{R} \,
\mathbb{C}\mathrm{ov}(\matf{M}) \,
\matf{R}^{\trans}
.
\end{equation}
\end{lemma}

Lemma~\ref{lemma. cov(RM + t) = R cov(M) R^T} shows that:
a) $\mathbb{C}\mathrm{ov}(\matf{M})$ is only related to rotations, and
b) the eigenvalues of $\mathbb{C}\mathrm{ov}(\matf{M})$ are preserved when applying rotations to $\matf{M}$.

\begin{definition}[(Eigenvalues of point-cloud covariance)]
\label{definition. eigenvalues of point-cloud covariance}
We denote $ \boldsymbol{\Lambda}  = \mathbf{diag}(\lambda_1, \dots, \lambda_d)$, where $\lambda_1 \ge \cdots \ge \lambda_d \ge 0$ are the $d$ eigenvalues of the point-cloud covariance $\mathbb{C}\mathrm{ov}(\matf{M})$.
\end{definition}

In addition, Lemma~\ref{lemma. cov(RM + t) = R cov(M) R^T} provides a means to diagonalize $\mathbb{C}\mathrm{ov}( \matf{M} )$ by rotating the point-cloud $\matf{M}$.
We consider the eigenvalue decomposition:
$$
\begin{aligned}
\mathbb{C}\mathrm{ov}(\matf{M}) 
 = 
 \matf{Q} \boldsymbol{\Lambda} \matf{Q}^{\trans}
& =
\sum_{k=1}^{d}
\lambda_k \vecf{q}_k \vecf{q}_k^{\trans}
,
\\[0pt]
& \mathrm{where}\quad
\matf{Q} \defeq
\begin{bmatrix}
\vecf{q}_1 & \cdots  & \vecf{q}_d
\end{bmatrix}
.
\end{aligned}
$$
It is always possible to have $\matf{Q}  \in \mathrm{SO}(d)$ by flipping the signs of $\vecf{q}_{k}$.
In Lemma~\ref{lemma. cov(RM + t) = R cov(M) R^T}, if we use $\matf{R} = \matf{Q}^{\trans}$, then:
$$
\mathbb{C}\mathrm{ov}(\matf{Q}^{\trans} \matf{M}) 
=
\matf{Q}^{\trans}
\mathbb{C}\mathrm{ov}(\matf{M})
\matf{Q}
= 
\boldsymbol{\Lambda}
,$$
where $\mathbb{C}\mathrm{ov}(\matf{Q}^{\trans} \matf{M})$ is of diagonal form.

We present the general result below, and give an illustration in Figure~\ref{fig. transformation constraints illustrated by rigid motions}.

\begin{theorem}
\label{proposition. any point-cloud can satisfy the transformation constraints}
For any $\matf{M}$, it is always possible to find a rigid transformation $(\matf{R},\, \vecf{t})$ such that the rigidly transformed $\matf{M}_r$:
$$
\matf{M}_r =  \matf{R} \matf{M} + \vecf{t} \vecf{1}^{\trans},
\qquad
\matf{R}  \in \mathrm{SO}(d),\, \vecf{t}  \in \mathbb{R}^{d},
$$
is a zero-centered point-cloud and
has a diagonal form point-cloud covariance:
$$
\matf{M}_{r} \vecf{1} = \vecf{0}
,\quad
\mathbb{C}\mathrm{ov}(\matf{M}_r)
=
\matf{M}_{r}
\matf{M}_{r}^{\trans}
=
 \boldsymbol{\Lambda}
,
$$
where $\boldsymbol{\Lambda}$, as defined in Definition \ref{definition. eigenvalues of point-cloud covariance}, contains the eigenvalues of the point-cloud covariances $\mathbb{C}\mathrm{ov}(\matf{M})$ and $\mathbb{C}\mathrm{ov}(\matf{M}_r)$.
\end{theorem}
\begin{proof}
It suffices to set $\matf{R} = \matf{Q}^{\trans}$
and $\vecf{t} = - \frac{1}{m} \matf{Q}^{\trans} \matf{M}\vecf{1}$.
\end{proof}

We are interested in the geometry \textit{i.e.,}~the shape, of point-cloud $\matf{M}$, discarding its position and orientation in the global coordinate system.
Thus we propose to solve for an $\matf{M}$ that is zero-centered with diagonal covariance:
\begin{numcases}
{\mathrm{constraints}}
\matf{M} \vecf{1} = \vecf{0}  \label{eq. gauge constrait 1} \\[5pt]
\matf{M} \matf{M}^{\trans} = \boldsymbol{\Lambda} = \mathbf{diag}(\lambda_1, \dots, \lambda_d)
\label{eq. gauge constrait 2}
,
\end{numcases}
where $\lambda_1 \ge \cdots \ge \lambda_d \ge 0$ are \textit{unknown parameters} representing the eigenvalues of the point-cloud covariance.

\begin{remark}
The constraints $\matf{M} \vecf{1} = \vecf{0}$, $\matf{M} \matf{M}^{\trans} = \boldsymbol{\Lambda}$ allow $\matf{M}$ to take all possible geometries,
as implied by Theorem~\ref{proposition. any point-cloud can satisfy the transformation constraints}.
\end{remark}

\section{Globally Optimal Solution}
\label{section: the kernel method to solve for transformation y and map M up to unknown Lambda}

\subsection{Formulation of Deformable SLAM}

Using deformable transformation (\ref{eq. transformation model - kernel - affine}) and regularization (\ref{eq. transformation regularizer - kernel}), we write the cost function at time $t$ as:
\begin{align}
\varphi_t (\matf{A}_t, \vecf{a}_t, \boldsymbol{\Omega}_t, \matf{M})  
& = \left\Vert 
\matf{A}_t \matf{P}_t + \vecf{a}_t \vecf{1}^{\trans}
+ \boldsymbol{\Omega}_t^{\trans}
\matf{K}_t
- \matf{M} \matf{\Gamma}_t \right\Vert_{\mathcal{F}}^2
\notag
\\[5pt] 
& + 
 \mu_t
\mathrm{tr}
\left(
\boldsymbol{\Omega}_t^{\trans}
\matf{K}_t  \boldsymbol{\Omega}_t
\right)
.
\label{eq. cost of phi_t in all variables}
\end{align}
Then we use constraints
$\matf{M} \vecf{1} = \vecf{0}$
and
$\matf{M} \matf{M}^{\trans} = \boldsymbol{\Lambda}$ to implicitly specify the free coordinate frame where to express the solution.
We complete formulation (\ref{eq. cost function of deformable SLAM}) as the following optimization problem:
\begin{equation}
\label{eq: deformable SLAM formulation using Kernel and dual variable}
\begin{aligned}
\min_{ \{ \matf{A}_t,\,\vecf{a}_t,\,\boldsymbol{\Omega}_t \}, \, \matf{M}}
\quad 
& \sum_{t=1}^{n} \varphi_t (\matf{A}_t,\,\vecf{a}_t,\,\boldsymbol{\Omega}_t, \, \matf{M})
\\[2pt]  & 
\mathrm{s.t.}\ 
\matf{M} \vecf{1} = \vecf{0},\ 
\matf{M} \matf{M}^{\trans} = \boldsymbol{\Lambda}
.
\end{aligned}
\end{equation}

In the remainder of this section, we derive the globally optimal solution to problem (\ref{eq: deformable SLAM formulation using Kernel and dual variable}) in function of the unknown $\boldsymbol{\Lambda}$.
We will recast problem (\ref{eq: deformable SLAM formulation using Kernel and dual variable}) as a special eigenvalue problem, and derive the solution in closed-form, see \citet*{bai2022ijcv} for affine models and TPS warps (a brief recapitulation is provided in Appendix~\ref{appendix. results of GPA formulation using the LBW}).

\subsection{Reduced Problem in $\matf{M}$}
\label{section. Reduced Problem in M}

We notice that in problem (\ref{eq: deformable SLAM formulation using Kernel and dual variable}), the transformation parameters $\matf{A}_t$, $\vecf{a}_t$ and $\boldsymbol{\Omega}_t$ are linearly dependent on $\matf{M}$.
This presents a separable structure and allows us to reduce the optimization to $\matf{M}$ only using the variable projection method \citet*{golub2003separable}.

\vspace{5pt}
\noindent
\textbf{The linear dependence of $\matf{A}_t$, $\vecf{a}_t$ and $\boldsymbol{\Omega}_t$ on $\matf{M}$.}
We first notice that in problem (\ref{eq: deformable SLAM formulation using Kernel and dual variable}), given $\matf{M}$, the summands in the cost function become independent.
This allows us to derive the dependence of $\matf{A}_t$, $\vecf{a}_t$ and $\boldsymbol{\Omega}_t$ on $\matf{M}$ by solving a linear least squares (LLS) optimization from cost (\ref{eq. cost of phi_t in all variables}):
\begin{equation}
\min_{ \{ \matf{A}_t,\,\vecf{a}_t,\,\boldsymbol{\Omega}_t \} }
\quad
\varphi_t (\matf{A}_t,\,\vecf{a}_t,\,\boldsymbol{\Omega}_t, \, \matf{M})
,
\quad
\mathrm{given}\ \matf{M}
.
\label{eq. optimization At, at, Omegat given M}
\end{equation} 
With some trivial calculations, see appendix~\ref{appendix. Derivation of linear dependence of A, t, Omega on M}, we write:
\begin{multline}
\label{eq. At, at, Omegat in M}
\left[
[\matf{A}_t, \, \vecf{a}_t], \, \boldsymbol{\Omega}_t^{\trans} 
\right]
= 
\matf{M} \matf{\Gamma}_t
\begin{bmatrix}
  \matf{\tilde{P}}_t^{\trans} &  \matf{K}_t
\end{bmatrix}
\boldsymbol{\Delta}_t^{\dagger}
\\ +
\matf{F}_t
\left(
\matf{I} -  \boldsymbol{\Delta}_t \boldsymbol{\Delta}_t^{\dagger}
\right)
,
\end{multline}
where $\matf{\tilde{P}}_t = [ \matf{P}_t^{\trans} ,\, \vecf{1} ]^{\trans}$, and $\matf{F}_t  \in \mathbb{R}^{d \times (m_t+d+1)}$ is a free matrix.
$\boldsymbol{\Delta}_t^{\dagger}$ is the Moore–Penrose pseudo-inverse of a positive definite (or positive semi-definite) matrix $\boldsymbol{\Delta}_t$ defined as:
\begin{equation*}
\boldsymbol{\Delta}_t
\defeq
\begin{bmatrix}
\matf{\tilde{P}}_t \matf{\tilde{P}}_t^{\trans} & \matf{\tilde{P}}_t \matf{K}_t \\[5pt]
\matf{K}_t \matf{\tilde{P}}_t^{\trans} & \matf{K}_t \matf{K}_t + \mu_t \matf{K}_t
\end{bmatrix}
.
\end{equation*}

\begin{remark}
The free matrix $\matf{F}_t$ is used to describe general solutions of the LLS problem (\ref{eq. optimization At, at, Omegat given M}), in case that $\boldsymbol{\Delta}_t$ is rank deficient (and thus not invertible).
If $\boldsymbol{\Delta}_t$ is positive definite, then $\matf{F}_t$ is not required since
$\matf{F}_t
(
\matf{I} -  \boldsymbol{\Delta}_t \boldsymbol{\Delta}_t^{\dagger}
) = \matf{O}$.
\end{remark}

\begin{lemma}
\label{lemma. Kt is PD iff PP is PD}
If $\matf{K}_t$ is positive definite and $\mu_t > 0$, then $\boldsymbol{\Delta}_t$ is positive definite if and only if $\matf{\tilde{P}}_t \matf{\tilde{P}}_t^{\trans}$ is positive definite.
\end{lemma}
\begin{proof}
See Appendix~\ref{section. positive definiteness between Delta and PP}.
\end{proof}
Otherwise stated, $\boldsymbol{\Delta}_t$ is invertible if and only if $\matf{\tilde{P}}_t $ has full row rank which is the case if the point-cloud $\matf{P}_t$ is not degenerate, \textit{e.g.,}~not flat if $d=3$ (namely residing in a plane in the 3D space) or not a line if $d=2$.

\vspace{5pt}
\noindent
\textbf{The reduced problem in $\matf{M}$.}
Substituting equation (\ref{eq. At, at, Omegat in M}) into the cost (\ref{eq. cost of phi_t in all variables}), we obtain a cost with respect to $\vecf{M}$ only, denoted as $\varphi_t (\matf{M})$.
With some trivial calculations, see appendix~\ref{appendix. Derivation of Qt and phit}, we show:
\begin{align*}
\varphi_t (\matf{M})  
=
\mathrm{tr} 
\left(
\matf{M} \matf{\Gamma}_t 
\matf{Q}_t
\matf{\Gamma}_t^{\trans} \matf{M}^{\trans}
\right)
,
\end{align*}
where $\matf{Q}_t$ is independent of the free matrix $\matf{F}$ occurring in equation (\ref{eq. At, at, Omegat in M}), defined as:
\begin{equation*}
\label{eq. the 1st form of Q_t}
\matf{Q}_t
\defeq
 \matf{I} - 
\begin{bmatrix}
  \matf{\tilde{P}}_t^{\trans} &  \matf{K}_t
\end{bmatrix}
\boldsymbol{\Delta}_t^{\dagger}
\begin{bmatrix}
  \matf{\tilde{P}}_t \\[5pt]  \matf{K}_t
\end{bmatrix}
.
\end{equation*}
Lastly problem (\ref{eq: deformable SLAM formulation using Kernel and dual variable}) is reduced to:
\begin{equation}
\label{eq: the minimization problem in Map only}
\begin{aligned}
\min_{\matf{M}} \quad 
&
\mathrm{tr}
\left(
\matf{M}
\boldsymbol{\mathcal{Q}}
\matf{M}^{\trans}
\right)
\\[2pt] &
\mathrm{s.t.}\ 
\matf{M} \vecf{1} = \vecf{0},\ 
\matf{M} \matf{M}^{\trans} = \boldsymbol{\Lambda}
,
\end{aligned}
\end{equation}
with:
$$
\boldsymbol{\mathcal{Q}} = \sum_{t=1}^{n} \matf{\Gamma}_t \matf{Q}_t \matf{\Gamma}_t^{\trans}
.$$
Problem (\ref{eq: the minimization problem in Map only}) is an optimization problem with respect to $\matf{M}$ only.
In particular, problem (\ref{eq: the minimization problem in Map only}) can be solved globally in closed-form if the all-one vector $\vecf{1}$ is an eigenvector of $\boldsymbol{\mathcal{Q}}$.

\vspace{5pt}
\noindent
\textbf{Properties of $\matf{Q}_t$ and $\boldsymbol{\mathcal{Q}}$.}
We can work out closed-form expressions for $\boldsymbol{\Delta}_t^{\dagger}$
using the Schur complement~\citet*{Gallier2010SchurComplement},
see Appendix~\ref{appendix. expansion of Pseudo inverse of Delta_t}.
With some trivial calculations, see~Appendix~\ref{appendix. derivation of constant Ft, Gt and Qt}, we show that $\matf{Q}_t$ can be rewritten as follows:
\begin{equation}
\matf{Q}_t
=
\left( \matf{I} - \boldsymbol{\mathcal{P}}_t \right)
 -
\left( \matf{I} - \boldsymbol{\mathcal{P}}_t \right)
\matf{K}_t \matf{S}_t^{-1} \matf{K}_t
\left( \matf{I} - \boldsymbol{\mathcal{P}}_t \right)
,
\label{eq. expression of Qt}
\end{equation}
with $ \boldsymbol{\mathcal{P}}_{t} \defeq \matf{\tilde{P}}_t^{\trans} ( \matf{\tilde{P}}_t \matf{\tilde{P}}_t^{\trans} )^{\dagger}  \matf{\tilde{P}}_t $, and:
\begin{equation}
\matf{S}_t \defeq \matf{K}_t \left( \matf{I} - \boldsymbol{\mathcal{P}}_t \right) \matf{K}_t + \mu_t \matf{K}_t
,
\label{eq. expression of St}
\end{equation}
being symmetric positive definite (and thus invertible), since we assume $ \matf{K}_t $ is chosen positive definite and $\mu_t > 0$.

\begin{proposition}
\label{proposition. properties of Q_t}
If $ \matf{K}_t $ is chosen positive definite and $\mu_t > 0$, then
$\matf{Q}_t$ is symmetric positive semidefinite where:
\begin{itemize}
\item $\matf{I} \succeq \matf{I} - \boldsymbol{\mathcal{P}}_t   \succeq  \matf{Q}_t \succeq \matf{O}$ 
\vspace{8pt}
\item $\matf{Q}_t \vecf{1}_{m_t} = \vecf{0}$
\end{itemize}
where $\matf{A} \succeq \matf{B}$ means $\matf{A} - \matf{B}$ is positive semidefinite.
\end{proposition}
\begin{proof}
See Appendix~\ref{appendix. proof of Q1 = 0}.
\end{proof}

\begin{theorem}
\label{proposition. Q1 = 0}
In problem (\ref{eq: the minimization problem in Map only}),
$\boldsymbol{\mathcal{Q}} \vecf{1} = \vecf{0}$
which means $\vecf{1}$ is an eigenvector of $\boldsymbol{\mathcal{Q}}$ corresponding to eigenvalue $0$.
\end{theorem}
\begin{proof}
This is obvious as
$
\matf{\Gamma}_t^{\trans}
\vecf{1}_{m} = \vecf{1}_{m_t}
$
and
$\matf{Q}_t \vecf{1}_{m_t} = \vecf{0}$.
\end{proof}

\begin{remark}
The expression of $\matf{Q}_t$ in equation (\ref{eq. expression of Qt}) is much more elegant than the one in \citet*{Bai-RSS-22}.
See Appendix~\ref{appendix. Qt expression in comparison to RSS version} for details.
\end{remark}

\subsection{Globally Optimal Estimate of $\matf{M}$}

We recapitulate necessary results to describe the globally optimal solution to problem (\ref{eq: the minimization problem in Map only}).

\begin{definition}[(The $d$ top eigenvectors and the $d$ bottom eigenvectors)]
\label{definition. The top eigenvectors and the bottom eigenvectors}
We consider a symmetric matrix $\matf{\Pi} \in \mathbb{R}^{m \times m}$ and its eigenvalue decomposition:
$$
\matf{\Pi} = \matf{U} \boldsymbol{\Sigma} \matf{U}^{\trans} 
= \sum_{k=1}^{m} \sigma_k \vecf{u}_k \vecf{u}_k^{\trans}
,$$
with
$\matf{U}
=
\left[\vecf{u}_1, \vecf{u}_2, \dots, \vecf{u}_m\right]$ being orthonormal,
and
$\boldsymbol{\Sigma} = \mathrm{diag}\left(\sigma_1, \sigma_2, \dots, \sigma_m\right)$ whose diagonal elements are arranged in the non-ascending order as
$\sigma_1 \ge  \sigma_2 \ge \dots \ge \sigma_m$.
We term:
$$\vecf{u}_1, \vecf{u}_2, \dots, \vecf{u}_d,$$
in sequence the $d$ \textit{top eigenvectors} of $\matf{\Pi}$, and:
$$\vecf{u}_m, \vecf{u}_{m-1}, \dots, \vecf{u}_{m-d+1},$$
in sequence the $d$ \textit{bottom eigenvectors} of $\matf{\Pi}$.
\end{definition}

\begin{lemma}
\label{theorem: general result: PCA with constraint XT u = 0}
We consider a symmetric matrix $\matf{\Pi} \in \mathbb{R}^{m \times m}$, and $\matf{X} \in \mathbb{R}^{m \times d}$.
Let $\boldsymbol{\Lambda} = \mathbf{diag}\left(\lambda_1, \lambda_2, \dots, \lambda_d\right)$ be a diagonal matrix with
$
\lambda_1 \ge  \lambda_2 \ge \dots \ge \lambda_d \ge 0
$.
If $\vecf{u}$ is an eigenvector of the symmetric matrix $\matf{\Pi}$,
then we have:
\begin{enumerate}
\item The globally optimal solution of:
\begin{equation}
\label{eq: general results: argmax, PCA in S, with S1=0}
\begin{aligned}
 \max_{\matf{X}} \quad  
 &
\mathrm{tr}\left( \matf{X}^{\trans} 
\matf{\Pi}
\matf{X} \boldsymbol{\Lambda} \right)
\quad 
\\[0pt] &
\mathrm{s.t.} \ 
\matf{X}^{\trans} \matf{X} = \matf{I},
\ 
\matf{X}^{\trans}  \vecf{u} = \vecf{0}
,
\end{aligned}
\end{equation}
is $\matf{X} = [\vecf{x}_1,\,\vecf{x}_2,\, \dots, \vecf{x}_d]$, where
$\vecf{x}_1,\,\vecf{x}_2,\, \dots, \vecf{x}_d$ are the $d$ top eigenvectors of $\matf{\Pi}$ excluding $\vecf{u}$.
\vspace{8pt}
\item The globally optimal solution of:
\begin{equation}
\label{eq: general results: argmin, PCA in S, with S1=0}
\begin{aligned}
 \min_{\matf{X}} \quad
 &
\mathrm{tr}\left( \matf{X}^{\trans} 
\matf{\Pi}
\matf{X} \boldsymbol{\Lambda} \right)
\quad
\\[0pt] & 
\mathrm{s.t.} \ 
\matf{X}^{\trans} \matf{X} = \matf{I},
\ 
\matf{X}^{\trans}  \vecf{u} = \vecf{0}
,
\end{aligned}
\end{equation}
is $\matf{X} = [\vecf{x}_1,\,\vecf{x}_2,\, \dots, \vecf{x}_d]$, where
$\vecf{x}_1,\,\vecf{x}_2,\, \dots, \vecf{x}_d$ are the $d$ bottom eigenvectors of $\matf{\Pi}$ excluding $\vecf{u}$.
\end{enumerate}

\end{lemma}

\begin{proof}
An initial version of the proof was given in \citet*{bai2022ijcv}.
Here we provide a conciser version without any further assumption on $\matf{\Pi}$,
{\color{black}
see Appendix~\ref{appendix. proof of general result: PCA with constraint XT u = 0}.
Some preliminaries are provided in Appendix~\ref{section. Brockett Cost Function on the Stiefel Manifold}.
}
\end{proof}

\begin{theorem}
\label{theorem: standard result: minimization in S, with S1=0}
The globally optimal solution to problem (\ref{eq: the minimization problem in Map only})
is in closed-form:
$$\matf{M} = \sqrt{\boldsymbol{\Lambda}} \matf{X}^{\trans}
,
\quad
\mathrm{where\ \ }
\matf{X} = [\vecf{x}_1,\,\vecf{x}_2,\, \dots, \vecf{x}_d]  \in \mathbb{R}^{m \times d}
,
$$ where $\vecf{x}_1,\,\vecf{x}_2,\, \dots, \vecf{x}_d$ in sequence are the $d$ bottom eigenvectors of $\boldsymbol{\mathcal{Q}}$ excluding the vector $\boldsymbol{1}$.
\end{theorem}

\begin{proof}
In problem (\ref{eq: the minimization problem in Map only}),
by letting $\matf{M} = \sqrt{\boldsymbol{\Lambda}} \matf{X}^{\trans}$, we have:
\begin{equation}
\begin{aligned}
\min_{\matf{X}}\quad 
& 
\mathrm{tr}\left( \matf{X}^{\trans} 
\boldsymbol{\mathcal{Q}}
\matf{X} \boldsymbol{\Lambda} \right)
\\[0pt] & 
\mathrm{s.t.} \ 
\matf{X}^{\trans} \matf{X} = \matf{I},
\ 
\matf{X}^{\trans}  \vecf{1} = \vecf{0}
.
\end{aligned}
\label{eq. formulation. reduced problem in X}
\end{equation}
From Proposition~\ref{proposition. Q1 = 0},
we see $\vecf{1}$ is an eigenvector of $\boldsymbol{\mathcal{Q}}$ (with eigenvalue $0$).
The result is immediate by applying Lemma~\ref{theorem: general result: PCA with constraint XT u = 0}.
\end{proof}

\begin{remark}[(Shifting eigenvectors)]
\label{remark. shifting eigenvectors of 1}
Since $\boldsymbol{\mathcal{Q}} \vecf{1} = \vecf{0}$, we can shift the eigenvector $\vecf{1}$ of $\boldsymbol{\mathcal{Q}}$ to the top by letting:
$$
\boldsymbol{\mathcal{Q}}' = \boldsymbol{\mathcal{Q}} + n \vecf{1} \vecf{1}^{\trans}  
,
$$
and solve for the $d$ bottom eigenvectors of $\boldsymbol{\mathcal{Q}}'$ to form $\matf{X}$.
\end{remark}

\subsection{Globally Optimal Estimate of the Deformable Transformation}

Upon obtaining the estimate of $\matf{M}$, we can decide the optimal transformation parameters.
From equation (\ref{eq. At, at, Omegat in M}), we set $\matf{F}_t = \matf{O}$, and take the specific solution:
\begin{equation}
\left[
[\matf{A}_t, \, \vecf{a}_t], \, \boldsymbol{\Omega}_t^{\trans} 
\right]
= 
\matf{M} \matf{\Gamma}_t
\begin{bmatrix}
  \matf{\tilde{P}}_t^{\trans} &  \matf{K}_t
\end{bmatrix}
\boldsymbol{\Delta}_t^{\dagger}
.
\end{equation}
We expand $\boldsymbol{\Delta}_t^{\dagger}$ in the term $\begin{bmatrix}
\matf{\tilde{P}}_t^{\trans} & \matf{K}_t
\end{bmatrix}
\boldsymbol{\Delta}_t^{\dagger}$, see equation (\ref{eq. the way to obtain Ft, Gt}) in Appendix~\ref{appendix. derivation of constant Ft, Gt and Qt}, and write the final result as:
\begin{align}
{}
[
\matf{A}_t, \, \vecf{a}_t
]
& 
=
\matf{M} \matf{\Gamma}_t 
\left(
	\matf{I} - \matf{H}_t \matf{K}_t
\right)
\matf{\tilde{P}}_t^{\dagger}
\label{eq. result [A, t] in M}
\\[5pt]
\boldsymbol{\Omega}_t^{\trans} 
&
=
\matf{M} \matf{\Gamma}_t  \matf{H}_t
,
\label{eq. result Omega in M}
\end{align}
with:
$$
\matf{H}_t 
\defeq
\left( \matf{I} - \boldsymbol{\mathcal{P}}_t \right)
\matf{K}_t \matf{S}_t^{-1}
.
$$

\begin{proposition}
\label{proposition. properties of Ht}
If $ \matf{K}_t $ is chosen positive definite and $\mu_t > 0$, then
$\matf{H}_t$ is symmetric positive definite where:
\begin{itemize}
\item $ \matf{H}_t = \matf{H}_t^{\trans}$
\vspace{8pt}
\item $\matf{Q}_t = \mu_t  \matf{H}_t$
\end{itemize}
\end{proposition}
\begin{proof}
See Appendix~\ref{appendix. properties of Ht}.
\end{proof}

The optimal KBT $\vecf{y}_t (\vecf{p})$ in equation (\ref{eq: transformation of a point by affine-kernel}), for an arbitrary query point $\vecf{p}$, can be written as:
\begin{equation}
\vecf{y}_t (\vecf{p})
 =
\matf{M} \matf{\Gamma}_t 
\left(
\left(
	\matf{I} - \matf{H}_t \matf{K}_t
\right)
\matf{\tilde{P}}_t^{\dagger}
\begin{bmatrix}
\vecf{p} \\[5pt] 1
\end{bmatrix}
+ \matf{H}_t \vecf{k}_t(\vecf{p})
\right)
.
\label{eq. optimal estimation of y_t(p)}
\end{equation}

Since $\matf{M} = \sqrt{\boldsymbol{\Lambda}} \matf{X}^{\trans}$, we establish the estimate of both $\matf{M}$ and $\vecf{y}_t (\cdot)$ up to an unknown $\boldsymbol{\Lambda}$.
It should be noted that any $\boldsymbol{\Lambda}$ admits a globally optimal solution to problem (\ref{eq: deformable SLAM formulation using Kernel and dual variable}).
Thus this is what we can maximally achieve by solving problem (\ref{eq: deformable SLAM formulation using Kernel and dual variable}).

\subsection{Coordinate Transformation of Data}
\label{subsection. Coordinate Transformation of Data}

In the data acquisition process, the point-cloud data $\matf{P}_t$ can be expressed in any user defined coordinate frames.
\begin{definition}[(Coordinate transformation)]
We refer to the \textit{coordinate transformation of data $\matf{P}_t $} as
$\mathbf{\breve{P}}_t = \mathbf{\breve{R}}_t \matf{P}_t + \mathbf{\breve{t}}_t  \vecf{1}^{\trans}$, with $(\mathbf{\breve{R}}_t,\, \mathbf{\breve{t}}_t)$ being any arbitrary rigid transformation.
\end{definition}

Ideally, we want the estimate of $\matf{M}$ to be invariant under coordinate transformations of data.
By equation (\ref{eq. expression of Qt}),
$\matf{Q}_t$ can be expressed with $\matf{I} - \boldsymbol{\mathcal{P}}_{t}$ and $\matf{K}_t$.
If both $\boldsymbol{\mathcal{P}}_{t}$ and $\matf{K}_t$ are invariant to the coordinate transformation of $\matf{P}_t $,
then $\matf{Q}_t$ is invariant to the coordinate transformation, thus so will be $\boldsymbol{\mathcal{Q}}$.

\begin{lemma}[(Lemma 5 in \citet*{bai2022ijcv})]
\label{lemma. invariance of orthogonal projection matrix under coordinate transformations}
The orthogonal projection matrix
$
\boldsymbol{\mathcal{P}}_{t}
=
\matf{\tilde{P}}_t^{\trans}
(
\matf{\tilde{P}}_t \matf{\tilde{P}}_t^{\trans}
)^{\dagger}
\matf{\tilde{P}}_t
$
remains unchanged under any coordinate transformation of $\matf{P}_t$.
\end{lemma}
\begin{proof}
See Appendix~\ref{appendix. proof of the invariance of the orthogonal projection matrix under coordinate transformations}.
\end{proof}

\begin{proposition}
\label{proposition. invariance of Kt by using the kernel function as RBFs}
If the kernel function $k(\cdot , \cdot)$ is chosen as the RBFs, \textit{i.e.,}~$k (\vecf{x}_i, \vecf{x}_j) = \phi (  \left\Vert \vecf{x}_i - \vecf{x}_j \right\Vert  )$, where $k (\vecf{x}_i, \vecf{x}_j)$ is only related to the Euclidean distance of $\vecf{x}_i$ and $\vecf{x}_j$, then $\matf{K}_t$ is invariant to the coordinate transformation.
\end{proposition}

\begin{proposition}
\label{proposition. invariance of M by using RBFs}
If the kernel function $k(\cdot , \cdot)$ is chosen as the RBFs,
then matrix $\boldsymbol{\mathcal{Q}}$ in problem (\ref{eq: the minimization problem in Map only}) remains unchanged. In this case, the optimal estimate of $\matf{M}$ remains unchanged under the coordinate transformation.
\end{proposition}

\section{Global Scale Ambiguity $\sqrt{\boldsymbol{\Lambda}}$}
\label{section: solve rigid transformation R, t and prior Lambda together from given map principal axes X}

\begin{figure*}[t]
	\centering		
	\includegraphics[width=0.85\textwidth]{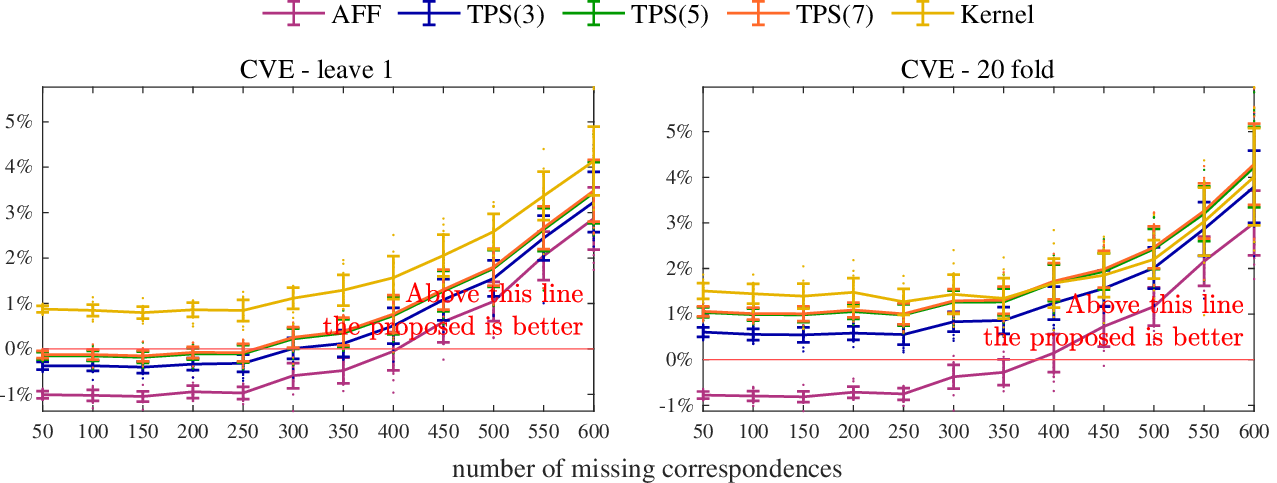}		
	\caption{Performance improvements of the proposed $\sqrt{\boldsymbol{\Lambda}}$ estimation method over the one in~\citet*{bai2022ijcv}.
	We gradually remove the correspondences in the HandBag dataset~\citet*{gallardo2017dense} (with $155\times8$ correspondences in total), and report statistics of $20$-trial Monte-Carlo runs based on leave-$1$ and $20$-fold cross-validations~\citet*{Bai-RSS-22}, for five transformation models: affine, TPS with $3 \times 3$, $5 \times 5$ and $7 \times 7$  control points, and the proposed KBT.}
\label{fig. performance improvements of the proposed Lambda estimation method over the one in IJCV}
\end{figure*}

In \citet*{bai2022ijcv}, the authors proposed a method to estimate the global scale ambiguities (\textit{i.e.,}~the diagonal elements of $\sqrt{\boldsymbol{\Lambda}}$) using pairwise rigid Procrustes analysis.
The method in \citet*{bai2022ijcv} requires the existence of some points to be globally visible across all point-clouds.
In this section, we propose a novel method to estimate $\boldsymbol{\Lambda}$ which does not require such visibility assumptions.

In addition, the method in this work solves $\sqrt{\boldsymbol{\Lambda}}$ by a global optimization formulation, whereas the method in \citet*{bai2022ijcv} relies on local pairwise registrations.
Thus the $\sqrt{\boldsymbol{\Lambda}}$ estimation method proposed in this work can be superior even if the globally visible correspondences are available.
We provide a justification to this claim in Figure~\ref{fig. performance improvements of the proposed Lambda estimation method over the one in IJCV}.

\subsection{As Rigid as Possible}

We want the deformable transformation to be as-rigid-as-possible, that means we want $ \matf{\Phi}_t (\cdot)$ in equation (\ref{eq. deformation transformation y_t}) to be close to an identity mapping.
In particular, without the deformation $ \matf{\Phi}_t (\cdot)$, we have:
$$
\vecf{y}_t (\matf{P}_t)
=
\matf{R}_t \matf{P}_t + \vecf{t}_t \vecf{1}^{\trans}
=
\matf{M} \matf{\Gamma}_t
=
\sqrt{\boldsymbol{\Lambda}} \matf{X}^{\trans} \matf{\Gamma}_t 
.
$$
This motivates us to characterize $\sqrt{\boldsymbol{\Lambda}}$ by an optimization formulation as follows:
\begin{equation}
\label{eq: formulation for rigid transformation, as-rigid-as-possible, general form}
\min_{ \{ \matf{R}_t,\, \vecf{t}_t \} ,\, \boldsymbol{\Lambda},\, \matf{R}_g}
\quad 
\sum_{t=1}^{n}
\, 
\left\Vert
\matf{R}_t \matf{P}_t + \vecf{t}_t \vecf{1}^{\trans}
- 
\sqrt{\boldsymbol{\Lambda}} \matf{R}_g
\matf{G}_t
\right\Vert_{\mathcal{F}}^2
,
\end{equation}
with:
$$
\matf{G}_t = \matf{X}^{\trans} \matf{\Gamma}_t 
,
$$
and $(\matf{R}_t, \, \vecf{t}_t)$ denoting the rigid transformation.
Here we have introduced an orthonormal matrix $\matf{R}_g  \in \mathrm{O}(d)$ for a reason we will explain later in Section~\ref{subsection. Zero-deformation and noise-free}.
At the moment, it suffices to think of $\matf{R}_g$ as an identity matrix.

\subsection{Reduced Formulation}

In formulation (\ref{eq: formulation for rigid transformation, as-rigid-as-possible, general form}), we notice $\vecf{t}_t$ is linearly dependent on the other parameters $\matf{R}_t,\, \boldsymbol{\Lambda},\, \matf{R}_g$.
Thus formulation (\ref{eq: formulation for rigid transformation, as-rigid-as-possible, general form}) admits a separable structure which allows us to eliminate $\vecf{t}_t$ from the formulation~\citet*{golub2003separable}.
In specific, given $\matf{R}_t$, $\boldsymbol{\Lambda}$ and $\matf{R}_g$, the estimates of $\vecf{t}_t$ are expressed as:
\begin{equation}
\label{eq. linear dependence of t on Rt, Rg, Lambda}
\vecf{t}_t = 
- \frac{1}{m_t}\left(
\matf{R}_t \matf{P}_t 
-
\sqrt{\boldsymbol{\Lambda}} \matf{R}_g
\matf{G}_t
\right)
\vecf{1}
,
\quad
\left( t \in [1 : n] \right)
.
\end{equation}
After substituting equation (\ref{eq. linear dependence of t on Rt, Rg, Lambda}) into formulation (\ref{eq: formulation for rigid transformation, as-rigid-as-possible, general form}), we obtain a reduced problem:
\begin{align}
\min_{ \{ \matf{R}_t \},\, \boldsymbol{\Lambda},\, \matf{R}_g}
\quad 
\sum_{t=1}^{n}
\, &
\left\Vert
\matf{R}_t \matf{\bar{P}}_t
- 
\sqrt{\boldsymbol{\Lambda}}  \matf{R}_g  \matf{\bar{G}}_t
\right\Vert_{\mathcal{F}}^2
,
\label{eq: formulation for rigid transformation, as-rigid-as-possible, in R only}
\end{align}
with $\matf{\bar{P}}_t
=
\matf{P}_t - \frac{1}{m_t} \matf{P}_t \vecf{1} \vecf{1}^{\trans}$
and
$
\matf{\bar{G}}_t
=
\matf{G}_t - \frac{1}{m_t} \matf{G}_t \vecf{1} \vecf{1}^{\trans}
$.

\subsection{Closed-form Evaluation of $\sqrt{\boldsymbol{\Lambda}}$ and $\matf{R}_g$}
\label{subsection. closed-form Lambda and Rg}

From formulation (\ref{eq: formulation for rigid transformation, as-rigid-as-possible, in R only}), we consider an affine relaxation of $\sqrt{\boldsymbol{\Lambda}}  \matf{R}_g$, and establish its linear dependence on $\matf{R}_t$ as:
\begin{align}
\sqrt{\boldsymbol{\Lambda}}  \matf{R}_g
\gets
\matf{R}_t
\matf{\bar{P}}_t \matf{\bar{G}}_t^{\dagger}
,
\quad
\left( t \in [1 : n] \right)
.
\label{eq. affine relaxation of Rt}
\end{align}
From relaxation (\ref{eq. affine relaxation of Rt}), we then compute
$(\sqrt{\boldsymbol{\Lambda}}  \matf{R}_g)^{\trans} \sqrt{\boldsymbol{\Lambda}}  \matf{R}_g$
and apply the orthonormal constraint $\matf{R}_t^{\trans} \matf{R}_t = \matf{I}$, as:
\begin{align*}
\matf{R}_g^{\trans}  \boldsymbol{\Lambda} \matf{R}_g 
=
\left( \matf{R}_t \matf{\bar{P}}_t \matf{\bar{G}}_t^{\dagger} \right)^{\trans}
\matf{R}_t \matf{\bar{P}}_t \matf{\bar{G}}_t^{\dagger}
=
(
\matf{\bar{P}}_t \matf{\bar{G}}_t^{\dagger}
)^{\trans} 
\,
\matf{\bar{P}}_t \matf{\bar{G}}_t^{\dagger}
.
\end{align*}
Given $n$ point-clouds, we take the average with respect to $t$ which corresponds to the maximum likelihood estimate:
\begin{equation}
\label{eq. eigenvalue decomposition of Rg Lamda Rg = L}
\matf{R}_g^{\trans}
\boldsymbol{\Lambda}
\matf{R}_g
=
\frac{1}{n} \sum_{t=1}^{n} 
(
\matf{\bar{P}}_t \matf{\bar{G}}_t^{\dagger}
)^{\trans} 
\matf{\bar{P}}_t \matf{\bar{G}}_t^{\dagger}
\defeq
\boldsymbol{\mathcal{L}}
.
\end{equation}
We see the lefthand of equation (\ref{eq. eigenvalue decomposition of Rg Lamda Rg = L}) forms the eigenvalue decomposition of $\boldsymbol{\mathcal{L}}$.
We thus compute the diagonals of $\boldsymbol{\Lambda}$ as the eigenvalues of $\boldsymbol{\mathcal{L}}$, and the rows of $\matf{R}_g$ as the corresponding eigenvectors.
We arrange the eigenvalues of $\boldsymbol{\mathcal{L}}$ in the non-descending order.
We notice $\boldsymbol{\mathcal{L}}$ is positive definite (or semi-definite), thus the eigenvalues of $\boldsymbol{\mathcal{L}}$ are non-negative.
Therefore $\sqrt{\boldsymbol{\Lambda}}$ is well-defined in the real domain.

\begin{remark}
The idea to factorize $\sqrt{\boldsymbol{\Lambda}}$ from equation (\ref{eq. eigenvalue decomposition of Rg Lamda Rg = L}) is maturer than the initial version in \citet*{Bai-RSS-22}.
In particular, the eigenvalue decomposition in equation (\ref{eq. eigenvalue decomposition of Rg Lamda Rg = L}) was not realized in \citet*{Bai-RSS-22}.
Critically, the method in \citet*{Bai-RSS-22} may lead to negative diagonals in $\boldsymbol{\Lambda}$, causing undefined $\sqrt{\boldsymbol{\Lambda}}$.
\end{remark}

\vspace{5pt}
\noindent
\textbf{Initialization of $\matf{R}_t$ and $\vecf{t}_t$.}
Given $\sqrt{\boldsymbol{\Lambda}}$ and $\matf{R}_g$, the rotation $\matf{R}_t$ can be solved from formulation (\ref{eq: formulation for rigid transformation, as-rigid-as-possible, in R only}) in closed-form by the special orthogonal Procrustes analysis \citet*{arun1987least, horn1988closed} between $\matf{\bar{P}}_t$ and $\sqrt{\boldsymbol{\Lambda}} \matf{R}_g \matf{\bar{G}}_t$.
Afterwards, we compute $\vecf{t}_t$ from equation (\ref{eq. linear dependence of t on Rt, Rg, Lambda}).

\subsection{Iterative Refinement}

We can solve formulation (\ref{eq: formulation for rigid transformation, as-rigid-as-possible, in R only}) exactly using iterative NLS optimization techniques, \textit{e.g.,}~by Gauss-Newton or Levenberg-Marquardt.
The rotation $\matf{R}_t$ can be readily handled with Lie group techniques.
The diagonal elements of $\sqrt{\boldsymbol{\Lambda}}$ are constrained to be non-negative, thus requiring special consideration.

\vspace{5pt}
\noindent
\textbf{Reflection.}
We notice that the columns of $\matf{X}$ (as the eigenvectors of $\boldsymbol{\mathcal{Q}}$), and thus the rows of $\matf{G}_t$ (and $\matf{\bar{G}}_t$), are defined up to signs.
This means that if we flip the sign of one column in $\matf{X}$, the solution is still optimal.
Using a specific $\matf{X}$, the optimal $\matf{R}_t$ in formulation (\ref{eq: formulation for rigid transformation, as-rigid-as-possible, in R only}) may have negative determinants $\mathrm{det} (\matf{R}_t) = -1$, which is called a \textit{reflection}.

\vspace{5pt}
We thus extend formulation (\ref{eq: formulation for rigid transformation, as-rigid-as-possible, in R only}) using $ \boldsymbol{\eta} \in \mathbb{R}^{d}$ to handle the possible reflections caused by the specification of $\matf{X}$:
\begin{align}
\min_{ \{ \matf{R}_t \in \mathrm{SO}(d) \},  \,  \boldsymbol{\eta} \in \mathbb{R}^{d} }
\quad 
\sum_{t=1}^{n}
\, &
\left\Vert
\matf{R}_t \matf{\bar{P}}_t
- 
\mathbf{diag} ( \boldsymbol{\eta} )
\matf{\bar{G}}_t
\right\Vert_{\mathcal{F}}^2
,
\label{eq. formulation of Rt in SO3 and Eta in R3}
\end{align}
where $\mathbf{diag} ( \boldsymbol{\eta} )$ is a diagonal matrix taking the components in $\boldsymbol{\eta}$.
We further denote $ \mathbf{sign}(\boldsymbol{\eta}) $ a vector containing the signs of the components in $\boldsymbol{\eta}$.

If there exist reflections, the optimal $\boldsymbol{\eta}$ in formulation (\ref{eq. formulation of Rt in SO3 and Eta in R3}) can have negative components.
In this case, we flip the sign of the columns of $\matf{X}$ accordingly.
In general, we set:
\begin{align*}
\matf{X}  & \gets  \matf{X} \,  \mathbf{diag} ( \mathbf{sign}(\boldsymbol{\eta}) )
\\[5pt]
\sqrt{\boldsymbol{\Lambda}} & \gets \mathbf{diag} ( \boldsymbol{\eta} ) \, \mathbf{diag} ( \mathbf{sign}(\boldsymbol{\eta}) )
.
\end{align*}
These operations preserve the optimality of both $\matf{X}$ and $\sqrt{\boldsymbol{\Lambda}}$.

\begin{remark}
It should be noted that the rigid transformations $(\matf{R}_t, \, \vecf{t}_t)$ solved from formulation (\ref{eq: formulation for rigid transformation, as-rigid-as-possible, general form}) are different from the ambiguous poses defined in equation (\ref{eq. entangled (R, t) and Phi}). In essence, formulation (\ref{eq: formulation for rigid transformation, as-rigid-as-possible, general form}) approximately solves Rigid-GPA, by constraining $\matf{M}$ as
$\matf{M} = \sqrt{\boldsymbol{\Lambda}} \matf{X}^{\trans} $.
Thus the optimal $(\matf{R}_t, \, \vecf{t}_t)$ obtained from formulation (\ref{eq: formulation for rigid transformation, as-rigid-as-possible, general form}) are similar to the poses obtained from Rigid-GPA, as shown in Figure~\ref{fig:trajectory_visualization}.
\end{remark}

\begin{figure}[h]
\includegraphics[width=0.48\textwidth]{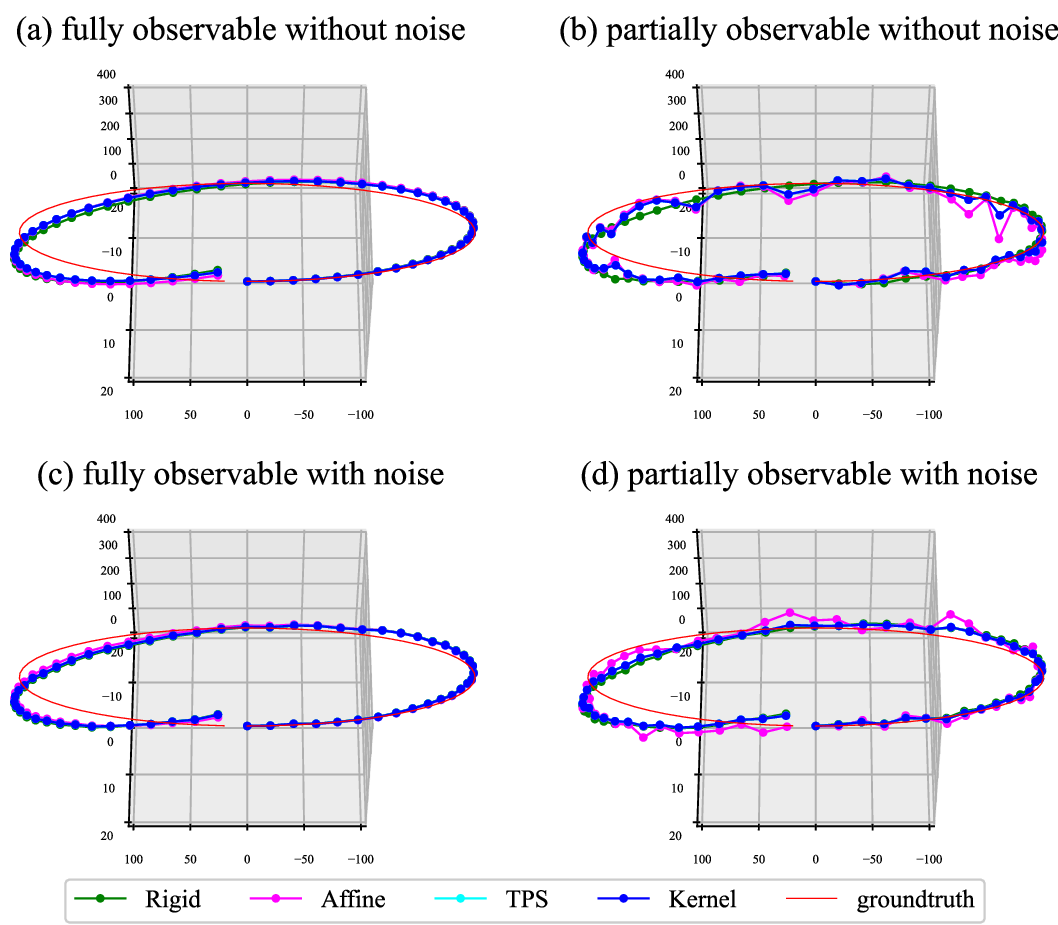}
\caption{Estimated trajectories of Rigid-GPA, Affine-GPA, TPS-GPA, and Kernel-GPA on the liver dataset.}
\label{fig:trajectory_visualization}
\end{figure}

\section{Degeneracies}
\label{section. Degeneracy}

\subsection{Zero-deformation and Noise-free}
\label{subsection. Zero-deformation and noise-free}

We consider the case where matrix $\boldsymbol{\mathcal{Q}}$ in problem (\ref{eq: the minimization problem in Map only}) has $d+1$ zero eigenvalues, where one of them corresponds to the eigenvector $\vecf{1}$ (Theorem~\ref{proposition. Q1 = 0}).
Following Remark~\ref{remark. shifting eigenvectors of 1}, we can drop the eigenvector $\vecf{1}$ easily by solving the $d$ bottom eigenvectors of
$\boldsymbol{\mathcal{Q}}' = \boldsymbol{\mathcal{Q}} + n \vecf{1} \vecf{1}^{\trans} $ instead to form the columns of $\vecf{X}$.
Note that in this case:
\begin{equation*}
 \boldsymbol{\mathcal{Q}} \vecf{X} = \matf{O} = \matf{X} \matf{U}_g \, \mathbf{diag} (\vecf{0}) 
,
\end{equation*}
where $\matf{U}_g$ is an arbitrary orthonormal matrix,
\textit{i.e.,}~$\matf{U}_g \matf{U}_g^{\trans} = \matf{U}_g^{\trans} \matf{U}_g = \matf{I}$.
We see that any $\matf{X}_g \defeq \matf{X} \matf{U}_g$ is a valid solution as $\matf{X}_g \matf{X}_g^{\trans} = \matf{I}$.
Thus, the optimal solution of problem (\ref{eq: the minimization problem in Map only}) will be defined up to an arbitrary $\matf{U}_g$ as:
$$
\matf{M} = \sqrt{\boldsymbol{\Lambda}} \matf{U}_g^{\trans} \matf{X}^{\trans}
.
$$
In this case, $\matf{U}_g$ is fundamentally ambiguous, which means there is no way to decide $\matf{U}_g$ from problem (\ref{eq: the minimization problem in Map only}) directly.

This is the reason why we introduce $\matf{R}_g$ in formulation (\ref{eq: formulation for rigid transformation, as-rigid-as-possible, general form}, \ref{eq: formulation for rigid transformation, as-rigid-as-possible, in R only}), where we essentially denote $\matf{R}_g = \matf{U}_g^{\trans}$.
If such a degeneracy occurs,
we can factorize $\matf{R}_g$ (and thus $\matf{U}_g$) from equation (\ref{eq. eigenvalue decomposition of Rg Lamda Rg = L}) by the eigenvalue decomposition.

\subsection{Flat Point-cloud in 3D Space}
\label{sec. degeneracy. flat point-cloud}

\vspace{5pt}
\noindent
\textbf{The cost function.}
We consider the case of $d=3$ and denote $\matf{X} = [\vecf{x}_1,\, \vecf{x}_2,\, \vecf{x}_3]$.
With some matrix manipulations, see Appendix~\ref{appendix. expansion of tr(x Q x Lambda)}, we show that the cost of problem (\ref{eq. formulation. reduced problem in X}) can be rewritten as:
\begin{align}
\mathrm{tr}\left( \matf{X}^{\trans} 
\boldsymbol{\mathcal{Q}}
\matf{X} \boldsymbol{\Lambda} \right)
& =
\notag
\\[0pt]
&
\sum_{k = 1}^{3}
\lambda_k
\left\Vert
\begin{bmatrix}
\sqrt{ \matf{Q}_1 } \,
( \matf{I} - \boldsymbol{\mathcal{P}}_1 )
\matf{\Gamma}_1^{\trans}
\\[5pt]
\vdots
\\[5pt]
\sqrt{ \matf{Q}_n } \,
( \matf{I} - \boldsymbol{\mathcal{P}}_n )
\matf{\Gamma}_n^{\trans}
\end{bmatrix}
\vecf{x}_k
\right\Vert_{\mathcal{F}}^2
.
\label{eq. the cost expansion of tr(X Q X Lambda)}
\end{align}
Matrix $ \boldsymbol{\mathcal{P}}_t $ is the orthogonal projector to the range space of $\matf{\tilde{P}}_t^{\trans}$~\citet*{meyer2000matrix}.
In particular, if a vector $\vecf{y}$ lies in the range space of $\matf{\tilde{P}}_t^{\trans}$, then
$ \left( \matf{I} - \boldsymbol{\mathcal{P}}_t \right) \vecf{y} = \vecf{0} $.
Hence, this cost is zero (and thus minimized) if each of $\matf{\Gamma}_t^{\trans} \vecf{x}_k$ can be chosen from the respective range space of $\matf{\tilde{P}}_t^{\trans}$,
which is usually impossible due to the existence of noise and deformations.

\vspace{5pt}
\noindent
\textbf{The canonical planar point-cloud.}
If the point-cloud $\matf{P}_t$ is flat, then $\matf{P}_t$ can be rigidly transformed to the $xy-$plane.
In addition, from Theorem~\ref{proposition. any point-cloud can satisfy the transformation constraints}, we conclude that for a flat $\matf{P}_t$, there exists a rigid transformation $(\matf{R}_c, \vecf{t}_c)$ and a canonical 2D point-cloud $\mathbf{P}_{txy}$ in the $xy-$plane such that:
\begin{equation*}
\matf{\tilde{P}}_t
=
\begin{bmatrix}
\matf{P}_t \\[5pt]
\vecf{1}^{\trans}
\end{bmatrix}
=
\begin{bmatrix}
\matf{R}_c & \vecf{t}_c \\[5pt]
\vecf{0}^{\trans} & 1
\end{bmatrix}
\begin{bmatrix}
\mathbf{P}_{txy} \\[5pt]
\vecf{0}^{\trans} \\[5pt]
\vecf{1}^{\trans}
\end{bmatrix}
,
\quad
\mathrm{with}\ 
\mathbf{P}_{txy}
=
\begin{bmatrix}
\vecf{u}_{tx}^{\trans} \\[5pt]
\vecf{u}_{ty}^{\trans} 
\end{bmatrix}
,
\label{eq. the basis of a rank deficient tildePt}
\end{equation*}
where $\vecf{u}_{tx}^{\trans} \vecf{u}_{ty} = 0$, $\vecf{u}_{tx}^{\trans} \vecf{1} = 0$, $\vecf{u}_{ty}^{\trans} \vecf{1} = 0$.
Note that $\matf{u}_{tx}$, $\matf{u}_{ty}$, and $\vecf{1}$ form an orthogonal basis of the range space of $\matf{\tilde{P}}_t^{\trans}$.

\vspace{5pt}
\noindent
\textbf{The solution of $\matf{X}$.}
Vector $\vecf{1}$ lies in the range space of each $\matf{\tilde{P}}_t^{\trans}$. However, due to the constraint $\matf{X}^{\trans} \vecf{1} = \vecf{0}$, we require the columns of $\matf{X}$ to be orthogonal to $\vecf{1}$. As a result, $\vecf{1}$ must be excluded from $\matf{X}$.
Hence, the columns of $\matf{X}$ are essentially constructed based on the ``closeness'' to the range space of each $\matf{P}_t^{\trans}$, or equivalently to the range space of each $\mathbf{P}_{txy}^{\trans}$,
by evaluating the cost (\ref{eq. the cost expansion of tr(X Q X Lambda)}).
Note that since each $\mathbf{P}_{txy}^{\trans}$ has a two dimensional range space, the last column of $\matf{X}$, \textit{i.e.,}~$\vecf{x}_3$ will be pushed toward the null space of $\mathbf{P}_{txy}$ by the orthogonality constraint $\vecf{x}_1^{\trans} \vecf{x}_3 = 0$ and $\vecf{x}_2^{\trans} \vecf{x}_3 = 0$.

\vspace{5pt}
\noindent
\textbf{The solution of $\sqrt{\boldsymbol{\Lambda}}$.}
After solving $\matf{X}$, we leverage formulation (\ref{eq. formulation of Rt in SO3 and Eta in R3}) to estimate $\sqrt{\boldsymbol{\Lambda}}$.
In particular,
we consider the following problem by using the canonical 2D point-clouds $\mathbf{P}_{txy}$ in the $xy-$plane, as:
\begin{equation}
\min_{ \{ \matf{R}_t \in \mathrm{SO}(d) \},\,  \boldsymbol{\eta} \in \mathbb{R}^{d} } \ \ 
\sum_{t=1}^{n}
\left\Vert
\matf{R}_{t}
\begin{bmatrix}
\mathbf{P}_{txy} \\[5pt]
\vecf{0}^{\trans}
\end{bmatrix}
 - 
\mathbf{diag} ( \boldsymbol{\eta} )
\matf{\bar{G}}_t
\right\Vert_{\mathcal{F}}^2
.
\label{eq. Rt and Lambda from canonical point clouds}
\end{equation}
If the optimal $\matf{R}_{t}$ of problem (\ref{eq. Rt and Lambda from canonical point clouds}) implements a rotation in the $xy$-plane,
then the last component in $\sqrt{\boldsymbol{\Lambda}}$ is zero, \textit{i.e.,}~$\lambda_3 = 0$, see Appendix~\ref{appendix. planar case. Rt in xy plane} for more details.
In this case, the optimal $\matf{M}$ will be flat, residing in the $xy-$plane.
This happens if GPA solved from the canonical 2D point-clouds $\mathbf{P}_{txy}$ with $t \in [1:n]$ is optimal in the embedded 3D space.

\begin{remark}
In general, if the 2D data are generated by flattening 3D observations to 2D, \textit{e.g.,}~a) by a projective function or b) by simply ignoring the $z-$coordinates, the optimal reconstruction in the embedded 3D space is usually not flat!
Such an example is the SfM problem.
\end{remark}

\begin{remark}
Similar discussions hold for the case of $d=2$, if the point-clouds degenerate to lines in the plane.
\end{remark}

\section{Implementation}
\label{section. implementation details}

\subsection{Regularization Strength $\mu_t$}
\label{subsection. regularization strength}

We rewrite $\matf{K}_t \matf{S}_t^{-1} $ as:
\begin{equation*}
\matf{K}_t \matf{S}_t^{-1} = 
\frac{1}{\mu_t}
\left(
 \frac{1}{\mu_t} \matf{K}_t  \left( \matf{I} - \boldsymbol{\mathcal{P}}_t \right) +   \matf{I}
\right)^{-1}
.
\end{equation*}
If $\mu_t \rightarrow  + \infty $, then $\frac{1}{\mu_t} \matf{K}_t \rightarrow \matf{O}$.
As a result, $\matf{K}_t \matf{S}_t^{-1} \rightarrow \matf{O}$ and thus $\matf{H}_t  = \left( \matf{I} - \boldsymbol{\mathcal{P}}_t \right) \matf{K}_t \matf{S}_t^{-1} \rightarrow \matf{O} $.
From equations (\ref{eq. expression of Qt}, \ref{eq. result [A, t] in M}, \ref{eq. result Omega in M}), we conclude when $\mu_t \rightarrow  + \infty$, KernelGPA becomes the Affine-GPA:
\begin{equation*}
{ \mathrm{Affine\ GPA} }
\begin{cases}
\matf{Q}_t =  \matf{I} - \boldsymbol{\mathcal{P}}_t
\\[5pt]
[ \matf{A}_t, \, \vecf{a}_t ] = \matf{M} \matf{\Gamma}_t  \matf{\tilde{P}}_t^{\dagger} 
\\[5pt]
\boldsymbol{\Omega}_t^{\trans} = \matf{O} .
\end{cases}
\end{equation*}

For general cases, from equations (\ref{eq. expression of Qt}, \ref{eq. result [A, t] in M}), we notice that both $\matf{Q}_t$ and $[ \matf{A}_t, \, \vecf{a}_t ]$ make use of the kernel matrix $\matf{K}_t$ in the form of $\matf{K}_t \matf{S}_t^{-1} \matf{K}_t$: 
\begin{equation*}
\matf{K}_t \matf{S}_t^{-1} \matf{K}_t = 
\left(
\left( \matf{I} - \boldsymbol{\mathcal{P}}_t \right) +  \left(\frac{1}{\mu_t} \matf{K}_t \right)^{-1}
\right)^{-1}
,
\end{equation*}
where $\mu_t$ controls the influence of $\matf{K}_t $ as $\frac{1}{\mu_t} \matf{K}_t $, and thus the allowed deformation.
The larger $\mu_t$, the smaller the influence of $\matf{K}_t$, and thus the lower the allowed deformation.

We use the same regularization strength for all point-clouds, by setting $\mu_t = \mu$ for $t \in [1 : n]$.

\subsection{Gaussian Kernel}

The proposed KernelGPA can be implemented with a range of kernel functions, up to the choice of the user.
Following Proposition~\ref{proposition. invariance of Kt by using the kernel function as RBFs} and Proposition~\ref{proposition. invariance of M by using RBFs}, we suggest designing the kernel function $k(\cdot, \cdot)$ as the RBFs.
Other than that, we do not pose any extra constraint on the possibilities of $k(\cdot, \cdot)$.

We specifically implement $k(\cdot, \cdot)$ using the Gaussian kernel, which is an RBF taking the form:
\begin{equation}
k (\vecf{x}_i, \vecf{x}_j) = \exp  \left( - \frac{\left\Vert \vecf{x}_i - \vecf{x}_j \right\Vert^2}{2 \sigma^2}  \right)
.
\end{equation}
We decide the \textit{kernel bandwidth} $\sigma$ as $\sigma =  p \bar{d}$, where $\bar{d}$ denotes the mean of the pairwise Euclidean distances between all the discrete training points:
\begin{equation}
\bar{d} = \operatorname{mean} \left( \left\Vert \vecf{x}_i - \vecf{x}_j \right\Vert \right)
, \quad \mathrm{for\ all\ }(i\neq j)
,
\end{equation}
and $p > 0$ is a tunable scale factor.

In our case, for each point-cloud $\matf{P}_t$ and thus each $\matf{K}_t$, we implement a Gaussian kernel with kernel bandwidth $\sigma_t$. We set $\sigma_t =  p \bar{d}_t$ where $\bar{d}_t$ denotes the mean pairwise Euclidean distances between all the corresponding points in $\matf{P}_t$.


\section{Experimental Results}
\label{section. experimental results}

We evaluate the performance of different GPA methods using three datasets:
a) the semi-synthetic liver dataset for smooth organ deformations,
b) the facial expression dataset for structural deformations,
and c) the TOPACS point-clouds extracted from computerized tomography (CT) scans for real medical scenarios.

\subsection{Preliminary}

\vspace{5pt}
\noindent
\textbf{Correspondences.}
The proposed GPA registration is based on correspondences, which can be extracted from RGB-D cameras, segmented meshes or raw point-clouds.
The computational complexity is determined by the dimension of the $\boldsymbol{\mathcal{Q}}$ matrix, and is thus decided by the number of used correspondences.
Since we have assumed low-dimensional deformations, the GPA registration does not require a large number of correspondences.
In contrast, in most cases, the redundancy of correspondences does not improve much the accuracy of the GPA registration, but cause strains on the computation.
Hence, we always suggest using a reasonable amount of correspondences, as long as they are sufficient to capture the underlying motions and deformations.

\vspace{5pt}
\noindent
\textbf{Test points.}
After solving GPA, we obtain an estimate of the deformable transformations $\vecf{y}_t (\cdot)$ and a reference map of used correspondences.
While formulated in the cost function, it is not a good idea to evaluate the residual $\vecf{y}_t (\matf{P}_{t}) - \matf{M} \matf{\Gamma}_t $, because $\vecf{y}_t (\cdot)$ may overfit the correspondences.
Therefore, we use correspondences to solve GPA, and afterwards benchmark the performance of GPA registration using the idea of \textit{test points}.
Importantly, the test points are never used to solve GPA (as the correspondences of these points are typically not available), but usually serve as a dense representation of the geometry of the scene.

\begin{table*}[t]
	\centering
	\caption{The statistics of different GPA methods on the liver dataset.}
	\renewcommand{\arraystretch}{1.0}
	\begin{tabular}{p{2.3cm} p{2cm}  p{2cm} | p{2cm} p{2cm} p{2cm} p{2cm}}
			 & &   & Rigid-GPA & Affine-GPA & TPS-GPA & Kernel-GPA \\
		\hline
		\multirow{3}{8em}{full visibility} & \multirow{3}{8em}{no noise} & min (mm) & 0.389 & 0.140 & 0.022 & \textbf{0.006}\\
		& & max (mm) & 5.208 & 3.008 & 1.567 & \textbf{1.310} \\
		& & mean (mm) & 2.470 & 1.423 & 0.459 & \textbf{0.174} \\
		\hline
		\multirow{3}{8em}{partial visibility} & \multirow{3}{8em}{no noise} & min (mm) & 0.384 & 0.238 & 0.062 & \textbf{0.042} \\
   		& & max (mm) & 5.207 & 3.031 & \textbf{1.518} & 1.868 \\
		& & mean (mm) & 2.480 & 1.435 & 0.503 & \textbf{0.453} \\
		\hline
		\multirow{3}{8em}{full visibility} & \multirow{3}{8em}{with noise} & min (mm) & 1.517 & 1.501 & 1.215 & \textbf{0.557}\\
		& & max (mm) & 5.691 & 3.506 & \textbf{2.404} & 2.448 \\
		& & mean (mm) & 3.062 & 2.229 & \textbf{1.713} & 1.749 \\
		\hline
		\multirow{3}{8em}{partial visibility} & \multirow{3}{8em}{with noise} & min (mm) & 1.531 & 1.499 & 1.328 & \textbf{1.127} \\
		& & max (mm) & 5.717 & 3.524 & \textbf{2.423} & 2.542 \\
		& & mean (mm) & 3.070 & 2.251 & \textbf{1.790} & 1.928 \\
	\end{tabular}
 \label{tab:liver_deformation_error}
\end{table*}

\vspace{5pt}
\noindent
\textbf{Consistency by extrapolation.}
We denote the test points as $\matf{\check{P}}_t$ $(t \in [1 : n])$. After solving deformable transformations $\vecf{y}_t (\cdot)$, we evaluate the coherence of the transformed points $\vecf{y}_t (\matf{\check{P}}_t)$ for all $t \in [1 : n]$.
To benchmark the closeness of these transformed points, we need to define a distance metric, based on \textit{e.g.,}~surface-to-surface or nearest neighboring point distances etc.
To simplify the evaluation, we assume the correspondence information for the test points are also known.
We use $\matf{\check{\Gamma}}_t$ to denote the corresponding visibility information of $\matf{\check{P}}_t$.
Such assumption allows us to evaluate the deviation of the transformed points $\vecf{y}_t (\matf{\check{P}}_t)$ directly.

%
%
%

\vspace{5pt}
\noindent
\textbf{Evaluation metrics.}
We define the \textit{mean map} of the test points using the mean of $\vecf{y}_t (\matf{\check{P}}_t)$, as:
\begin{equation}
\matf{\check{M}} \defeq
\left(  \sum_{t=1}^n  \vecf{y}_t (\matf{\check{P}}_t) \matf{\check{\Gamma}}_t^{\trans}  \right)
\left( \sum_{t=1}^n  \matf{\check{\Gamma}}_t  \matf{\check{\Gamma}}_t^{\trans} \right)^{\dagger}
,
\end{equation}
where $\sum_{t=1}^n  \matf{\check{\Gamma}}_t  \matf{\check{\Gamma}}_t^{\trans}$ count the total visibilities of each correspondence.
We shall use the mean map $\matf{\check{M}}$ as the reconstruction of the test points.
Then we benchmark the accuracy of the mean map $\matf{\check{M}}$ using the consistencies of the transformed test points.
In specific, for each point in $\matf{\check{M}} $, we define the \textit{point-wise consistencies} of the test points as:
\begin{equation*}
\boldsymbol{\check{\delta}} = 
\sqrt{
\vecf{1}^{\trans}
\left(
\sum_{t=1}^n  \left( \boldsymbol{\Sigma}_t  *  \boldsymbol{\Sigma}_t \right)   \matf{\check{\Gamma}}_t^{\trans}
\right)
\left( \sum_{t=1}^n  \matf{\check{\Gamma}}_t  \matf{\check{\Gamma}}_t^{\trans} \right)^{\dagger}
}
,
\end{equation*}
where
$
\boldsymbol{\Sigma}_t = \vecf{y}_t (\matf{\check{P}}_t) - \matf{\check{M}} \matf{\check{\Gamma}}_t 
,
\quad
\left( t \in [1 : n] \right)	
.
$
Here $\boldsymbol{\Sigma}_t  *  \boldsymbol{\Sigma}_t$ denotes element-wise matrix multiplication, and the outermost square-root is also computed element-wise.

\vspace{5pt}
\noindent
\textbf{Benchmark methods.}
We term GPA with the TPS warp as TPS-GPA, and GPA with the KBT as Kernel-GPA.
We compare Kernel-GPA with the Rigid-GPA, Affine-GPA and TPS-GPA methods. We use in total $125 = 5\times5\times5$ control points for the TPS warp, which are evenly distributed along the principle axes of the point-cloud. The regularization strength of the TPS warp is set to $0.01$ as suggested in~\citet*{bai2022ijcv} for 3D data.

\subsection{Liver}
\label{exp. sec. liver data}

\begin{figure}[t]
	\centering
	\includegraphics[width=0.42\textwidth]{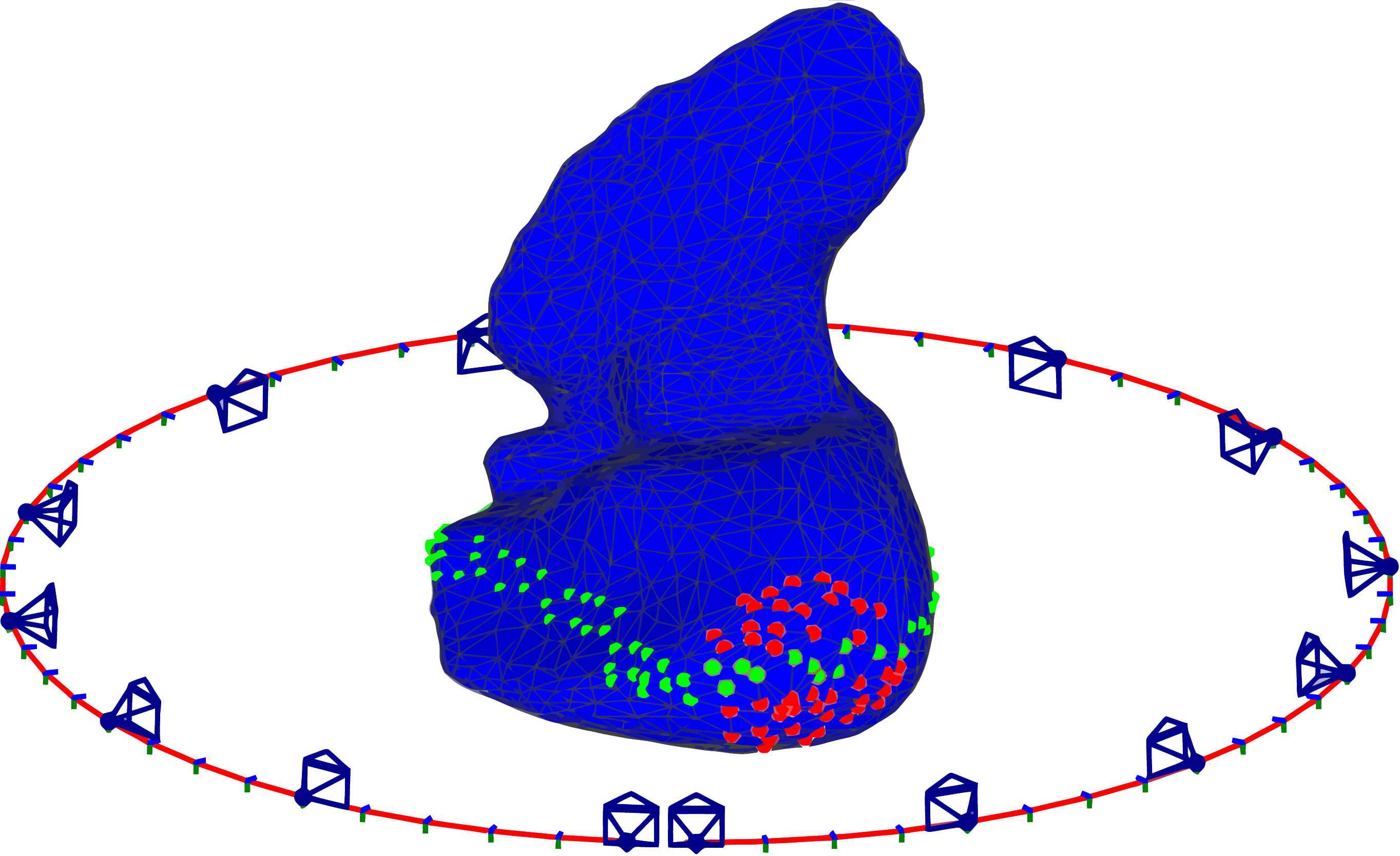}
	\caption{The simulated liver dataset. These dataset contains $60$ meshes, deformed from a template liver mesh by the As-Rigid-As-Possible method~\citet*{ARAP-surface-Alexander2007}. Each mesh has $2002$ vertices with known correspondences. We observe these $60$ meshes from different perspectives, by assigning $60$ poses along a simulated circular trajectory, where only parts of the sensor poses are shown as blue pyramids.	
The green dots denote the control points used to generate deformations, and the red dots denote the disabled correspondences in Figure~\ref{fig:liver_unobservable_parts}.}
	\label{fig:liver_dataset_overview}
\end{figure}
\begin{figure}[t]
	\centering
	\includegraphics[width=0.42\textwidth]{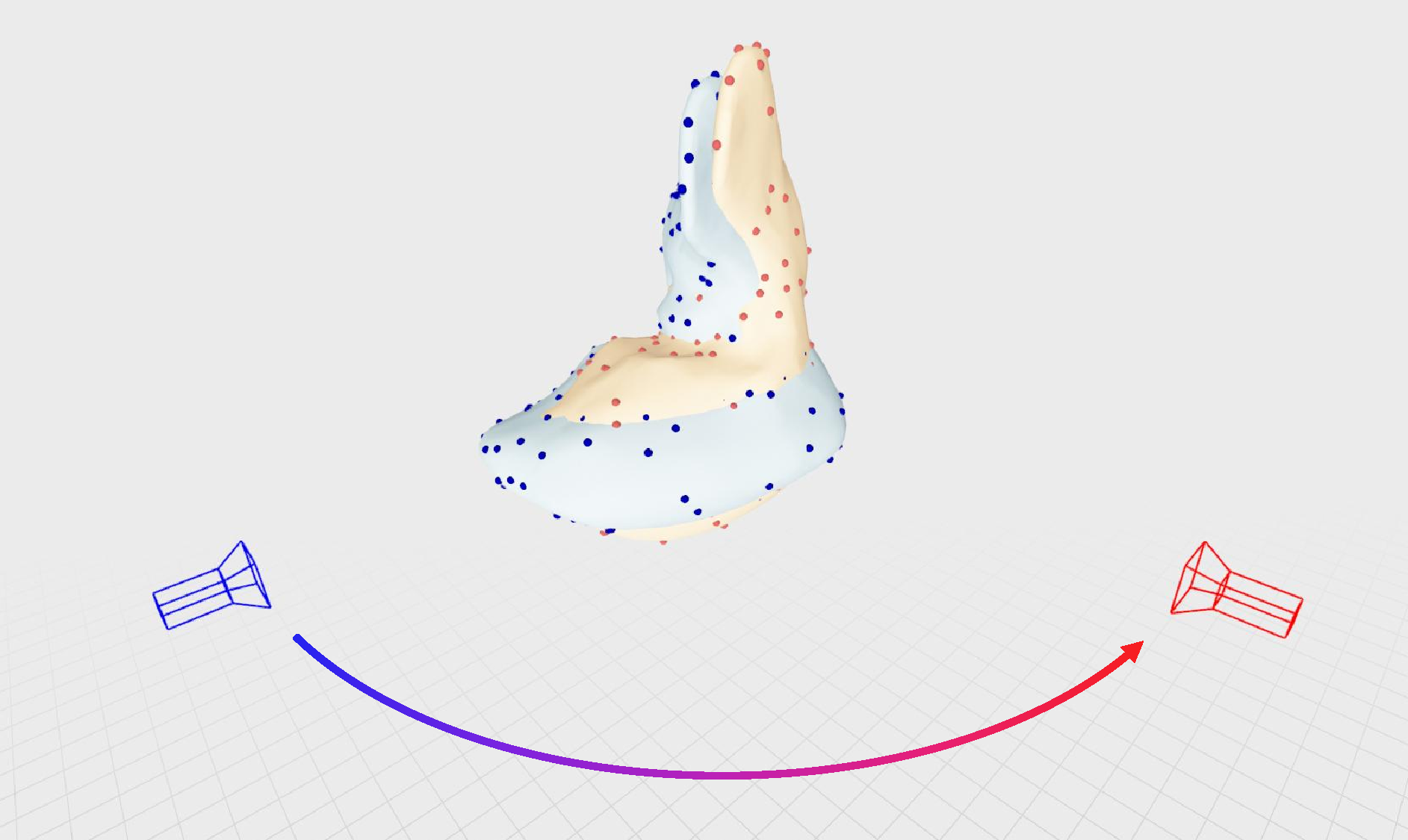}
	\caption{The simulated partial visibility of correspondences. The blue and red dots represent correspondences visible in the blue and red views, respectively. For a region without any correspondences, the deformation therein is never captured and thus is subject to information loss. Thus we drop correspondences randomly to simulate partial visibilities.}
	\label{fig:liver_with_two_viewpoints}
\end{figure}

\vspace{5pt}
\noindent
\textbf{Data generation.}
We use a segmented \textit{liver mesh} model, as shown in Figure~\ref{fig:liver_dataset_overview}, which has $2002$ vertices and $201$ of them are selected as correspondences.
We simulate deformations using the As-Rigid-As-Possible method~\citet*{ARAP-surface-Alexander2007} implemented in the CGAL\footnote{https://www.cgal.org} library.
We simulate a circular trajectory comprising $60$ poses, as shown in Figure~\ref{fig:liver_dataset_overview}.
For the reason of clarity, only parts of the poses are plotted as the pyramid shapes.
At each pose, the sensor observes a deformed mesh in its local coordinate frame, subject to partial visibilities and measurement noise:
\begin{itemize}
\item \textit{Partial visibility.} We randomly drop $30\%$ of the $201$ correspondences to simulate partial visibilities caused by correspondence detection failures, see Figure~\ref{fig:liver_with_two_viewpoints}.
\\[0pt]
\item \textit{Measurement noise.} We add zero-mean Gaussian noise with its standard-deviation set to $1$ mm, to simulate imperfect sensor measurements.
\end{itemize}

\begin{figure}[t]
	\centering
	\includegraphics[width=0.41\textwidth]{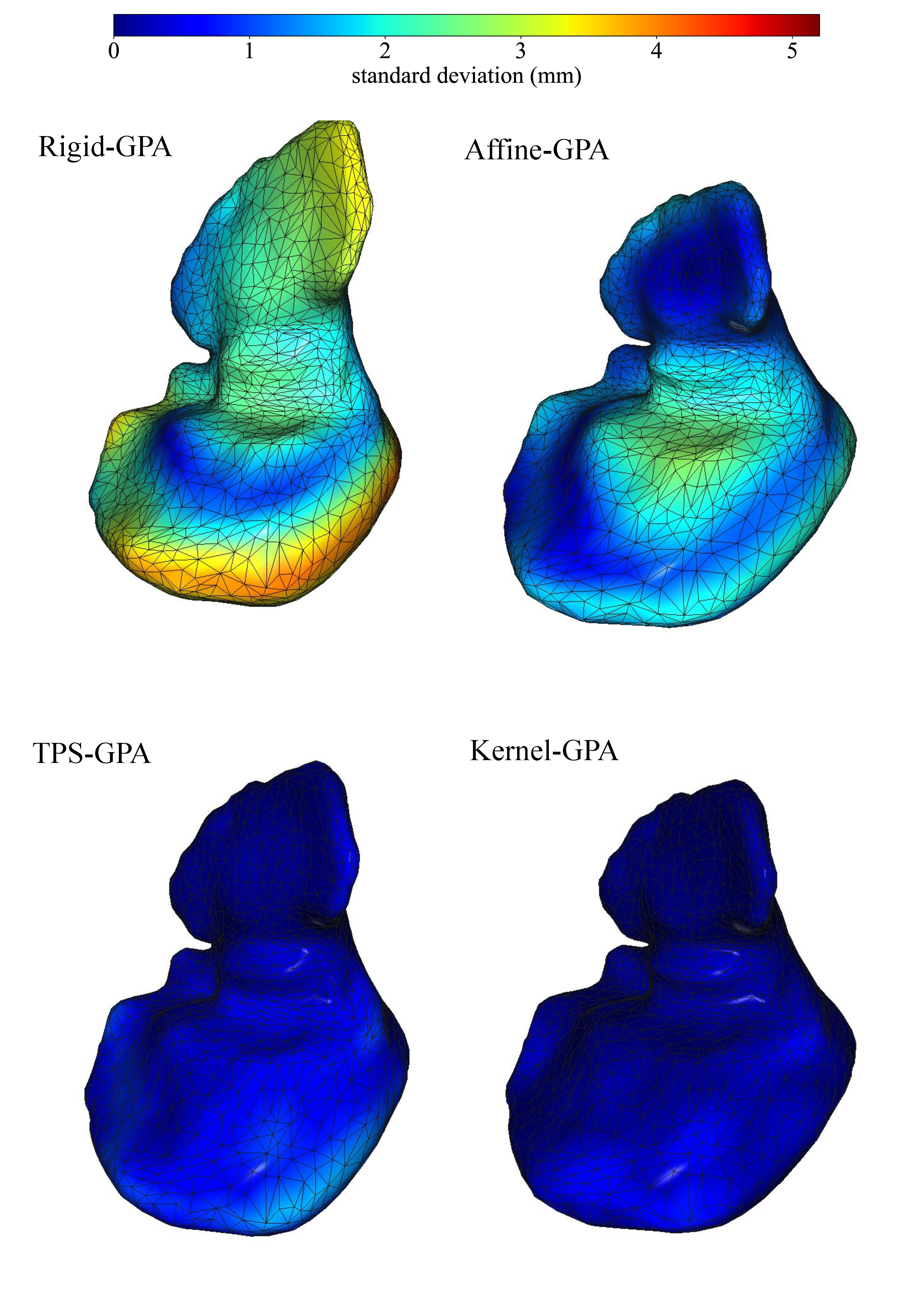}	
	\caption{The accuracy of different GPA methods on the liver dataset. We visualize the mean map $\matf{\check{M}}$ and encode point-wise consistencies of the test points $\boldsymbol{\check{\delta}}$ with color. Both TPS-GPA and Kernel-{\color{black}GPA} give significantly better performance.}
	\label{fig:deformed_mesh_error}
\end{figure}

\vspace{5pt}
\noindent
\textbf{Evaluation.}
We set tuning parameters $p = 0.25$ and $\mu = 0.1$.
We compute the GPA registration using the downsampled $201$ correspondences, and then evaluate the performance of different GPA methods using all the $2002$ correspondences.
For each tested case, we report the minimum, maximum and mean of the point-wise registration error $\boldsymbol{\check{\delta}}$ in Table~\ref{tab:liver_deformation_error}.
We visualize the mean map $\matf{\check{M}}$, and the point-wise registration error $\boldsymbol{\check{\delta}}$ in~Figure~\ref{fig:deformed_mesh_error}, by using the case where the meshes are fully-observable without noise.
It can be seen that GPA with deformable transformations (\textit{i.e.,}~TPS-GPA and Kernel-GPA) can significantly outperform classical Rigid-GPA and Affine-GPA methods.
The proposed Kernel-GPA method gives better results for regions with larger deformations.

\begin{figure}[th]
	\centering
	\includegraphics[width=0.41\textwidth]{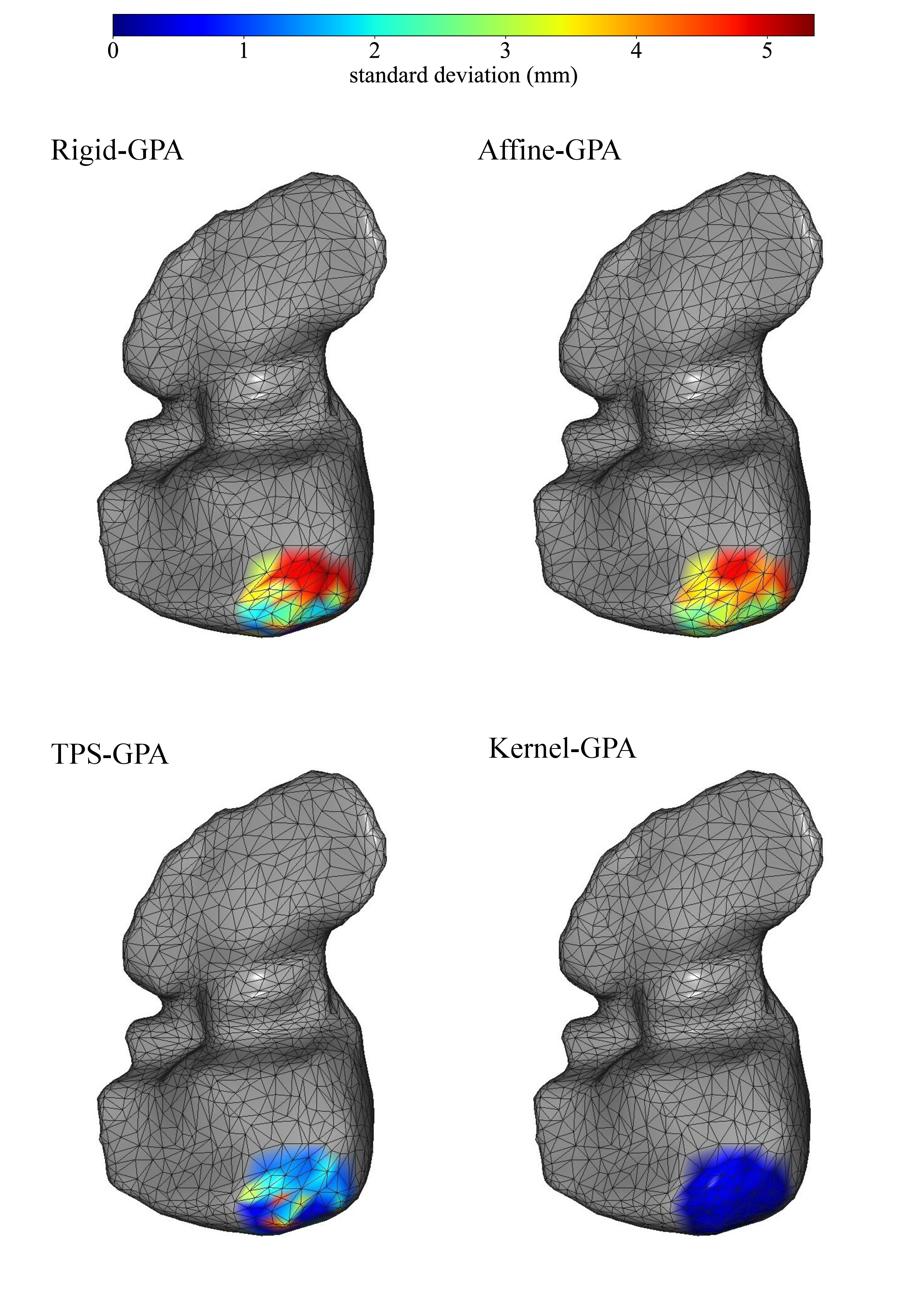}
	\caption{Extrapolation on the liver dataset. We disable the correspondences in the colored region across all the $60$ meshes. We solve GPA without the disabled correspondences, and then transform the test points in the region to construct a predicted mean surface. The reconstruction error of the predicted mean surface is given as the point-wise consistencies of the transformed test points, color coded.}
	\label{fig:liver_unobservable_parts}
\end{figure}

We further set a small region of the liver to be invisible in all the $60$ measurements, as seen in Figure~\ref{fig:liver_unobservable_parts}, and use the mesh vertices therein as test points.
In this test, we extrapolate the situation in the invisible region using $\vecf{y}_t(\cdot)$ computed from correspondences outside the invisible region.
The predicted mean map $\matf{\check{M}}$ and the point-wise consistencies of the transformed test points $\boldsymbol{\check{\delta}}$ are shown in Figure~\ref{fig:liver_unobservable_parts} for each GPA method.
This result further backs our claim on the superior performance of TPS-GPA and Kernel-GPA, where both methods can extrapolate the deformation in the invisible region with very similar performances.

Overall, for smooth deformations, we find both TPS-GPA and Kernel-GPA can give satisfactory results.

\begin{figure*}[t]
	\centering
	\begin{subfigure}[t]{0.18\textwidth}
		\centering
		\includegraphics[width=\textwidth]{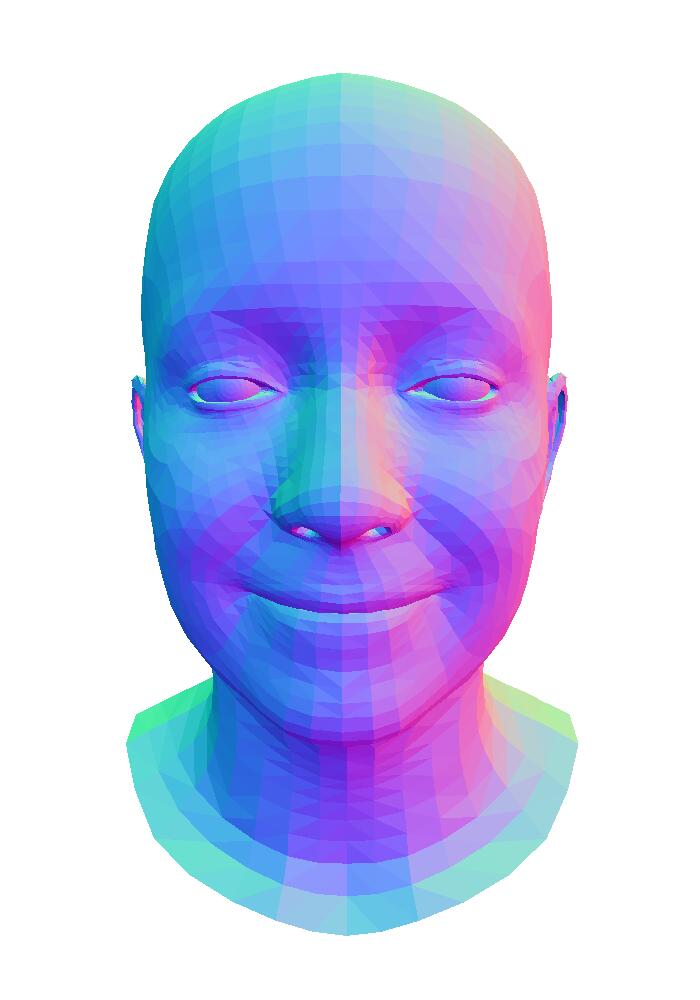}
	\end{subfigure}
	\hfil
	\begin{subfigure}[t]{0.18\textwidth}  
		\centering 
		\includegraphics[width=\textwidth]{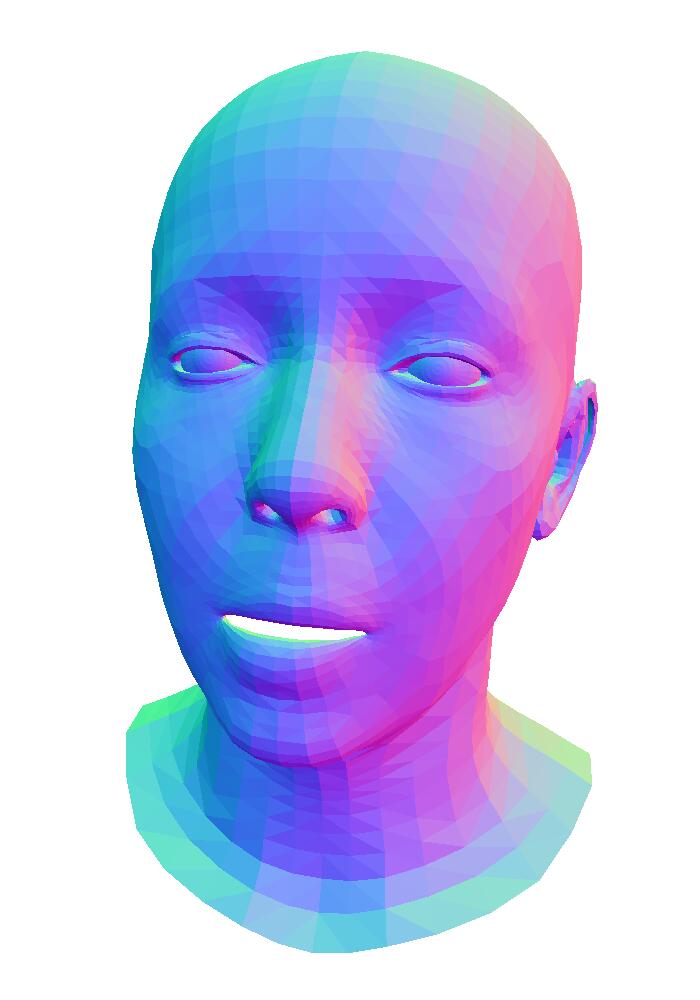}
	\end{subfigure}
	\hfil
	\begin{subfigure}[t]{0.18\textwidth}   
		\centering 
		\includegraphics[width=\textwidth]{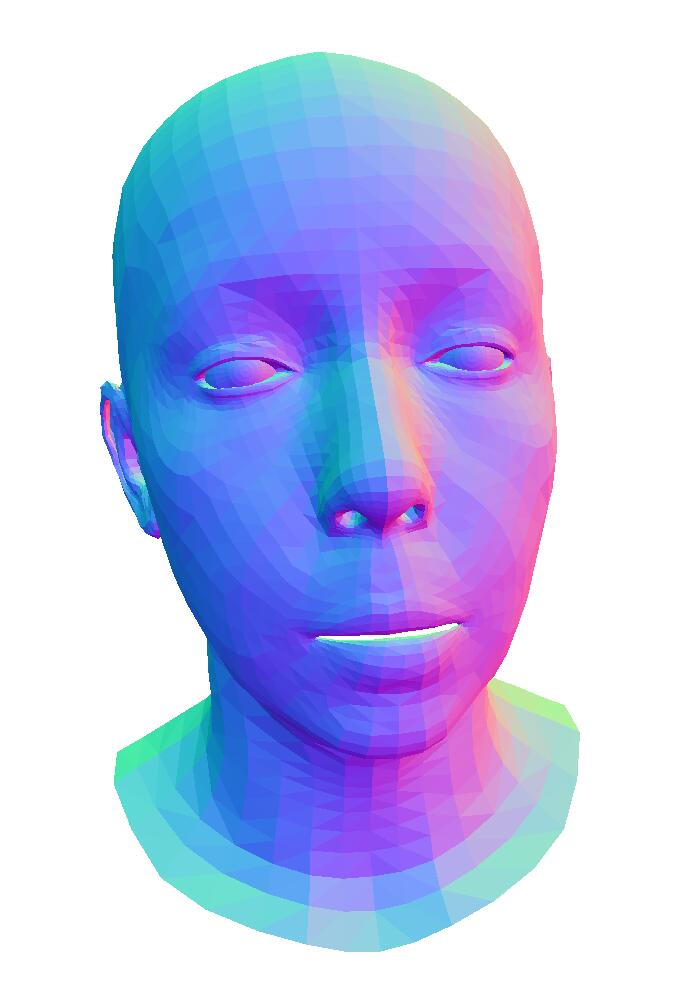}
	\end{subfigure}
	\hfil
	\begin{subfigure}[t]{0.18\textwidth}   
		\centering 
		\includegraphics[width=\textwidth]{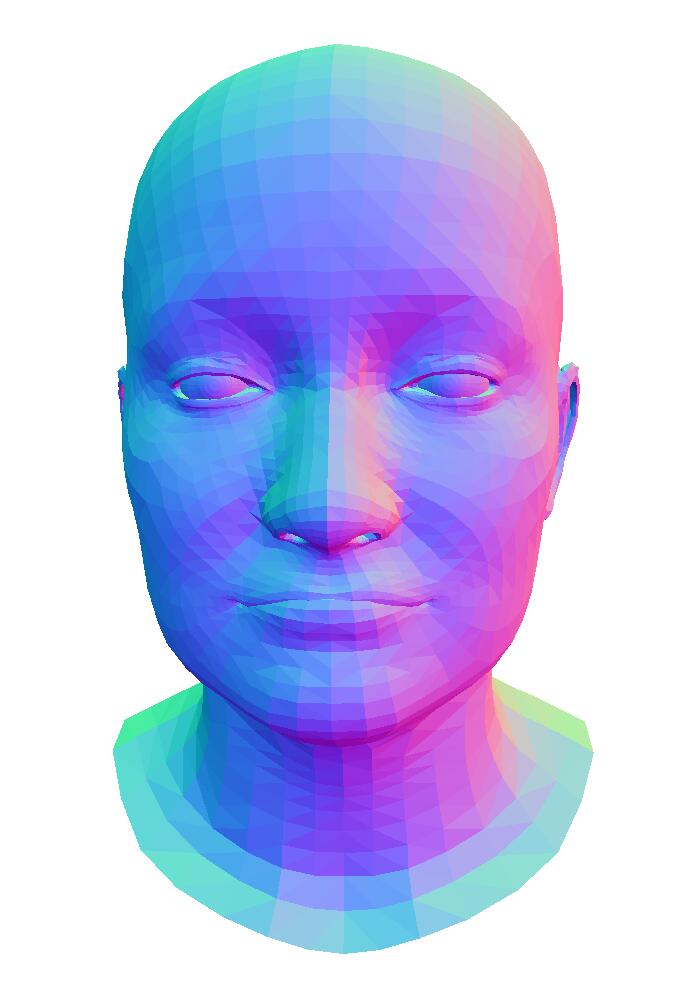}
	\end{subfigure}
	\hfil
	\begin{subfigure}[t]{0.18\textwidth}   
		\centering 
		\includegraphics[width=\textwidth]{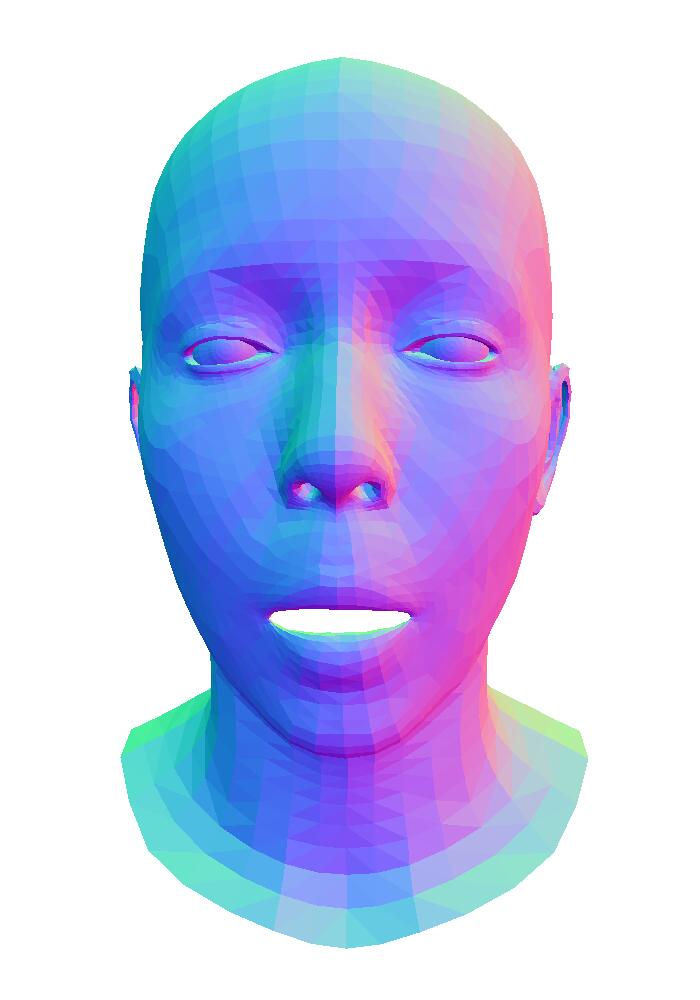}
	\end{subfigure}
	\caption{The facial expression dataset. We reconstruct the 3D model using DECA. The expressions from left to the right are respectively: 1) smile face, 2) curling the lip to the left, 3) curling the lip to the right, 4) cheek blowing, and 5) opening the mouth.}
	\label{fig:deca_meshes}
\end{figure*}

\subsection{Facial Expression}
\label{exp. sec. facial expression}

\begin{table*}[t]
	\centering
	\caption{The statistics of different GPA methods on the facial expression dataset.}
	\renewcommand{\arraystretch}{1.0}
	\begin{tabular}{p{2.5cm} p{2cm} | p{2cm} p{2cm} p{2cm} p{2cm}}
		      &   & Rigid-GPA & Affine-GPA & TPS-GPA & Kernel-GPA \\
		\hline
		\multirow{3}{8em}{smiling} & min (mm) & 0.129 &  0.138 & 0.152 & \textbf{0.003}\\
		& max (mm) & 9.901 & 10.536 & 10.518 & \textbf{2.797} \\
		& mean (mm) & 1.552 & 1.564 & 1.553 & \textbf{0.374} \\
		\hline
		\multirow{3}{8em}{curling left} & min (mm) &  0.099 & 0.138 & 0.152 & \textbf{0.002} \\
		& max (mm) & 8.415 & 7.583 & 7.559 & \textbf{4.290} \\
		& mean (mm) & 1.797 & 1.579 & 1.567 & \textbf{0.309} \\
		\hline
		\multirow{3}{8em}{curling right} & min (mm) &  0.158 & 0.103 & 0.090 & \textbf{0.004} \\
		& max (mm) & 8.265 & 8.283 & 8.251 & \textbf{3.613} \\
		& mean (mm) & 1.730 & 1.652 & 1.640 & \textbf{0.400} \\
		\hline
		\multirow{3}{8em}{cheek blowing} & min (mm) & 0.186 & 0.065 & 0.073 & \textbf{0.004} \\
		& max (mm) & 16.941 & 13.580 & 13.507 & \textbf{7.209} \\
		& mean (mm) & 3.445 & 2.984 & 2.955 & \textbf{0.527} \\
		\hline	
		\multirow{3}{8em}{opening mouth} & min (mm) & 0.240 & 0.156 & 0.145 & \textbf{0.003} \\
		& max (mm) & 11.998 & 10.280 & 10.207 & \textbf{4.747} \\
		& mean (mm) & 2.173 & 2.343 & 2.317 & \textbf{0.368} \\
	\end{tabular}
	\label{tab:face_deformation_error}
\end{table*}

\begin{figure}[ht]
	\centering 
	\includegraphics[width=0.42\textwidth]{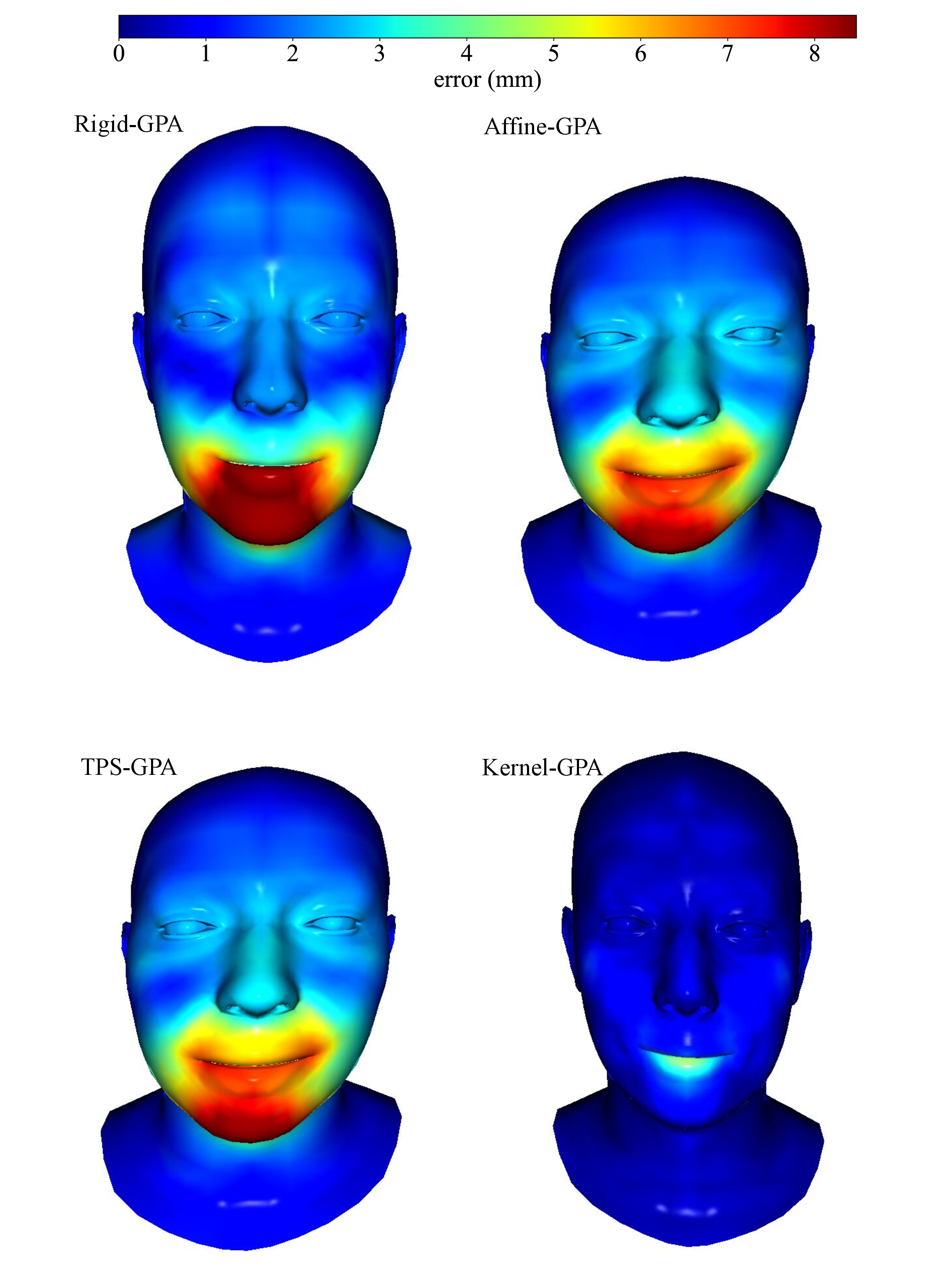}
	\caption{The face model $\matf{\check{M}}$ reconstructed from each GPA method, textured with the point-wise consistencies $\boldsymbol{\check{\delta}}$.}
	\label{fig:face_template}
\end{figure}

\begin{figure*}[ht]
		\centering 
		\includegraphics[width=0.95\textwidth]{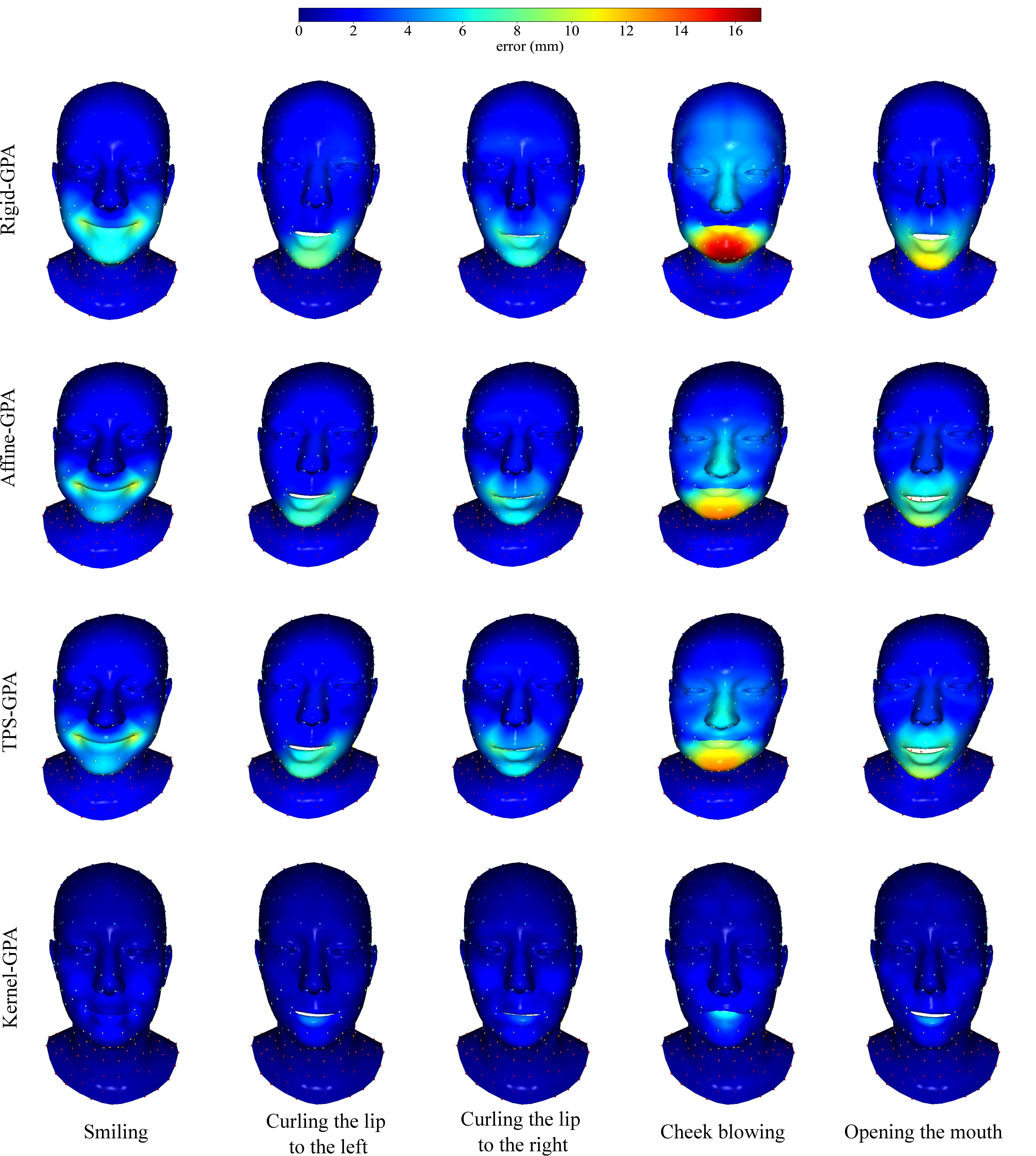}
	\caption{The deformable transformations $\vecf{y}_t(\cdot)$ of different GPA methods on the facial expression dataset. We visualize the shape of the transformed test points $\vecf{y}_t (\matf{\check{P}}_t)$ and encode the point-wise discrepancies between $\vecf{y}_t (\matf{\check{P}}_t)$ and the mean map $\matf{\check{M}} \matf{\check{\Gamma}}_t$ with color. The markers on the face represent the transformed correspondences, \textit{i.e.,}~$\vecf{y}_t (\matf{P}_{t})$.}
	\label{fig:face_deformation_error}
\end{figure*}

\begin{figure*}[ht]
	\centering 
	\begin{subfigure}[c]{0.43\textwidth}
		\centering
		\caption{\textrm{Disabled correspondences in red}}				
		\includegraphics[width=0.35\textwidth]{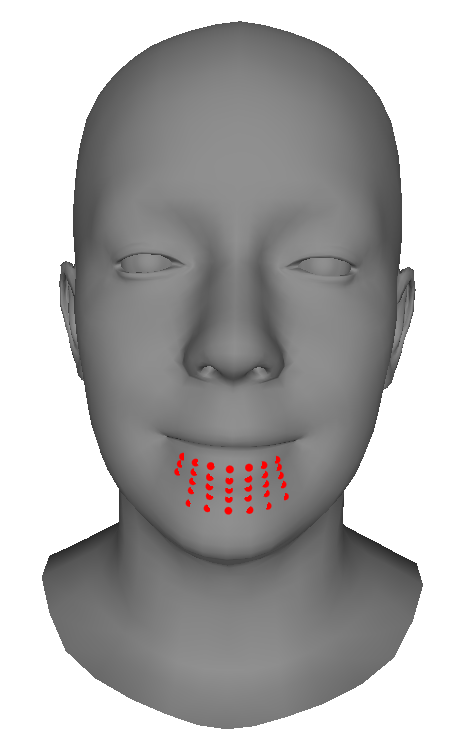}
		\label{fig:face_missing_points}
	\end{subfigure}	
	\begin{subfigure}[c]{0.43\textwidth}
		\centering
		\includegraphics[width=\textwidth]{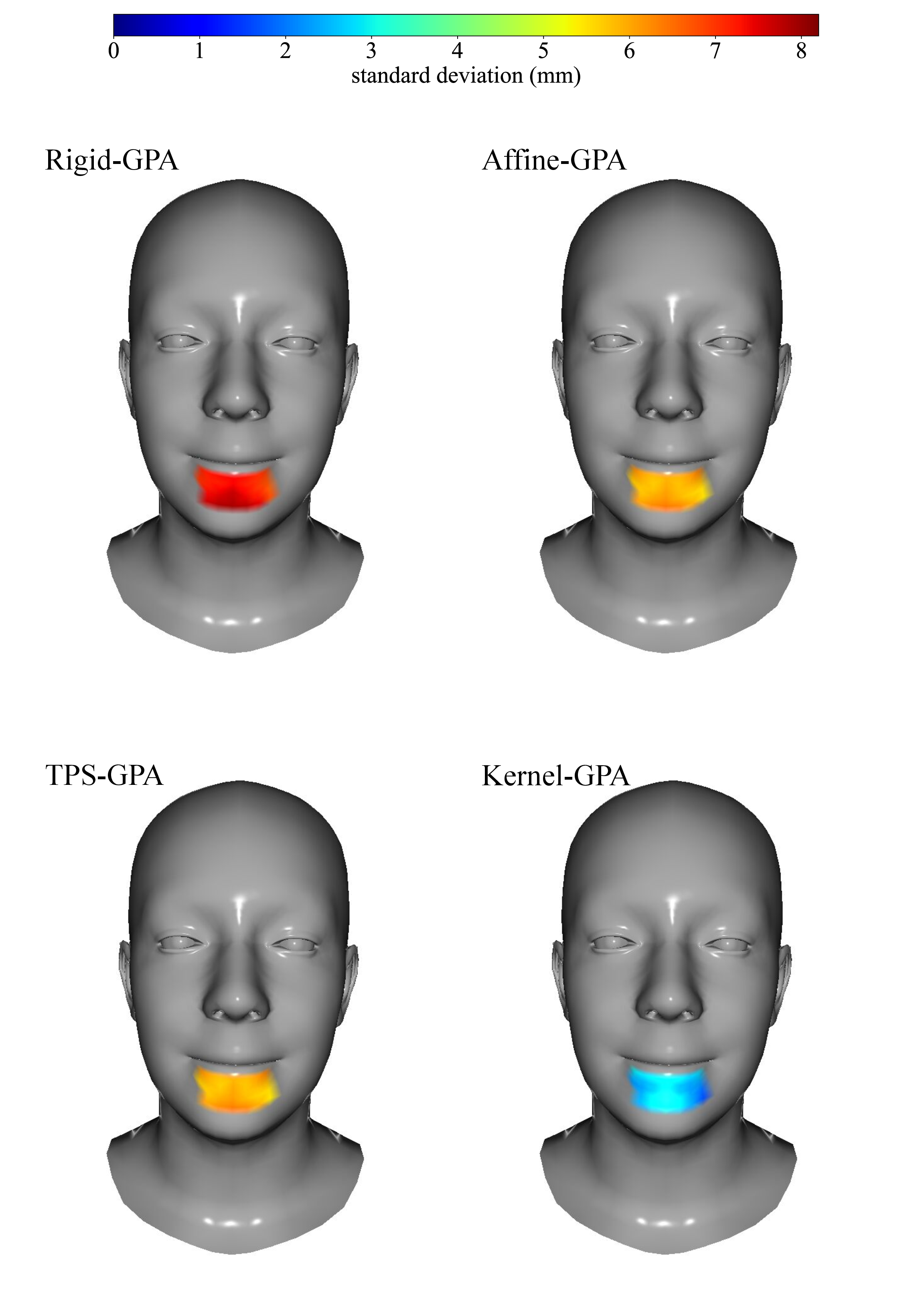}
		\caption{\textrm{Extrapolation error of each GPA method}}				
		\label{fig:face_deformation_prediction}
	\end{subfigure}		
	\caption{Extrapolation on the facial expression dataset. We solve GPA without the correspondences around the chin, and then transform the test points in the chin area to construct a predicted mean surface to to fill the hole. The point-wise consistencies of the transformed test points are color coded.}
\end{figure*}

\vspace{5pt}
\noindent
\textbf{Data generation.}
We create a \textit{facial expression} dataset which contains the meshes of $5$ facial expressions: 1) smiling, 2) curling the lip to the left, 3) curling the lip to the right, 4) cheek blowing and 5) opening the mouth, as shown in Figure~\ref{fig:deca_meshes}. The meshes of the head model are reconstructed with detailed facial geometry from a single input image using off-the-shelf toolbox DECA\footnote{https://deca.is.tue.mpg.de} from \citet*{DECA:Siggraph2021}. In the reconstructed meshes, the indices of the vertices are consistent thus the correspondences are available.
There are $5118$ vertices in total for each mesh, and we select $326$ as correspondences for GPA registration and the rest for test.

\vspace{5pt}
\noindent
\textbf{Evaluation.}
We set tuning parameters $p = 0.25$ and $\mu = 0.2$.
We use the selected $326$ correspondences to solve GPA, and test the registration performance using all the $5118$ points.
We first show the reconstructed mean maps $\matf{\check{M}}$ for each GPA method in Figure~\ref{fig:face_template}, and encode the point-wise consistencies $\boldsymbol{\check{\delta}}$ with textures.
We specifically examine the discrepancy between the transformed test points $\vecf{y}_t (\matf{\check{P}}_t)$ and the mean map $\matf{\check{M}}$, for each $t$ individually.
The statistics are reported in Table~\ref{tab:face_deformation_error}, and the visualization is given in Figure~\ref{fig:face_deformation_error}.
For structural deformations, Kernel-GPA significantly outperforms the other methods, owing to its capability to handle \textit{e.g.,}~the $4$-th \textit{cheek blowing} point-cloud.
Such data are challenging for TPS-GPA, as facial expressions are less smooth, with particularly large deformations on the cheek, around the nose and the mouth.

We examine the extrapolation ability of $\vecf{y}_t (\cdot)$ around the chin area, as shown in Figure~\ref{fig:face_missing_points}, by disabling the correspondences in the selected region.
We solve GPA without the disabled correspondences, and use the computed $\vecf{y}_t (\cdot)$ to extrapolate the deformation.
Within the region, the predicted mean map $\matf{\check{M}}$ and the point-wise consistencies of the transformed test points $\boldsymbol{\check{\delta}}$ are visualized in Figure~\ref{fig:face_deformation_prediction}.
It can be seen that the Kernel-GPA gives significantly better prediction compared with the other three GPA methods, confirming the superior modeling power of the KBT.

Overall, for structural deformations, we find the proposed Kernel-GPA method outperforms the TPS-GPA , the Affine-GPA and the Rigid-GPA methods.

\subsection{CT Point-cloud}

\begin{figure*}[t]
	\centering
	\begin{overpic}[ width=0.62\textwidth]{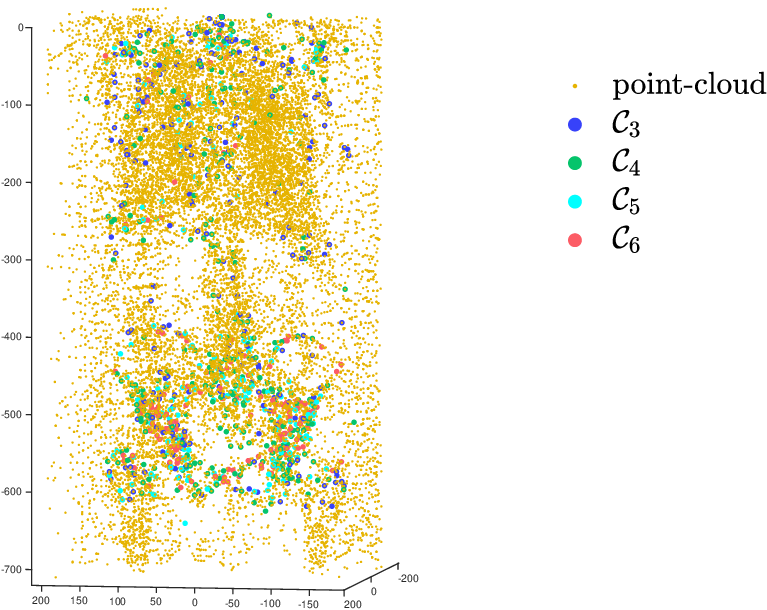}
\put(70,20){\parbox{0.4\textwidth}{\scalebox{1.0}{
	\begin{tabular}{ c c }
	\multicolumn{2}{c}{no. of correspondences} \\
	\toprule
	$\mathcal{C}_3$ &   480  \\
	$\mathcal{C}_4$ &   482  \\
	$\mathcal{C}_5$ &   211  \\
	$\mathcal{C}_6$ &   147  \\
	\midrule
	total &  1320 \\
	\bottomrule
	\end{tabular}}
}}
	\end{overpic}
\begin{tabular}{c}
	\\[20pt]
\end{tabular}
	\caption{The TOPACS point-cloud dataset. The point-clouds are processed from real CT scans using SURF3D features. This dataset contains $6$ point-clouds, with $20000$ points for each point-cloud. There are in total $1320$ correspondences classified into four categories according to their occurrences.}
	\label{fig. sample point-cloud of the TOPACS dataset}	
\end{figure*}

\begin{figure*}[t]
	\centering
	\includegraphics[width=0.72\textwidth]{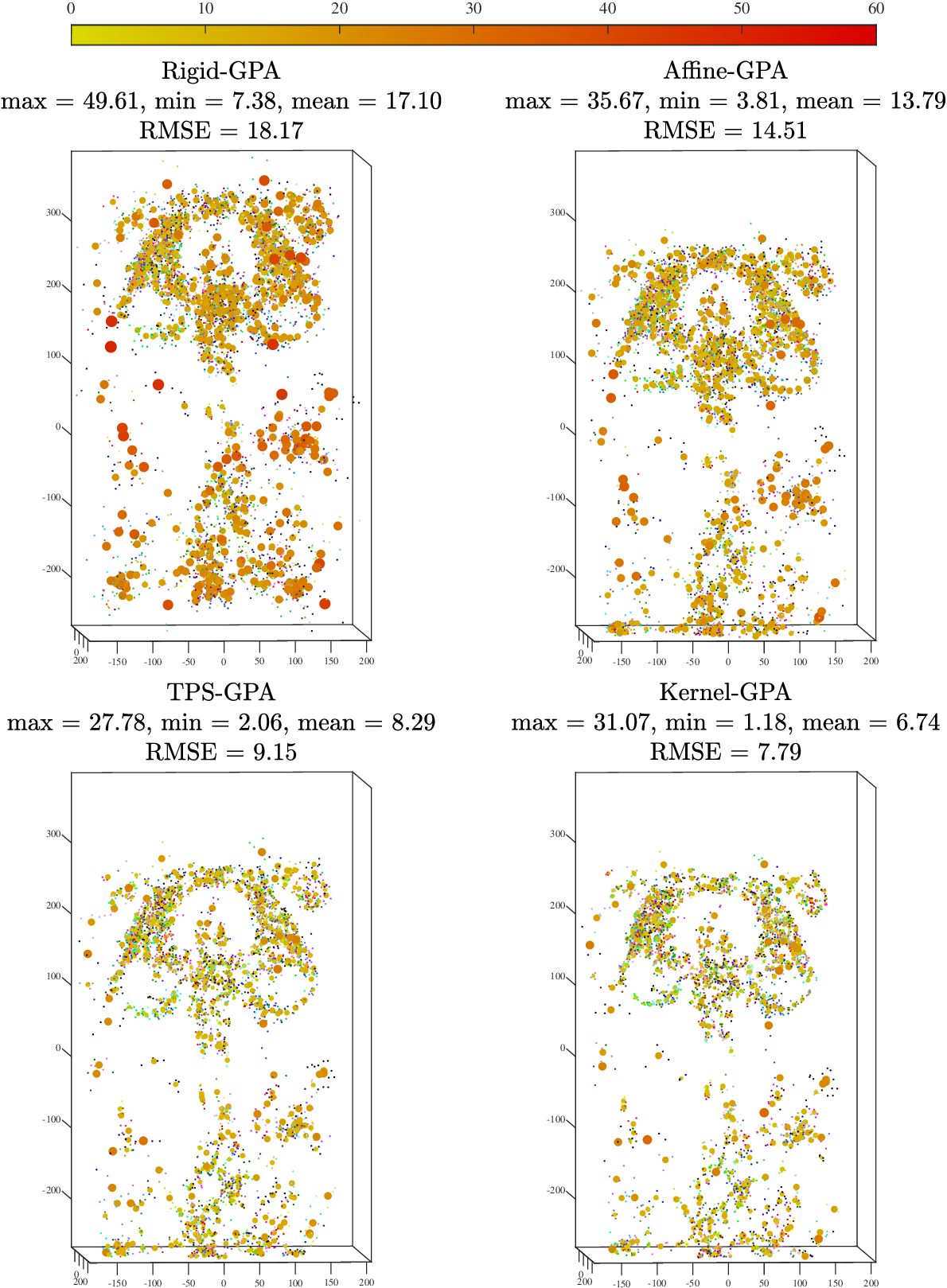}
	\caption{The registration of different GPA methods on the TOPACS dataset, with exactly the same tuning parameters used in Table~\ref{table. different GPA methods on the TOPACS dataset}. In this example, we use $\mathcal{C}_{3}$ to solve GPA, and then show the final registration result by visualizing the point-wise consistencies of $\mathcal{C}_{4},\mathcal{C}_{5},\mathcal{C}_{6}$, as both size-coded and color-coded with the filled circles. The smaller the marker size, the better. The transformed correspondences are also plotted as colored dots, where the corresponding points are plotted with the same color.}
\label{fig. The registration results of different GPA methods}
\end{figure*}

\vspace{5pt}
\noindent
\textbf{Data generation.}
We provide a dataset, termed TOPACS, for computerized tomography (CT) registration.
The CT data we use, shown in Figure~\ref{fig. sample point-cloud of the TOPACS dataset}, contain $6$ scans of lungs, which are processed by the SURF3D features \citet*{Raju1993surf3dA3} resulting in $6$ point-clouds (with $20000$ points for each point-cloud).
Initial correspondences are found by matching feature descriptors and then refined by an ICP algorithm.
The global correspondences are found by a graph matching algorithm, and the ambiguous ones are removed based on distances.
We categorize the correspondences into four sets
$\mathcal{C}_3,\, \mathcal{C}_4,\, \mathcal{C}_5,\, \mathcal{C}_6$
by their visibilities.
For example, $\mathcal{C}_3$ collects the correspondences visible in exactly three point-clouds, and others are defined analogously.

\begin{table*}[th]
	\caption{The statistics of different GPA methods on the TOPACS dataset.}
	\label{table. different GPA methods on the TOPACS dataset}
	\renewcommand{\arraystretch}{1.0}
	\centering	
\begin{tabular}{ p{2cm} p{2cm} p{2cm} | p{2cm} p{2cm} p{2cm} p{2cm} } 
 	registration  & test   &   & Rigid-GPA & Affine-GPA  & TPS-GPA  &  Kernel-{\color{black}GPA}  \\ 
\hline
   &    &  min (mm)  & 7.38  &  3.81  &  2.06  &  \textbf{1.18} \\ 
 $\mathcal{C}_{3}$  &  $\mathcal{C}_{4}$$\mathcal{C}_{5}$$\mathcal{C}_{6}$  &  max (mm)  & 49.61  &  35.67  &  \textbf{27.78}  &  31.07 \\ 
   &    &  mean (mm)  & 17.10  &  13.79  &  8.29  &  \textbf{6.74} \\ 
\hline
   &    &  min (mm)  & 3.26  &  2.86  &  1.19  &  \textbf{0.55} \\ 
 $\mathcal{C}_{4}$  &  $\mathcal{C}_{3}$$\mathcal{C}_{5}$$\mathcal{C}_{6}$  &  max (mm)  & 59.41  &  42.36  &  36.79  &  \textbf{34.68} \\ 
   &    &  mean (mm)  & 17.30  &  13.29  &  8.15  &  \textbf{6.62} \\ 
\hline
   &    &  min (mm)  & 2.76  &  2.57  &  1.56  &  \textbf{0.59} \\ 
 $\mathcal{C}_{5}$  &  $\mathcal{C}_{3}$$\mathcal{C}_{4}$$\mathcal{C}_{6}$  &  max (mm)  & 61.20  &  45.78  &  37.32  &  \textbf{35.69} \\ 
   &    &  mean (mm)  & 17.13  &  13.22  &  8.57  &  \textbf{7.64} \\ 
\hline
   &    &  min (mm)  & 2.21  &  1.63  &  1.05  &  \textbf{0.88} \\ 
 $\mathcal{C}_{6}$  &  $\mathcal{C}_{3}$$\mathcal{C}_{4}$$\mathcal{C}_{5}$  &  max (mm)  & 63.71  &  56.54  &  47.97  &  \textbf{46.82} \\ 
   &    &  mean (mm)  & 17.66  &  14.12  &  10.04  &  \textbf{9.94} \\ 
\end{tabular}
\end{table*}

\vspace{5pt}
\noindent
\textbf{Evaluation.}
We set tuning parameters $p = 0.20$ and $\mu = 0.05$.
We use one category of the correspondences $\mathcal{C}_{k}$ to to solve GPA, which gives an estimate of deformable transformations $\vecf{y}_t(\cdot)$.
Then we use the other categories $\mathcal{C}_{j}\ (j\neq k)$ as the test points.
We compute the point-wise consistencies of the transformed test points $\boldsymbol{\check{\delta}}$ (thanks to the known correspondences), and report the statistics in Table~\ref{table. different GPA methods on the TOPACS dataset}.
For this dataset, we see a remarkable reduction of the registration error from the Rigid-GPA to the Affine-GPA.
This is probably because of the fact that each point-cloud is for a different subject, and the subject's morphology varies a lot in width, length and thickness.
Another explanation is that a large portion of deformations are caused by the stretching of shoulders in the preparation process of the CT scanning. Such shearing is well-modeled by the affine transformation.
We further see that both the TPS-GPA and the Kernel-GPA methods can further improve the performance of the affine-{\color{black}GPA}, owing to their capabilities to model nonlinear deformations caused by breathing.

We provide a visualization in Figure~\ref{fig. The registration results of different GPA methods}, by using the correspondences $\mathcal{C}_{3}$ to solve GPA and $\mathcal{C}_{4},\mathcal{C}_{5},\mathcal{C}_{6}$ for validation,
as this is the worst case in~Table \ref{table. different GPA methods on the TOPACS dataset}.
We see that both the TPS-GPA and the Kernel-GPA methods give similar performances, while there are marginal differences in the statistics.
This can be understood as the underlying deformations are smooth, similar to the liver data studied in Section~\ref{exp. sec. liver data}.
This further backs the claim that both TPS-GPA and Kernel-GPA are suitable for surgical applications, while Kernel-GPA is preferred in case of more complex deformations, \textit{e.g.,}~the facial expression data studied in Section~\ref{exp. sec. facial expression}.

In contrast to the simulated liver in Section~\ref{exp. sec. liver data} with perfect correspondences, the correspondences from CT scans (\textit{i.e.,}~slices of gray images) are never perfect due to the lack of textures and are thus subject to mismatches (\textit{i.e.,}~outliers).
However, we show that the correspondence based method works well, even for such challenging CT data.
In practice, we expect better performance of GPA methods, if the correspondences are extracted from RGB images.

%
%
%
%
%
%
%
%
%
%

\section{Discussion and Conclusion}
\label{section. conclusion}

We have proposed KernelGPA, a novel GPA method using the KBT as the deformable transformation.
We have given detailed mathematical derivations to show the point that KernelGPA can be solved globally in closed-form up to some global scale ambiguities.
We have proposed to estimate the global scale ambiguities by an optimization formulation that favors rigidity, which has also allowed us to give insights on the degenerate cases.
While we have implemented KernelGPA using the Gaussian kernel, the proposed KernelGPA can be implemented using any positive definite kernel, \textit{e.g.,}~the Laplacian kernel.
We have validated the performance of KernelGPA with both simulated and real datasets.
{\color{black}Our Matlab code and data are publicly available for future comparison.}

{\color{black}

\vspace{5pt}
\noindent
\textbf{Computational complexity.}
Computationally, the complexity of KernelGPA is characterized by the number of correspondences used, and largely independent of the number of poses.
In specific, the most expensive part of KernelGPA comprises: 1) the construction of matrix $\matf{Q}_t$ in equation (\ref{eq. expression of Qt}) which requires the inversion of matrix $\matf{S}_t$, and 2) the Eigen decomposition of matrix $\boldsymbol{\mathcal{Q}}$ to solve formulation (\ref{eq: the minimization problem in Map only}).
The dimension of both $\matf{S}_t$ and $\boldsymbol{\mathcal{Q}}$ are decided by the number of correspondences used.
In practice, for example in medical applications, the number of correspondences are typically limited to a few hundred, which can be handled effectively.
For certain cases, if a large number of correspondences are available, we suggest selecting a reasonable number of robust correspondences that cover the deformable part of the scene.

\vspace{5pt}
\noindent
\textbf{Connection to the Tomasi-Kanade factorization.}
For the affine case, the affine transformation and the canonical map can be jointly factorized by the singular value decomposition (SVD), see Section 3.2 in~\citet*{bai2022ijcv} and the $\ast$AFF\_d method. This SVD approach is in the same spirit of the Tomasi-Kanade factorization~\citet*{tomasi1992shape} in computer vision based on the orthographic camera model. The SVD approach has been extended to handle structural deformations, see~\citet*{bregler2000recovering}.
In this work, we have proposed an alternative factorization method based on the Eigen decomposition.
As for the cost function, the residual of the SVD approach is defined in the coordinate frame of the sensor, while the residual of our Eigen approach is defined in the coordinate frame of the canonical map.
Critically, we show in Section~\ref{sec. Transformation Constraint} that the geometry of the canonical map $\matf{M}$ can be defined up to $d$ global scale ambiguities.
This point is not realized in the classical SVD approach, which thus does not use the constraints in Section~\ref{sec. Transformation Constraint} to further reduce the ambiguities.
As a result, the SVD approach gives a solution up to a global affine transformation, while our Eigen approach gives a solution up to $d$ global scale ambiguities.
Recall that the affine transformation has $d^2$ parameters (not considering the translation), hence more than the $d$ of our method.

\vspace{5pt}
\noindent
\textbf{Future work.}
The future work includes handling pose ambiguities (for example, by adding additional sensor information or deformation assumptions), incorporating probabilistic models to handle non-isotropic noise, extending the proposed GPA method to SfM problems, exploring different kernel functions, and exploiting the sparsity of the kernel matrix for even faster computation.

}

\appendix


\section{Brockett Cost Function on the Stiefel Manifold}
\label{section. Brockett Cost Function on the Stiefel Manifold}

\begin{definition}[(Brockett cost function on the Stiefel manifold)]
The matrix Stiefel manifold is the set of matrices satisfying:
\begin{equation*}
\mathrm{St}(d, m)
=
\left\{  \matf{X} \in \mathbb{R}^{m \times d}
\ \vert\ 
\matf{X}^{\trans} \matf{X} = \matf{I} \right\}
.	
\end{equation*}
The following function defined on the Stiefel manifold is termed the Brockett cost function~\citet*{brockett1989least, absil2009optimization}:
\begin{equation}
\label{eq: Stiefel manifold optimzation, Brockett cost fucntion}
f_{ \matf{X} \in \mathrm{St}(d, m)} (\matf{X})
=
\mathrm{tr}\left(  \matf{X}^{\trans} 
\matf{\Pi}
\matf{X}  \matf{\Lambda} \right)
,
\end{equation}
where $\matf{\Pi} \in \mathbb{R}^{m \times m}$ is symmetric,
and
$
\boldsymbol{\Lambda}
=
\mathbf{diag}(\lambda_1, \lambda_{2}, \dots, \lambda_{d})
$
with
$
\lambda_1 \ge \lambda_{2} \ge \dots \ge \lambda_{d} \ge 0
$.
\end{definition}

\begin{lemma}[(\citet*{brockett1989least, birtea2019first, absil2009optimization})]
The critical points of the Brockett cost function $f_{ \matf{X} \in \mathrm{St}(d, m)} (\matf{X})$ on the Stiefel manifold are the eigenvectors of $\matf{\Pi}$.
\end{lemma}

If we choose $\matf{X'} = [
\boldsymbol{\xi}_1,\,
\boldsymbol{\xi}_2,\,
\dots,
\boldsymbol{\xi}_d
]$ with $\matf{\Pi} \boldsymbol{\xi}_k = \alpha_k \boldsymbol{\xi}_k$ (\textit{i.e.,}~$\boldsymbol{\xi}_k$ is an eigenvector of $\matf{\Pi}$ corresponding to eigenvalue $\alpha_k$), then we have cost 
$f_{ \matf{X'} \in \mathrm{St}(d, m)} (\matf{X'}) = \lambda_1 \boldsymbol{\xi}_1 +  \lambda_2 \boldsymbol{\xi}_2 + \dots +  \lambda_d \boldsymbol{\xi}_d$.

\begin{lemma}[(Hardy-Littlewood-Polya \citet*{hardy1952inequalities})]
For two sequences of real numbers $x_1 \ge x_2 \ge \dots \ge x_n$ and $y_1, y_2, \dots, y_n$ in any order, we consider:
$$
\eta = \sum_i^n x_i y_{\pi(i)}
,
$$
where $\pi$ denotes a permutation of indices in $[1 : n]$.
The maximum of $\eta$ is attained when $y_{\pi(1)} \ge y_{\pi(2)} \ge \dots \ge y_{\pi(n)}$.
The minimum of $\eta$ is attained when $y_{\pi(1)} \le y_{\pi(2)} \le \dots \le y_{\pi(n)}$.

\end{lemma}

\begin{lemma}[(\citet*{brockett1989least})]

For a symmetric $\matf{\Pi} \in \mathbb{R}^{m \times m}$,
we denote its eigenvalue decomposition as:
$$
\matf{\Pi} = \matf{U} \boldsymbol{\Sigma} \matf{U}^{\trans} 
= \sum_{k=1}^{m} \sigma_k \vecf{u}_k \vecf{u}_k^{\trans}
,$$
where
$\matf{U}
=
\left[\vecf{u}_1, \vecf{u}_2, \dots, \vecf{u}_m\right]$ is an orthonormal matrix,
and
$\boldsymbol{\Sigma} = \mathbf{diag}\left(\sigma_1, \sigma_2, \dots, \sigma_m\right)$ with $\sigma_1 \ge  \sigma_2 \ge \dots \ge \sigma_m$.
Let
$
\boldsymbol{\Lambda}
=
\mathbf{diag}(\lambda_1, \lambda_{2}, \dots, \lambda_{d})
$
with
$
\lambda_1 \ge \lambda_{2} \ge \dots \ge \lambda_{d} \ge 0
$.
We have:
\begin{enumerate}
\item
$\max_{ \matf{X} \in \mathrm{St}(d, m)} \mathrm{tr}\left(  \matf{X}^{\trans} 
\matf{\Pi}
\matf{X}  \matf{\Lambda} \right)$ is attained at:
$$\matf{X} =  [\vecf{u}_1, \vecf{u}_{2}, \dots, \vecf{u}_{d}] ,$$
which comprises the $d$ top eigenvectors of $\matf{\Pi}$,
with cost $\lambda_1 \sigma_1 + \lambda_2 \sigma_{2} + \dots + \lambda_d \sigma_{d}$.
\vspace{8pt}
\item $\min_{ \matf{X} \in \mathrm{St}(d, m)} \mathrm{tr}\left(  \matf{X}^{\trans} 
\matf{\Pi}
\matf{X}  \matf{\Lambda} \right)$ is attained at:
$$\matf{X} =  [\vecf{u}_m, \vecf{u}_{m-1}, \dots, \vecf{u}_{m-d+1}] ,$$
which comprises the $d$ bottom eigenvectors of $\matf{\Pi}$,
with cost $\lambda_1 \sigma_m  +  \lambda_2 \sigma_{m-1} + \dots +  \lambda_d \sigma_{m-d+1}$.
\end{enumerate}
\end{lemma}

\section{Proof of Lemma \ref{theorem: general result: PCA with constraint XT u = 0}}
\label{appendix. proof of general result: PCA with constraint XT u = 0}

We first notice that when $\matf{X}^{\trans}  \vecf{u} = \vecf{0}$, the cost is equivalent to:
$$
\mathrm{tr}\left( \matf{X}^{\trans} 
\matf{\Pi}
\matf{X} \boldsymbol{\Lambda} \right)
=
\mathrm{tr}\left( \matf{X}^{\trans} 
\left( \matf{\Pi} -  c \vecf{u} \vecf{u}^{\trans} \right)
\matf{X} \matf{\Lambda} \right)
,
$$
where $c$ is an arbitrary scalar.
Importantly, by using different $c$, we can shift $\vecf{u}$ to the top or bottom eigenvector of $\matf{\Pi} -  c \vecf{u} \vecf{u}^{\trans}$.
We denote the eigenvalue decomposition of
$\matf{\Pi}$
as:
$$
\matf{\Pi} 
= \sum_{k=1}^{m} \sigma_k \vecf{u}_k \vecf{u}_k^{\trans}
,$$
with $\sigma_1 \ge  \sigma_2 \ge \dots \ge \sigma_m$.

\vspace{5pt}
\noindent
\textbf{Case 1.}
We consider $c > \sigma_1 - \sigma_m$, and the following relaxation of problem (\ref{eq: general results: argmax, PCA in S, with S1=0}) without constraint $\matf{X}^{\trans}  \vecf{u} = \vecf{0}$:
\begin{equation}
\label{eq: property: proof - final result in X}
\begin{aligned}
\max_{\matf{X}}\quad
&
\mathrm{tr}\left( \matf{X}^{\trans} 
\left( \matf{\Pi} -  c \vecf{u} \vecf{u}^{\trans} \right)
\matf{X} \matf{\Lambda} \right)
\\
&
\mathrm{s.t.} \quad
\matf{X}^{\trans}  \matf{X} = \matf{I}
.
\end{aligned}
\end{equation}
If $\matf{X}_*$ is optimal to problem (\ref{eq: property: proof - final result in X}) and satisfies $\matf{X}_{*}^{\trans}  \vecf{u} = \vecf{0}$, then $\matf{X}_*$ is optimal to problem (\ref{eq: general results: argmax, PCA in S, with S1=0}).

Problem (\ref{eq: property: proof - final result in X}) admits a Brockett cost on the Stiefel manifold {\color{black}(see Appendix~\ref{section. Brockett Cost Function on the Stiefel Manifold})}, where we denote its solution by $\matf{X}_{*}$. The columns of $\matf{X}_{*}$ comprise the $d$ top eigenvectors of $\matf{\Pi} - c \vecf{u} \vecf{u}^{\trans}$.
If $c > \sigma_1 - \sigma_m$, $\vecf{u}$ becomes the bottom eigenvector of $\matf{\Pi} - c \vecf{u} \vecf{u}^{\trans}$.
Thus $ {\matf{X}_*}^{\trans} \vecf{u} = \vecf{0}$ because eigenvectors with respect to different eigenvalues are orthogonal.
To conclude, if $c > \sigma_1 - \sigma_m$, problem (\ref{eq: property: proof - final result in X}) is a tight relaxation to problem (\ref{eq: general results: argmax, PCA in S, with S1=0}).

Lastly, when $c > \sigma_1 - \sigma_m$, the $d$ top eigenvectors of $\matf{\Pi} - c \vecf{u} \vecf{u}^{\trans}$ are the $d$ top eigenvectors of $\matf{\Pi}$ excluding $\vecf{u}$.

\vspace{5pt}
\noindent
\textbf{Case 2.}
We consider $c > \sigma_1 - \sigma_m$, and the following relaxation of problem (\ref{eq: general results: argmin, PCA in S, with S1=0}) without constraint $\matf{X}^{\trans}  \vecf{u} = \vecf{0}$:
\begin{equation}
\begin{aligned}
\min_{\matf{X}}\quad
&
\mathrm{tr}\left( \matf{X}^{\trans} 
\left( \matf{\Pi} +  c \vecf{u} \vecf{u}^{\trans} \right)
\matf{X} \matf{\Lambda} \right)
\\
&
\mathrm{s.t.} \quad
\matf{X}^{\trans}  \matf{X} = \matf{I}
,
\end{aligned}
\end{equation}
which is a tight relaxation to problem (\ref{eq: general results: argmin, PCA in S, with S1=0}) if $c > \sigma_1 - \sigma_m$.

\section{Derivation of the Reduced Problem}
\label{appendix. Derivation of the reduced problem}

\subsection{Linear Dependence of $\matf{A}_t$, $\vecf{a}_t$, $\boldsymbol{\Omega}_t$ on $\matf{M}$}
\label{appendix. Derivation of linear dependence of A, t, Omega on M}

By defining $\matf{\tilde{P}}_t = [ \matf{P}_t^{\trans} ,\, \vecf{1} ]^{\trans}$,
we notice that the affine part can be rewritten as:
$$
\matf{A}_t \matf{P}_t + \vecf{a}_t \vecf{1}^{\trans}
=
\left[
\matf{A}_t,\,\vecf{a}_t
\right]
\begin{bmatrix}
\matf{P}_t \\[5pt]
 \vecf{1}^{\trans}
\end{bmatrix}
=
\left[
\matf{A}_t,\,\vecf{a}_t
\right]
\matf{\tilde{P}}_t
,
$$
Then we write cost $\varphi_t (\matf{A}_t,\,\vecf{a}_t,\,\boldsymbol{\Omega}_t, \, \matf{M})$ in matrix form:
\begin{equation*}
\begin{aligned}
&
\varphi_t (\matf{A}_t, \vecf{a}_t, \boldsymbol{\Omega}_t, \matf{M})   = \left\Vert 
\matf{A}_t \matf{P}_t + \vecf{a}_t \vecf{1}^{\trans}
+ \boldsymbol{\Omega}_t^{\trans}
\matf{K}_t
- \matf{M} \matf{\Gamma}_t \right\Vert_{\mathcal{F}}^2
\\[5pt] & +
\mu_t
\mathrm{tr}
\left(
\boldsymbol{\Omega}_t^{\trans}
\matf{K}_t  \boldsymbol{\Omega}_t
\right)
\\[5pt] & =
\left\Vert
\left[
[\matf{A}_t,\,\vecf{a}_t],\,\boldsymbol{\Omega}_t^{\trans}
\right]
\begin{bmatrix}
\matf{\tilde{P}}_t & \vecf{O} \\[5pt]
\matf{K}_t & \sqrt{\mu_t \matf{K}_t}
\end{bmatrix}
-
\begin{bmatrix}
\matf{M} \matf{\Gamma}_t & \vecf{O}
\end{bmatrix}
\right\Vert_{\mathcal{F}}^2
\\[5pt]
& =
\left\Vert
\left[
[\matf{A}_t,\,\vecf{a}_t],\,\boldsymbol{\Omega}_t^{\trans}
\right]
\matf{C}_t 
-
\begin{bmatrix}
\matf{M} \matf{\Gamma}_t & \vecf{O}
\end{bmatrix}
\right\Vert_{\mathcal{F}}^2
,
\end{aligned}
\end{equation*}
where we have defined the matrix constant $\matf{C}_t$ as:
\begin{equation*}
\matf{C}_t
=
\begin{bmatrix}
\matf{\tilde{P}}_t & \vecf{O} \\[5pt]
\matf{K}_t & \sqrt{\mu_t \matf{K}_t}
\end{bmatrix}
.
\end{equation*}
Given $\matf{M}$, the problem regarding $\matf{A}_t,\,\vecf{a}_t,\,\boldsymbol{\Omega}_t$:
$$
\min_{\matf{A}_t,\,\vecf{a}_t,\,\boldsymbol{\Omega}_t}
\quad
\varphi_t (\matf{A}_t,\,\vecf{a}_t,\,\boldsymbol{\Omega}_t, \, \matf{M})
,
\quad
\mathrm{given}\ \matf{M}
,
$$
is a LLS optimization problem.
The optimal solution is in closed-form:
\begin{multline}
\label{eq. [At, at, Omegat] in M as Ct}
\left[
[\matf{A}_t,\,\vecf{a}_t],\,\boldsymbol{\Omega}_t^{\trans}
\right]
=
\begin{bmatrix}
\matf{M} \matf{\Gamma}_t & \vecf{O}
\end{bmatrix}
\matf{C}_t^{\dagger}
\\[5pt] +
\matf{F}_t
\left(
\matf{I} -  \matf{C}_t \matf{C}_t^{\dagger}
\right)
,
\end{multline}
where $\matf{C}_t^{\dagger}$ is the Moore–Penrose pseudo-inverse of $\matf{C}_t$, and $\matf{F}_t \in \mathbb{R}^{d \times (m_t+d+1)}$ is a free matrix (\textit{i.e.,}~an arbitrary matrix with the compatible dimension).
We denote:
\begin{equation*}
\boldsymbol{\Delta}_t
\defeq
\matf{C}_t  \matf{C}_t^{\trans} 
=
\begin{bmatrix}
\matf{\tilde{P}}_t \matf{\tilde{P}}_t^{\trans} & \matf{\tilde{P}}_t \matf{K}_t \\[5pt]
\matf{K}_t \matf{\tilde{P}}_t^{\trans} & \matf{K}_t \matf{K}_t + \mu_t \matf{K}_t
\end{bmatrix}
,
\end{equation*}
and expand the Moore–Penrose pseudo-inverse $\matf{C}_t^{\dagger}$ as:
$$
\matf{C}_t^{\dagger}
=
\matf{C}_t^{\trans} \left(  \matf{C}_t  \matf{C}_t^{\trans} \right)^{\dagger} 
=
\matf{C}_t^{\trans} \boldsymbol{\Delta}_t^{\dagger}
.
$$
Lastly, we express equation (\ref{eq. [At, at, Omegat] in M as Ct}) using $\boldsymbol{\Delta}_t$ as:
\begin{align}
\left[
[\matf{A}_t,\,\vecf{a}_t],\,\boldsymbol{\Omega}_t^{\trans}
\right]
& =
\begin{bmatrix}
\matf{M} \matf{\Gamma}_t & \vecf{O}
\end{bmatrix}
\matf{C}_t^{\trans} \boldsymbol{\Delta}_t^{\dagger}
+
\matf{F}_t
\left(
\matf{I} - \matf{C}_t \matf{C}_t^{\trans} \boldsymbol{\Delta}_t^{\dagger}
\right)
\notag
\\[5pt] & =
\matf{M} \matf{\Gamma}_t
\begin{bmatrix}
  \matf{\tilde{P}}_t^{\trans} &  \matf{K}_t
\end{bmatrix}
\boldsymbol{\Delta}_t^{\dagger}
+
\matf{F}_t
\left(
\matf{I} -  \boldsymbol{\Delta}_t \boldsymbol{\Delta}_t^{\dagger}
\right)
,
\notag
\end{align}
which is the form in equation (\ref{eq. At, at, Omegat in M}).

\subsection{Cost $\varphi_t (\matf{M})$}
\label{appendix. Derivation of Qt and phit}

We first notice that $\matf{Y}
\left(
\matf{I} -  \matf{C}_t \matf{C}_t^{\dagger}
\right)
\matf{C}_t
=
\matf{O}
$
because the Moore–Penrose pseudo-inverse satisfies $\matf{C}_t = \matf{C}_t \matf{C}_t^{\dagger} \matf{C}_t $.
Substituting equation (\ref{eq. [At, at, Omegat] in M as Ct}) into the cost $\varphi_t (\matf{A}_t, \vecf{a}_t, \boldsymbol{\Omega}_t, \matf{M})$, we obtain the reduced cost $\varphi_t (\matf{M}) $:
\begin{align*}
\varphi_t (\matf{M})  
 & =
\left\Vert
\begin{bmatrix}
\matf{M} \matf{\Gamma}_t & \vecf{O}
\end{bmatrix}
\left(
\matf{C}_t^{\dagger}
\matf{C}_t
- \matf{I}
\right)
\right\Vert_{\mathcal{F}}^2
\\[5pt] & =
\mathrm{tr} 
\left(
\begin{bmatrix}
\matf{M} \matf{\Gamma}_t & \vecf{O}
\end{bmatrix}
\left( \matf{I} - \matf{C}_t^{\dagger} \matf{C}_t \right)
\begin{bmatrix}
\matf{M} \matf{\Gamma}_t & \vecf{O}
\end{bmatrix}^{\trans}
\right)
,
\end{align*}
where we have used the fact that $\matf{I} - \matf{C}_t^{\dagger} \matf{C}_t$ is symmetric and idempotent, since it is an orthogonal projection matrix (\textit{i.e.,}~the orthogonal projector to the null space of $\matf{C}_t$).
In particular, we can write $\matf{C}_t^{\dagger} \matf{C}_t$ as:
$$
\matf{C}_t^{\dagger} \matf{C}_t
=
\matf{C}_t^{\trans} \left(  \matf{C}_t  \matf{C}_t^{\trans} \right)^{\dagger}  \matf{C}_t
=
\matf{C}_t^{\trans} \boldsymbol{\Delta}_t^{\dagger} \matf{C}_t
.
$$
The matrix multiplication shows:
$$
\begin{bmatrix}
\matf{M} \matf{\Gamma}_t & \vecf{O}
\end{bmatrix}
\matf{C}_t^{\trans}
=
\matf{M} \matf{\Gamma}_t
\begin{bmatrix}
  \matf{\tilde{P}}_t^{\trans} &  \matf{K}_t
\end{bmatrix}
.
$$
Lastly, we write $\varphi_t (\matf{M})$ as:
\begin{align*}
\varphi_t (\matf{M})  
=
\mathrm{tr} 
\left(
\matf{M} \matf{\Gamma}_t 
\left( \matf{I} - 
\begin{bmatrix}
  \matf{\tilde{P}}_t^{\trans} &  \matf{K}_t
\end{bmatrix}
\boldsymbol{\Delta}_t^{\dagger}
\begin{bmatrix}
  \matf{\tilde{P}}_t \\[5pt]  \matf{K}_t
\end{bmatrix}
\right)
\matf{\Gamma}_t^{\trans} \matf{M}^{\trans}
\right)
.
\end{align*}

\section{Positive Definiteness of $\boldsymbol{\Delta}_t$ and $\matf{\tilde{P}}_t \matf{\tilde{P}}_t^{\trans}$}
\label{section. positive definiteness between Delta and PP}

\subsection{Preliminary}
\begin{lemma}[(\citet*{Gallier2010SchurComplement})]
For any symmetric matrix $\matf{S}$ of the form:
\begin{equation*}
\matf{S} = 
\begin{bmatrix}
\matf{A} & \matf{B} 
\\[5pt]
\matf{B}^{\trans} & \matf{C}
\end{bmatrix},
\end{equation*}
if $\matf{C}$ is positive definite, then $\matf{S}$ is positive definite if and only if
$
\matf{A} - \matf{B}  \matf{C}^{-1}  \matf{B}^{\trans}
$
is positive definite.
\end{lemma}

\begin{lemma}[(Woodbury matrix identity)]
For any invertible $\matf{A}$ and $\matf{D}$, the following identity holds:
\begin{multline*}
( \matf{A} + \matf{B} \matf{D} \matf{C} )^{-1}
=
\matf{A}^{-1} 
\\[0pt]
-
\matf{A}^{-1} \matf{B}
\left(
\matf{D}^{-1} + \matf{C} \matf{A}^{-1} \matf{B}
\right)^{-1} 
\matf{C} \matf{A}^{-1}
.
\end{multline*}
\end{lemma}

\subsection{Proof of Lemma~\ref{lemma. Kt is PD iff PP is PD}}

We have denoted $\boldsymbol{\Delta}_t$ as:
\begin{equation*}
\boldsymbol{\Delta}_t
=
\begin{bmatrix}
\matf{\tilde{P}}_t \matf{\tilde{P}}_t^{\trans} & \matf{\tilde{P}}_t \matf{K}_t \\[5pt]
\matf{K}_t \matf{\tilde{P}}_t^{\trans} & \matf{K}_t \matf{K}_t + \mu_t \matf{K}_t
\end{bmatrix}
.
\end{equation*}
If $\matf{K}_t $ is chosen positive definite and $\mu_t > 0$, it suffices to examine the positive definiteness of:
\begin{align*}
\matf{E}_t
& =  \matf{\tilde{P}}_t \matf{\tilde{P}}_t^{\trans} -  \matf{\tilde{P}}_t \matf{K}_t 
(\matf{K}_t \matf{K}_t + \mu_t \matf{K}_t)^{-1} \matf{K}_t \matf{\tilde{P}}_t^{\trans}
\\[5pt]
& =
\matf{\tilde{P}}_t
\left(
\matf{I} - (\matf{I} + \mu_t \matf{K}_t^{-1})^{-1}
\right)
\matf{\tilde{P}}_t^{\trans}
\\[5pt]
& =
\matf{\tilde{P}}_t
( \matf{I} + \frac{1}{\mu_t} \matf{K}_t )^{-1}
\matf{\tilde{P}}_t^{\trans}
,
\end{align*}
where the last equality holds because of the Woodbury matrix identity.

We notice that $\matf{I} + \frac{1}{\mu_t} \matf{K}_t$ is positive definite.
As a result, $\matf{E}_t$ is positive definite if and only if $\matf{\tilde{P}}_t \matf{\tilde{P}}_t^{\trans}$ is positive definite.

\section{Expansion of $\boldsymbol{\Delta}_t^{\dagger}$}
\label{appendix. expansion of Pseudo inverse of Delta_t}

We expand $\boldsymbol{\Delta}_t^{\dagger}$ with respect to $\matf{\tilde{P}}_t \matf{\tilde{P}}_t^{\trans}$, by the Schur complement for positive semi-definite matrices.

To this end, we denote the Moore–Penrose pseudo-inverse of $\matf{\tilde{P}}_t \matf{\tilde{P}}_t^{\trans}$ as $(\matf{\tilde{P}}_t \matf{\tilde{P}}_t^{\trans})^{\dagger}$, and define:
\begin{equation}
\label{eq. def. appendix. definition of S_t}
\matf{S}_t \defeq \matf{K}_t \matf{K}_t + \mu_t \matf{K}_t - \matf{K}_t \matf{\tilde{P}}_t^{\trans} (\matf{\tilde{P}}_t \matf{\tilde{P}}_t^{\trans})^{\dagger} \matf{\tilde{P}}_t \matf{K}_t
.
\end{equation}
In equation (\ref{eq. def. appendix. definition of S_t}), using the notation $ \boldsymbol{\mathcal{P}}_{t} = \matf{\tilde{P}}_t^{\trans} ( \matf{\tilde{P}}_t \matf{\tilde{P}}_t^{\trans} )^{\dagger}  \matf{\tilde{P}}_t $,
we can rewrite $\matf{S}_t$ in form of equation (\ref{eq. expression of St}), as:
\begin{equation*}
\matf{S}_t =
\matf{K}_t \left( \matf{I} - \boldsymbol{\mathcal{P}}_t \right) \matf{K}_t + \mu_t \matf{K}_t
.
\end{equation*}
As an orthogonal projector, matrix $\matf{I} - \boldsymbol{\mathcal{P}}_t$ is always positive semi-definite.
Therefore if $\matf{K}_t$ is positive definite and $\mu_t > 0$, then $\matf{S}_t$ is positive definite.

By the Schur complement~\citet*{Gallier2010SchurComplement}, the expansion of $\boldsymbol{\Delta}_t^{\dagger}$ with respect to $\matf{\tilde{P}}_t \matf{\tilde{P}}_t^{\trans}$ is:
\begin{equation*}
\boldsymbol{\Delta}_t^{\dagger}
=
\begin{bmatrix}
\matf{\tilde{P}}_t \matf{\tilde{P}}_t^{\trans} & \matf{\tilde{P}}_t \matf{K}_t \\[5pt]
\matf{K}_t \matf{\tilde{P}}_t^{\trans} & \matf{K}_t \matf{K}_t + \mu_t \matf{K}_t
\end{bmatrix}^{\dagger}
=
\begin{bmatrix}
\boldsymbol{\Sigma}_{11} & \boldsymbol{\Sigma}_{12}
\\[5pt]
\boldsymbol{\Sigma}_{21} & \boldsymbol{\Sigma}_{22}
\end{bmatrix}
,
\end{equation*}
where the relevant blocks are defined as:
\begin{equation*}
\begin{cases}
\boldsymbol{\Sigma}_{11} \defeq (\matf{\tilde{P}}_t \matf{\tilde{P}}_t^{\trans})^{\dagger} + 
(\matf{\tilde{P}}_t \matf{\tilde{P}}_t^{\trans})^{\dagger} 
\matf{\tilde{P}}_t \matf{K}_t \matf{S}_t^{-1} \matf{K}_t \matf{\tilde{P}}_t^{\trans}
(\matf{\tilde{P}}_t \matf{\tilde{P}}_t^{\trans})^{\dagger}
\\[5pt]
\boldsymbol{\Sigma}_{12} \defeq - (\matf{\tilde{P}}_t \matf{\tilde{P}}_t^{\trans})^{\dagger} \matf{\tilde{P}}_t \matf{K}_t \matf{S}_t^{-1}
\\[5pt]
\boldsymbol{\Sigma}_{21} \defeq - \matf{S}_t^{-1} \matf{K}_t \matf{\tilde{P}}_t^{\trans} (\matf{\tilde{P}}_t \matf{\tilde{P}}_t^{\trans})^{\dagger}
\\[5pt]
\boldsymbol{\Sigma}_{22} \defeq \matf{S}_t^{-1}
.
\end{cases}
\end{equation*}

\section{Derivation of $\matf{Q}_t$}
\label{appendix. derivation of constant Ft, Gt and Qt}

Following Appendix~\ref{appendix. expansion of Pseudo inverse of Delta_t}, by denoting $\matf{\tilde{P}}_t^{\dagger} = \matf{\tilde{P}}_t^{\trans} (\matf{\tilde{P}}_t \matf{\tilde{P}}_t^{\trans})^{\dagger}
$
which is the Moore–Penrose pseudo-inverse of $\matf{\tilde{P}}_t$,
we further compute:
\begin{equation}
\begin{bmatrix}
\matf{\tilde{P}}_t^{\trans} & \matf{K}_t
\end{bmatrix}
\boldsymbol{\Delta}_t^{\dagger}
=
\begin{bmatrix}
\matf{\tilde{P}}_t^{\dagger} - \matf{H}_t \matf{K}_t \matf{\tilde{P}}_t^{\dagger} & \matf{H}_t
\end{bmatrix}
,
\label{eq. the way to obtain Ft, Gt}
\end{equation}
where:
\begin{equation}
\label{eq. define matrix constant Ht}
\matf{H}_t 
 \defeq
\left( \matf{I} - \boldsymbol{\mathcal{P}}_t \right)
\matf{K}_t \matf{S}_t^{-1}
.
\end{equation}
Matrix $\matf{H}_t$ thus defined is symmetric, \textit{i.e.,}~$\matf{H}_t = \matf{H}_t^{\trans}$, as to be shown in Proposition~\ref{proposition. properties of Ht} and Appendix~\ref{appendix. properties of Ht}.
From equation (\ref{eq. the way to obtain Ft, Gt}), we obtain:
\begin{align*}
\matf{Q}_t & = 
\matf{I} - 
\begin{bmatrix}
\matf{\tilde{P}}_t^{\trans} & \matf{K}_t
\end{bmatrix}
\boldsymbol{\Delta}_t^{\dagger}
\begin{bmatrix}
\matf{\tilde{P}}_t 
\\[5pt]
\matf{K}_t
\end{bmatrix}
\\[5pt]
& = 
\matf{I} - 
\left(
	\matf{\tilde{P}}_t^{\dagger} \matf{\tilde{P}}_t  - \matf{H}_t \matf{K}_t \matf{\tilde{P}}_t^{\dagger}  \matf{\tilde{P}}_t + \matf{H}_t \matf{K}_t
\right)
\\[5pt]
& = 
\left( \matf{I} - \boldsymbol{\mathcal{P}}_t \right)
 - \matf{H}_t \matf{K}_t
\left( \matf{I} - \boldsymbol{\mathcal{P}}_t \right)
\\[5pt]
& =
\left( \matf{I} - \boldsymbol{\mathcal{P}}_t \right)
 -
\left( \matf{I} - \boldsymbol{\mathcal{P}}_t \right)
\matf{K}_t \matf{S}_t^{-1} \matf{K}_t
\left( \matf{I} - \boldsymbol{\mathcal{P}}_t \right)
.
\end{align*}

We refer to the book~\citet*{meyer2000matrix} for some properties of the concepts of:
Moore–Penrose pseudo-inverse (Exercise 5.12.16) and orthogonal projectors (Chapter 5.13).

\section{Proof of Proposition~\ref{proposition. properties of Q_t}: Properties of $\matf{Q}_t$}
\label{appendix. proof of Q1 = 0}

\subsection{$\matf{I} \succeq \matf{I} - \boldsymbol{\mathcal{P}}_t   \succeq  \matf{Q}_t \succeq \matf{O}$}

\begin{lemma}
Both $ \boldsymbol{\mathcal{P}}_t$ and $\matf{I} - \boldsymbol{\mathcal{P}}_t$ are symmetric positive semidefinite.
\end{lemma}

\begin{lemma}
$\matf{S}_t =
\matf{K}_t \left( \matf{I} - \boldsymbol{\mathcal{P}}_t \right) \matf{K}_t + \mu_t \matf{K}_t
$ is symmetric positive definite, which is always invertible.
\end{lemma}
\begin{proof}
$ \mu_t \matf{K}_t$ is symmetric positive definite, since we assume symmetric positive definite $\matf{K}_t$ and $\mu_t > 0$.
The orthogonal projector $\matf{I} - \boldsymbol{\mathcal{P}}_t$ is symmetric positive semidefinite. Thus $\matf{K}_t \left( \matf{I} - \boldsymbol{\mathcal{P}}_t \right) \matf{K}_t$ is symmetric positive semidefinite.
As as result, $\matf{S}_t$ is symmetric positive definite.
\end{proof}

\begin{lemma}
$\matf{I} \succeq \matf{I} - \boldsymbol{\mathcal{P}}_t   \succeq  \matf{Q}_t$.
\end{lemma}
\begin{proof}
Given the fact that
$\matf{S}_t$ is symmetric positive definite, we know $\matf{S}_t^{-1}$ is symmetric positive definite.
As a result,
$
\left( \matf{I} - \boldsymbol{\mathcal{P}}_t \right)
\matf{K}_t \matf{S}_t^{-1} \matf{K}_t
\left( \matf{I} - \boldsymbol{\mathcal{P}}_t \right)$
is symmetric positive semidefinite.
Therefore:
$$
\left( \matf{I} - \boldsymbol{\mathcal{P}}_t \right) - \matf{Q}_t \succeq \matf{O}
\Leftrightarrow
\matf{I} - \boldsymbol{\mathcal{P}}_t \succeq \matf{Q}_t
.
$$
$\matf{I} \succeq \matf{I} - \boldsymbol{\mathcal{P}}_t $ is true since $\boldsymbol{\mathcal{P}}_t $ is positive semidefinite.
\end{proof}

\begin{lemma}
$\matf{Q}_t$ can be rewritten as:
\begin{multline}
\matf{Q}_t
= (\matf{H}_t \matf{K}_t - \matf{I})
(\matf{I} - \boldsymbol{\mathcal{P}}_{t})
( \matf{K}_t^{\trans} \matf{H}_t^{\trans} - \matf{I})
\\[5pt] +
\mu_t
 \matf{H}_t \matf{K}_t \matf{H}_t^{\trans}
,
\label{eq. the 1st form of Q_t}
\end{multline}
where $\matf{H}_t$ has been defined in equation (\ref{eq. define matrix constant Ht}).
\end{lemma}
\begin{proof}
The proof is given in Appendix~\ref{appendix. Qt expression in comparison to RSS version}.
\end{proof}

\begin{lemma}
$\matf{Q}_t$ is symmetric positive semidefinite.
\end{lemma}
\begin{proof}
In equation (\ref{eq. the 1st form of Q_t}), we notice that $\matf{I} - \boldsymbol{\mathcal{P}}_{t}$ is symmetric positive semidefinite, and $\matf{K}_t$ is symmetric positive definite.
\end{proof}

\subsection{$ \matf{Q}_t \vecf{1}_{m_t} = \vecf{0} $}

We notice that matrix
$\boldsymbol{\mathcal{P}}_{t}$ is the orthogonal projector to the range space of:
$$\matf{\tilde{P}}_t^{\trans} = \left[ \matf{P}_t^{\trans} ,\, \vecf{1}_{m_t} \right] . $$
We further observe that $\vecf{1}_{m_t}$ is in fact a column of $\matf{\tilde{P}}_t^{\trans}$ thus lying in the range of $\matf{\tilde{P}}_t^{\trans}$.
As a result, we have:
$$
\boldsymbol{\mathcal{P}}_{t} \vecf{1}_{m_t} = \vecf{1}_{m_t}
\Leftrightarrow
(\matf{I} - \boldsymbol{\mathcal{P}}_{t})  \vecf{1}_{m_t} = \vecf{0}
.
$$
It is thus obvious to see $ \matf{Q}_t \vecf{1}_{m_t} = \vecf{0} $.

\section{Proof of Proposition~\ref{proposition. properties of Ht}: Properties of $\matf{H}_t$}
\label{appendix. properties of Ht}

We compute:
\begin{align}
\matf{K}_t \matf{S}_t^{-1} \matf{K}_t 
& = 
\matf{K}_t
\left(
	\matf{K}_t \left( \matf{I} - \boldsymbol{\mathcal{P}}_t \right) \matf{K}_t + \mu_t \matf{K}_t
\right)^{-1}
\matf{K}_t
\notag
\\[5pt] & =
\left(
	\left( \matf{I} - \boldsymbol{\mathcal{P}}_t \right) +  \mu_t  \matf{K}_t^{-1}
\right)^{-1}
.
\notag
\end{align}
Matrix $\matf{K}_t \matf{S}_t^{-1} \matf{K}_t $ is invertible. Thus we have:
\begin{equation}
\begin{cases}
\left(
	\left( \matf{I} - \boldsymbol{\mathcal{P}}_t \right) +  \mu_t  \matf{K}_t^{-1}
\right)
\matf{K}_t \matf{S}_t^{-1} \matf{K}_t 
& = \matf{I}
\notag
\\[5pt]
\matf{K}_t \matf{S}_t^{-1} \matf{K}_t   \left(
	\left( \matf{I} - \boldsymbol{\mathcal{P}}_t \right) +  \mu_t  \matf{K}_t^{-1}
\right)
& = \matf{I}
.
\notag
\end{cases}
\end{equation}
Thus the following equalities hold true:
\begin{align}
\mu_t \matf{S}_t^{-1} \matf{K}_t
& =
\matf{I} - \left( \matf{I} - \boldsymbol{\mathcal{P}}_t \right) \matf{K}_t \matf{S}_t^{-1} \matf{K}_t
\label{eq. (I - P) K invS K}
\\[5pt]
\mu_t  \matf{K}_t \matf{S}_t^{-1}
& =
\matf{I} - \matf{K}_t \matf{S}_t^{-1} \matf{K}_t  \left( \matf{I} - \boldsymbol{\mathcal{P}}_t \right)
.
\label{eq. K invS K (I - P)}
\end{align}
By right-multiplying equation (\ref{eq. (I - P) K invS K}) by $\matf{I} - \boldsymbol{\mathcal{P}}_t$
and left-multiplying equation (\ref{eq. K invS K (I - P)}) by $\matf{I} - \boldsymbol{\mathcal{P}}_t$,
we have:
\begin{equation}
\label{eq. Qt = mu Ht = mu Ht^T}
\matf{Q}_t  =  \mu_t \, \underbrace{ \matf{S}_t^{-1} \matf{K}_t \left( \matf{I} - \boldsymbol{\mathcal{P}}_t \right) }_{ \matf{H}_t^{\trans} }
 =  \mu_t \, \underbrace{ \left( \matf{I} - \boldsymbol{\mathcal{P}}_t \right) \matf{K}_t \matf{S}_t^{-1} }_{ \matf{H}_t }
,
\end{equation}
where we have used the expression of $\matf{Q}_t$ in equation (\ref{eq. expression of Qt}).

\section{Connection to the Result in \citet*{Bai-RSS-22}}
\label{appendix. Qt expression in comparison to RSS version}

In~\citet*{Bai-RSS-22}, $\matf{Q}_t$ was defined as equation(\ref{eq. the 1st form of Q_t}) which can be simplified to the form in equation (\ref{eq. expression of Qt}).

In equation (\ref{eq. Qt = mu Ht = mu Ht^T}), we have proved that:
\begin{equation*}
\matf{H}_t
=
 ( \matf{I} - \boldsymbol{\mathcal{P}}_t )
  \matf{K}_t \matf{S}_t^{-1}
  =
  \matf{S}_t^{-1} \matf{K}_t 
 ( \matf{I} - \boldsymbol{\mathcal{P}}_t )
=
\matf{H}_t^{\trans} 
.
\end{equation*}
Hence, we can rewrite equation (\ref{eq. K invS K (I - P)}) in Appendix~\ref{appendix. properties of Ht} as:
\begin{align}
 \mu_t  \matf{K}_t \matf{S}_t^{-1} 
=
 \matf{I} 
 - 
 \matf{K}_t
 ( \matf{I} - \boldsymbol{\mathcal{P}}_t )
  \matf{K}_t \matf{S}_t^{-1}
.
\label{eq. an identity on KS}
\end{align}
We examine the term $\mu_t \matf{H}_t \matf{K}_t \matf{H}_t^{\trans}$ with the identity (\ref{eq. an identity on KS}), as:
\begin{align}
\mu_t
 \matf{H}_t \matf{K}_t \matf{H}_t^{\trans}
 & = 
 \mu_t
\matf{H}_t \matf{K}_t
\matf{S}_t^{-1} \matf{K}_t
 \left( \matf{I} - \boldsymbol{\mathcal{P}}_t \right)
\notag
\\[5pt]
&
=
\matf{H}_t 
\left(
	 \matf{I}   -   \matf{K}_t ( \matf{I} - \boldsymbol{\mathcal{P}}_t ) \matf{K}_t \matf{S}_t^{-1}
\right)
 \matf{K}_t
 \left( \matf{I} - \boldsymbol{\mathcal{P}}_t \right)
\notag
\\[5pt]
& = \matf{H}_t \matf{K}_t \left( \matf{I} - \boldsymbol{\mathcal{P}}_t \right)
-
\matf{H}_t \matf{K}_t \left( \matf{I} - \boldsymbol{\mathcal{P}}_t \right)
\matf{K}_t \matf{H}_t^{\trans} 
\notag
\\[5pt]
& = - \matf{H}_t \matf{K}_t \left( \matf{I} - \boldsymbol{\mathcal{P}}_t \right)
\left( 
	\matf{K}_t \matf{H}_t^{\trans}  - \matf{I} 
\right)
.
\label{eq. mu H K H}
\end{align}
Substituting equation (\ref{eq. mu H K H}) into equation (\ref{eq. the 1st form of Q_t}) to cancel the term $\mu_t
 \matf{H}_t \matf{K}_t \matf{H}_t^{\trans}$,
we reach $\matf{Q}_t$ in equation (\ref{eq. expression of Qt}) by some trivial matrix manipulations:
\begin{align*}
& (\matf{H}_t \matf{K}_t - \matf{I})
(\matf{I} - \boldsymbol{\mathcal{P}}_{t})
( \matf{K}_t^{\trans} \matf{H}_t^{\trans} - \matf{I})
+
\mu_t
 \matf{H}_t \matf{K}_t \matf{H}_t^{\trans}
 \\[5pt]
  =\, & 
-
(\matf{I} - \boldsymbol{\mathcal{P}}_{t})
( \matf{K}_t \matf{H}_t^{\trans} - \matf{I})
\\[5pt]
=\, &
-
(\matf{I} - \boldsymbol{\mathcal{P}}_{t})
( \matf{K}_t \matf{S}_t^{-1} \matf{K}_t \left( \matf{I} - \boldsymbol{\mathcal{P}}_t \right)
 - \matf{I})
\\[5pt]
=\, & \left( \matf{I} - \boldsymbol{\mathcal{P}}_t \right)
 -
\left( \matf{I} - \boldsymbol{\mathcal{P}}_t \right)
\matf{K}_t \matf{S}_t^{-1} \matf{K}_t
\left( \matf{I} - \boldsymbol{\mathcal{P}}_t \right)
= 
\matf{Q}_t
.
\end{align*}

\section{Proof of Lemma~\ref{lemma. invariance of orthogonal projection matrix under coordinate transformations}}
\label{appendix. proof of the invariance of the orthogonal projection matrix under coordinate transformations}

It can be shown that:
\begin{align*}	
	\mathbf{\tilde{\breve{P}}}_t
	 \defeq
	\begin{bmatrix}
	\mathbf{\breve{P}}_t \\[5pt] 
	\matf{1}^{\trans} 
	\end{bmatrix}
	=
	\begin{bmatrix}
	\mathbf{\breve{R}}_t \matf{P}_t + \mathbf{\breve{t}}_t  \vecf{1}^{\trans} \\[5pt]
	\matf{1}^{\trans}
	\end{bmatrix}
	 & =
	\begin{bmatrix}
	\mathbf{\breve{R}}_t & \mathbf{\breve{t}}_t \\[5pt]
	\matf{0}^{\trans}  & 1
	\end{bmatrix}
	\begin{bmatrix}
	\matf{P}_t \\[5pt] 
	\matf{1}^{\trans} 
	\end{bmatrix}	
	\\[5pt]
	& =
	\begin{bmatrix}
	\mathbf{\breve{R}}_t & \mathbf{\breve{t}}_t \\[5pt]
	\matf{0}^{\trans} & 1
	\end{bmatrix}	
	\matf{\tilde{P}}_t
.
\end{align*}
We notice that $\mathbf{\tilde{\breve{P}}}_t^{\trans}$ and $\matf{\tilde{P}}_t^{\trans}$ have the same range space, thus the orthogonal projection matrices are the same by the uniqueness~\citet*{meyer2000matrix}.
We can also verify the result by direct matrix calculations. We notice:
$$
\left(
\mathbf{\tilde{\breve{P}}}_t \mathbf{\tilde{\breve{P}}}_t^{\trans}
\right)^{\dagger}
=
\begin{bmatrix}
	\mathbf{\breve{R}}_t & \mathbf{\breve{t}}_t \\[5pt]
	\matf{0}^{\trans}  & 1
\end{bmatrix}^{ - \trans}
\left(
\matf{\tilde{P}}_t \matf{\tilde{P}}_t^{\trans}
\right)^{\dagger}
\begin{bmatrix}
	\mathbf{\breve{R}}_t & \mathbf{\breve{t}}_t \\[5pt]
	\matf{0}^{\trans}  & 1
\end{bmatrix}^{-1}
.
$$
Thus we have
$
\mathbf{\tilde{\breve{P}}}_t^{\trans}
(
\mathbf{\tilde{\breve{P}}}_t \mathbf{\tilde{\breve{P}}}_t^{\trans}
)^{\dagger}
\mathbf{\tilde{\breve{P}}}_t
=
\matf{\tilde{P}}_t^{\trans}
(
\matf{\tilde{P}}_t \matf{\tilde{P}}_t^{\trans}
)^{\dagger}
\matf{\tilde{P}}_t
$.

\section{Expansion of $\mathrm{tr}\left( \matf{X}^{\trans}  \boldsymbol{\mathcal{Q}} \matf{X} \boldsymbol{\Lambda} \right)$}
\label{appendix. expansion of tr(x Q x Lambda)}

We denote $\matf{X} = [\vecf{x}_1,\, \vecf{x}_2,\, \vecf{x}_3]$.
We denote the matrix square root of $\matf{Q}_t$ as $\sqrt{ \matf{Q}_t }$, where:
$
\matf{Q}_t =  \sqrt{ \matf{Q}_t }^{\trans} \sqrt{ \matf{Q}_t }
$.
\begin{align*}
 \mathrm{tr}\left( \matf{X}^{\trans} 
\boldsymbol{\mathcal{Q}}
\matf{X} \boldsymbol{\Lambda} \right)
& =
\sum_{t = 1}^{n}
\mathrm{tr}\left( \matf{X}^{\trans} 
 \matf{\Gamma}_t
\matf{Q}_t
 \matf{\Gamma}_t^{\trans}
\matf{X} \boldsymbol{\Lambda} \right)
\\[5pt]
&
 =
\sum_{t = 1}^{n}
\sum_{k = 1}^3
\lambda_k
\mathrm{tr}\left( \vecf{x}_k^{\trans} 
 \matf{\Gamma}_t
\matf{Q}_t
\matf{\Gamma}_t^{\trans}
\vecf{x}_k \right)
\\[5pt]
& =
\sum_{t = 1}^{n}
\sum_{k = 1}^3
\lambda_k
\left\Vert
\sqrt{ \matf{Q}_t} \,
( \matf{I} - \boldsymbol{\mathcal{P}}_t )
\matf{\Gamma}_t^{\trans}
\vecf{x}_k
\right\Vert_{\mathcal{F}}^2
\\[5pt]
& =
\sum_{k = 1}^3
\sum_{t = 1}^{n}
\lambda_k
\left\Vert
\sqrt{ \matf{Q}_t } \,
( \matf{I} - \boldsymbol{\mathcal{P}}_t )
\matf{\Gamma}_t^{\trans}
\vecf{x}_k
\right\Vert_{\mathcal{F}}^2
\\[5pt]
& =
\sum_{k = 1}^3
\lambda_k
\left\Vert
\begin{bmatrix}
\sqrt{ \matf{Q}_1 } \,
( \matf{I} - \boldsymbol{\mathcal{P}}_1 )
\matf{\Gamma}_1^{\trans}
\\[5pt]
\vdots
\\[5pt]
\sqrt{ \matf{Q}_n } \,
( \matf{I} - \boldsymbol{\mathcal{P}}_n )
\matf{\Gamma}_n^{\trans}
\end{bmatrix}
\vecf{x}_k
\right\Vert_{\mathcal{F}}^2
.
\end{align*}

\section{Planar Case}
\label{appendix. planar case. Rt in xy plane}

The optimal $\sqrt{\boldsymbol{\Lambda}}$ is characterized by formulation (\ref{eq: formulation for rigid transformation, as-rigid-as-possible, general form}), and thus formulation (\ref{eq. formulation of Rt in SO3 and Eta in R3}).
If $\matf{P}_t$ is flat, then $\matf{P}_t$ can be rigidly transformed to $\begin{bmatrix}
\mathbf{P}_{txy}^{\trans} &
\vecf{0}
\end{bmatrix}^{\trans}$.
Thus it suffices to discuss the estimate of $\sqrt{\boldsymbol{\Lambda}}$ from the canonical 2D point-clouds $\mathbf{P}_{txy}$.
We denote
$
\matf{G}_t = \matf{X}^{\trans} \matf{\Gamma}_t 
$
where $\matf{X}$ is given at this stage.
We further denote
$
\matf{\bar{G}}_t
=
\matf{G}_t - \frac{1}{m_t} \matf{G}_t \vecf{1} \vecf{1}^{\trans}
$.
By the fact that $\mathbf{P}_{txy}$ is zero-centered, $\sqrt{\boldsymbol{\Lambda}}$ is characterized by problem (\ref{eq. Rt and Lambda from canonical point clouds}).

In problem (\ref{eq. Rt and Lambda from canonical point clouds}),
if the optimal $\matf{R}_t$ implements a rotation in the $xy-$plane, then $\matf{R}_t$ can be formed as:
$$
\matf{R}_t = 
\begin{bmatrix}
\matf{R}_{txy} &  \\[4pt]
 & 1
\end{bmatrix}
.
$$
As a result, problem (\ref{eq. Rt and Lambda from canonical point clouds}) can be decomposed as:
\begin{multline*}
\min_{ \{ \matf{R}_{txy} \in \mathrm{SO}(2) \} ,\, \boldsymbol{\eta}_{xy} \in \mathbb{R}^2  ,\, \eta_z \in \mathbb{R} } \quad 
\sum_{t=1}^{n}
\eta_z
\left\Vert
\vecf{g}_3^{\trans} 
\right\Vert^2
\\[5pt]
+
\sum_{t=1}^{n}
\left\Vert
\matf{R}_{txy}
\mathbf{P}_{txy}
 - 
\mathbf{diag} ( \boldsymbol{\eta}_{xy} )
\begin{bmatrix}
\vecf{g}_1^{\trans} \\[5pt]
\vecf{g}_2^{\trans} 
\end{bmatrix}
\right\Vert_{\mathcal{F}}^2
,
\end{multline*}
where
$\matf{\bar{G}}_t = \begin{bmatrix}
\vecf{g}_1 & \vecf{g}_2 & \vecf{g}_3
\end{bmatrix}^{\trans}
$.
The cost of this problem is minimized if and only if $\eta_z = 0$. Thus $\lambda_3 = \vert \eta_z \vert = 0$.

\section{GPA Using the LBW in \citet*{bai2022ijcv}}
\label{appendix. results of GPA formulation using the LBW}

We recapitulate the result of \citet*{bai2022ijcv}.
If using the LBW, we will be solving a GPA formulation as:
\begin{equation}
\label{eq: deformable SLAM formulation using LBWs}
\begin{aligned}
\min_{ \{ \matf{W}_t \}, \, \matf{M}}
\quad 
& \sum_{t=1}^{n} \varphi_t (\matf{W}_t, \, \matf{M})
\\[2pt]  & 
\mathrm{s.t.}\ 
\matf{M} \vecf{1} = \vecf{0},\ 
\matf{M} \matf{M}^{\trans} = \boldsymbol{\Lambda}
,
\end{aligned}
\end{equation}
where:
\begin{align*}
\varphi_t (\matf{W}_t, \, \matf{M})  
 = & \left\Vert 
\matf{W}_t^{\trans}
\boldsymbol{\mathcal{B}}_t(\matf{P}_t)
-
 \matf{M} \matf{\Gamma}_t \right\Vert_{\mathcal{F}}^2
\notag
\\[5pt] & + 
\mu_t 
\mathrm{tr}
\left(
\matf{W}_t^{\trans}
\matf{\Xi}_t  \matf{W}_t
\right)
.
\end{align*}
We define the matrix $\boldsymbol{\mathcal{Q}}$ as:
$$
\boldsymbol{\mathcal{Q}} \defeq
\matf{\Gamma}_t 
\left(
\matf{I}
-
\boldsymbol{\mathcal{B}}_t^{\trans}
\left(
	\boldsymbol{\mathcal{B}}_t \boldsymbol{\mathcal{B}}_t^{\trans}
	+ \mu_t  \matf{\Xi}_t
\right)^{-1}
\boldsymbol{\mathcal{B}}_t
\right)
\matf{\Gamma}_t^{\trans}
,
$$
where we have used the shorthand
$\boldsymbol{\mathcal{B}}_t \defeq \boldsymbol{\mathcal{B}}_t(\matf{P}_t)$.

If $\boldsymbol{\mathcal{Q}} \vecf{1} = \vecf{0}$, then the optimal $\matf{M}$ of problem (\ref{eq: deformable SLAM formulation using LBWs}) is:
$$\matf{M} = \sqrt{\boldsymbol{\Lambda}} \matf{X}^{\trans}
,
\quad
\mathrm{with\ \ }
\matf{X} = [\vecf{x}_1,\,\vecf{x}_2,\, \dots, \vecf{x}_d]  \in \mathbb{R}^{m \times d}
,
$$ where $\vecf{x}_1,\,\vecf{x}_2,\, \dots, \vecf{x}_d$ in sequence are the $d$ bottom eigenvectors of $\boldsymbol{\mathcal{Q}}$ excluding the vector $\vecf{1}$, or equivalently are the $d$ bottom eigenvectors of $\boldsymbol{\mathcal{Q}}' = \boldsymbol{\mathcal{Q}} + n \vecf{1} \vecf{1}^{\trans}$.
The optimal transformation parameters $\matf{W}_t$ are:
\begin{equation*}
\matf{W}_t^{\trans} = \matf{M} \matf{\Gamma}_t 
\boldsymbol{\mathcal{B}}_t^{\trans}
\left(
	\boldsymbol{\mathcal{B}}_t \boldsymbol{\mathcal{B}}_t^{\trans}
	+ \mu_t  \matf{\Xi}_t
\right)^{-1}
,
\quad
\left( t \in [1 : n] \right)
.
\end{equation*}

It has been shown that
$\boldsymbol{\mathcal{Q}} \vecf{1} = \vecf{0}$ happens if the LBW has a free translation.
In particular, if the LBW is chosen as the affine transformation, or the TPS warp, then
$\boldsymbol{\mathcal{Q}} \vecf{1} = \vecf{0}$.

\begin{acks}
The author would like to thank professor Yi Dong in Tongji University, Shanghai, China, for the hosting and support to finish the initial manuscript.
We want to express our gratitude to the TOPACS team (project No. ANR-19-CE45-0015) who produced the original CT point-cloud for our experiments.
\end{acks}

\input{paper.bbl}

\end{document}

%% file: IJRR_figure1_DefGPA.eps_tex
\begingroup%
  \makeatletter%
  \providecommand\color[2][]{%
    \errmessage{(Inkscape) Color is used for the text in Inkscape, but the package 'color.sty' is not loaded}%
    \renewcommand\color[2][]{}%
  }%
  \providecommand\transparent[1]{%
    \errmessage{(Inkscape) Transparency is used (non-zero) for the text in Inkscape, but the package 'transparent.sty' is not loaded}%
    \renewcommand\transparent[1]{}%
  }%
  \providecommand\rotatebox[2]{#2}%
  \newcommand*\fsize{\dimexpr\f@size pt\relax}%
  \newcommand*\lineheight[1]{\fontsize{\fsize}{#1\fsize}\selectfont}%
  \ifx\svgwidth\undefined%
    \setlength{\unitlength}{453.51566825bp}%
    \ifx\svgscale\undefined%
      \relax%
    \else%
      \setlength{\unitlength}{\unitlength * \real{\svgscale}}%
    \fi%
  \else%
    \setlength{\unitlength}{\svgwidth}%
  \fi%
  \global\let\svgwidth\undefined%
  \global\let\svgscale\undefined%
  \makeatother%
  \begin{picture}(1,0.36852872)%
    \lineheight{1}%
    \setlength\tabcolsep{0pt}%
    \put(0,0){\includegraphics[width=\unitlength,page=1]{IJRR_figure1_DefGPA.eps}}%
    \put(0.5274534,0.14214016){\color[rgb]{0,0,0}\rotatebox{57.91233317}{\makebox(0,0)[lt]{\lineheight{1.25}\smash{\begin{tabular}[t]{l}\textit{x}\end{tabular}}}}}%
    \put(0.47200356,0.11674316){\color[rgb]{0,0,0}\rotatebox{57.91233317}{\makebox(0,0)[lt]{\lineheight{1.25}\smash{\begin{tabular}[t]{l}\textit{y}\end{tabular}}}}}%
    \put(0.51114777,0.10055265){\color[rgb]{0,0,0}\rotatebox{57.91233015}{\makebox(0,0)[lt]{\lineheight{1.25}\smash{\begin{tabular}[t]{l}\textit{o}\end{tabular}}}}}%
    \put(0.15880546,0.06321865){\color[rgb]{0,0,0}\makebox(0,0)[lt]{\lineheight{1.25}\smash{\begin{tabular}[t]{l}\textit{x}\end{tabular}}}}%
    \put(0.07796024,0.12673668){\color[rgb]{0,0,0}\makebox(0,0)[lt]{\lineheight{1.25}\smash{\begin{tabular}[t]{l}\textit{y}\end{tabular}}}}%
    \put(0.08513904,0.06377344){\color[rgb]{0,0,0}\makebox(0,0)[lt]{\lineheight{1.25}\smash{\begin{tabular}[t]{l}\textit{o}\end{tabular}}}}%
    \put(0.5286806,0.1012418){\color[rgb]{0,0,1}\makebox(0,0)[lt]{\lineheight{1.25}\smash{\begin{tabular}[t]{l}$(\mathbf{R}_t,\, \mathbf{t}_t)$\end{tabular}}}}%
    \put(0.20530987,0.14746979){\color[rgb]{0,0,1}\makebox(0,0)[lt]{\lineheight{1.25}\smash{\begin{tabular}[t]{l}$\mathbf{\Phi}_t(\cdot)$\end{tabular}}}}%
    \put(0.38996268,0.01868224){\color[rgb]{0,0,0}\makebox(0,0)[lt]{\lineheight{1.25}\smash{\begin{tabular}[t]{l}$\mathbf{P}_t = \mathbf{R}_t^{\mathsf{T}} \left( \mathbf{\Phi}_t (\mathbf{M} \mathbf{\Gamma}_t) -  \mathbf{t}_t \mathbf{1}^{\mathsf{T}} \right)$ $\Longleftrightarrow \mathbf{R}_t \mathbf{P}_t + \mathbf{t}_t \mathbf{1}^{\mathsf{T}} = \mathbf{\Phi}_t (\mathbf{M} \mathbf{\Gamma}_t)$\end{tabular}}}}%
    \put(0.64421088,0.26918501){\color[rgb]{0,0,0}\makebox(0,0)[lt]{\lineheight{1.25}\smash{\begin{tabular}[t]{l}$\mathbf{y}_t (\mathbf{P}_t) \ \stackrel{\mathrm{def}}{=} \ \mathbf{\Phi}_t^{-1} ( \mathbf{R}_t \mathbf{P}_t + \mathbf{t}_t \mathbf{1}^{\mathsf{T}})$\end{tabular}}}}%
    \put(0.31714822,0.30744879){\color[rgb]{0,0,1}\makebox(0,0)[lt]{\lineheight{1.25}\smash{\begin{tabular}[t]{l}$\mathbf{M}$\end{tabular}}}}%
    \put(0.61490877,0.18670237){\color[rgb]{0,0.05882353,0.43921569}\makebox(0,0)[lt]{\lineheight{1.25}\smash{\begin{tabular}[t]{l}The pose and deformation are entangled,\\which means both are ambiguous.\end{tabular}}}}%
    \put(0.61680346,0.31350613){\color[rgb]{0,0.05882353,0.43921569}\makebox(0,0)[lt]{\lineheight{1.25}\smash{\begin{tabular}[t]{l}\textbf{Deformable transformation:}\end{tabular}}}}%
    \put(0.76735316,0.22827043){\color[rgb]{0,0,0}\makebox(0,0)[lt]{\lineheight{1.25}\smash{\begin{tabular}[t]{l}$\mathbf{M} \mathbf{\Gamma}_t$\end{tabular}}}}%
    \put(0.4394935,0.16147193){\color[rgb]{0,0,0}\makebox(0,0)[lt]{\lineheight{1.25}\smash{\begin{tabular}[t]{l}$\mathbf{\Gamma}_t$\end{tabular}}}}%
    \put(0.28757792,0.05886952){\color[rgb]{1,0.14901961,0}\makebox(0,0)[lt]{\lineheight{1.25}\smash{\begin{tabular}[t]{l}observation at the sensor’s local coordinate frame\end{tabular}}}}%
    \put(0.01645612,0.03230382){\color[rgb]{0,0,0}\makebox(0,0)[lt]{\lineheight{1.25}\smash{\begin{tabular}[t]{l}the global coordinate frame\end{tabular}}}}%
  \end{picture}%
\endgroup%

%% file: IJRR_figure2_Constraints.eps_tex
\begingroup%
  \makeatletter%
  \providecommand\color[2][]{%
    \errmessage{(Inkscape) Color is used for the text in Inkscape, but the package 'color.sty' is not loaded}%
    \renewcommand\color[2][]{}%
  }%
  \providecommand\transparent[1]{%
    \errmessage{(Inkscape) Transparency is used (non-zero) for the text in Inkscape, but the package 'transparent.sty' is not loaded}%
    \renewcommand\transparent[1]{}%
  }%
  \providecommand\rotatebox[2]{#2}%
  \newcommand*\fsize{\dimexpr\f@size pt\relax}%
  \newcommand*\lineheight[1]{\fontsize{\fsize}{#1\fsize}\selectfont}%
  \ifx\svgwidth\undefined%
    \setlength{\unitlength}{453.51847972bp}%
    \ifx\svgscale\undefined%
      \relax%
    \else%
      \setlength{\unitlength}{\unitlength * \real{\svgscale}}%
    \fi%
  \else%
    \setlength{\unitlength}{\svgwidth}%
  \fi%
  \global\let\svgwidth\undefined%
  \global\let\svgscale\undefined%
  \makeatother%
  \begin{picture}(1,0.36853273)%
    \lineheight{1}%
    \setlength\tabcolsep{0pt}%
    \put(0,0){\includegraphics[width=\unitlength,page=1]{IJRR_figure2_Constraints.eps}}%
    \put(0.07456268,0.30785434){\color[rgb]{0,0,0}\makebox(0,0)[lt]{\lineheight{1.25}\smash{\begin{tabular}[t]{l}$\mathbb{C}\mathrm{ov}(\mathbf{M}) = \mathbf{Q} \boldsymbol{\Lambda} \mathbf{Q}^{\mathsf{T}}$\end{tabular}}}}%
    \put(0.45998818,0.02487278){\color[rgb]{0,0,0}\makebox(0,0)[lt]{\lineheight{1.25}\smash{\begin{tabular}[t]{l}$\mathbf{M}_c \mathbf{1} = \mathbf{0}$\end{tabular}}}}%
    \put(0.77506612,0.02487276){\color[rgb]{0,0,0}\makebox(0,0)[lt]{\lineheight{1.25}\smash{\begin{tabular}[t]{l}$\mathbf{M}_r \mathbf{1} = \mathbf{0}, \, \mathbf{M}_r \mathbf{M}_r^{\mathsf{T}}  = \boldsymbol{\Lambda}$\end{tabular}}}}%
    \put(0.18630371,0.14211927){\color[rgb]{0,0,0}\rotatebox{-14.155536}{\makebox(0,0)[lt]{\lineheight{1.25}\smash{\begin{tabular}[t]{l}$\mathbf{M}_c \gets \mathbf{M} - \frac{1}{m} \mathbf{M} \mathbf{1} \mathbf{1}^{\mathsf{T}}$\end{tabular}}}}}%
    \put(0.18627814,0.10119593){\color[rgb]{0,0,0}\rotatebox{-14.155536}{\makebox(0,0)[lt]{\lineheight{1.25}\smash{\begin{tabular}[t]{l}apply translation\end{tabular}}}}}%
    \put(0.62331371,0.10391515){\color[rgb]{0,0,0}\rotatebox{19.65934472}{\makebox(0,0)[lt]{\lineheight{1.25}\smash{\begin{tabular}[t]{l}$\mathbf{M}_r \gets \mathbf{Q}^{\mathsf{T}} \mathbf{M}_c $\end{tabular}}}}}%
    \put(0.63420511,0.06566259){\color[rgb]{0,0,0}\rotatebox{19.659336}{\makebox(0,0)[lt]{\lineheight{1.25}\smash{\begin{tabular}[t]{l}apply rotation\end{tabular}}}}}%
    \put(0.13493137,0.13685113){\color[rgb]{0,0,0}\makebox(0,0)[lt]{\lineheight{1.25}\smash{\begin{tabular}[t]{l}x\end{tabular}}}}%
    \put(0.13818374,0.21790169){\color[rgb]{0,0,0}\makebox(0,0)[lt]{\lineheight{1.25}\smash{\begin{tabular}[t]{l}y\end{tabular}}}}%
    \put(0.03025694,0.31516063){\color[rgb]{0,0,0}\makebox(0,0)[lt]{\lineheight{1.25}\smash{\begin{tabular}[t]{l}z\end{tabular}}}}%
    \put(0.02568043,0.17085184){\color[rgb]{0,0,0}\makebox(0,0)[lt]{\lineheight{1.25}\smash{\begin{tabular}[t]{l}o\end{tabular}}}}%
    \put(0.59198466,0.09054618){\color[rgb]{0,0,0}\makebox(0,0)[lt]{\lineheight{1.25}\smash{\begin{tabular}[t]{l}x\end{tabular}}}}%
    \put(0.59523703,0.17159674){\color[rgb]{0,0,0}\makebox(0,0)[lt]{\lineheight{1.25}\smash{\begin{tabular}[t]{l}y\end{tabular}}}}%
    \put(0.48731022,0.26885568){\color[rgb]{0,0,0}\makebox(0,0)[lt]{\lineheight{1.25}\smash{\begin{tabular}[t]{l}z\end{tabular}}}}%
    \put(0.48273372,0.12454689){\color[rgb]{0,0,0}\makebox(0,0)[lt]{\lineheight{1.25}\smash{\begin{tabular}[t]{l}o\end{tabular}}}}%
    \put(0.59198466,0.09054618){\color[rgb]{0,0,0}\makebox(0,0)[lt]{\lineheight{1.25}\smash{\begin{tabular}[t]{l}x\end{tabular}}}}%
    \put(0.59523703,0.17159674){\color[rgb]{0,0,0}\makebox(0,0)[lt]{\lineheight{1.25}\smash{\begin{tabular}[t]{l}y\end{tabular}}}}%
    \put(0.48731022,0.26885568){\color[rgb]{0,0,0}\makebox(0,0)[lt]{\lineheight{1.25}\smash{\begin{tabular}[t]{l}z\end{tabular}}}}%
    \put(0.48273372,0.12454689){\color[rgb]{0,0,0}\makebox(0,0)[lt]{\lineheight{1.25}\smash{\begin{tabular}[t]{l}o\end{tabular}}}}%
    \put(0.95517952,0.13685079){\color[rgb]{0,0,0}\makebox(0,0)[lt]{\lineheight{1.25}\smash{\begin{tabular}[t]{l}x\end{tabular}}}}%
    \put(0.95843189,0.21790136){\color[rgb]{0,0,0}\makebox(0,0)[lt]{\lineheight{1.25}\smash{\begin{tabular}[t]{l}y\end{tabular}}}}%
    \put(0.85050509,0.31516029){\color[rgb]{0,0,0}\makebox(0,0)[lt]{\lineheight{1.25}\smash{\begin{tabular}[t]{l}z\end{tabular}}}}%
    \put(0.84592858,0.1708515){\color[rgb]{0,0,0}\makebox(0,0)[lt]{\lineheight{1.25}\smash{\begin{tabular}[t]{l}o\end{tabular}}}}%
  \end{picture}%
\endgroup%